\pdfoutput=1

\documentclass{article}

\usepackage{iclr2027_conference,times}

\usepackage{xcolor}
\usepackage[normalem]{ulem} %
\usepackage[
  backgroundcolor=orange!15,
  bordercolor=orange!80!black,
  linecolor=orange!80!black
]{todonotes}

\usepackage{amsmath}
\usepackage{amssymb}
\usepackage{amsthm}

\usepackage{microtype}

\usepackage{algorithm}
\usepackage{algpseudocode}

\usepackage{graphicx}
\usepackage{float}
\usepackage{booktabs}
\usepackage{array}
\usepackage{threeparttable}
\usepackage{subcaption}

\usepackage{wrapfig}

\graphicspath{{figures/}{graph_scripts/scripts/}}

\usepackage{xspace}

\usepackage{xurl}
\usepackage[hypertexnames=false]{hyperref} %
\usepackage[nameinlink,noabbrev]{cleveref}
\usepackage{autonum} %
\usepackage{wrapfig}
\makeatletter
\AtBeginDocument{\autonum@generatePatchedReferenceCSL{Cref}}
\makeatother
\usepackage{enumitem}
\usepackage[skins,breakable]{tcolorbox} %
\definecolor{linknavy}{RGB}{16,72,140}
\hypersetup{
  colorlinks = true,
  citecolor  = linknavy,
  linkcolor  = linknavy,
  urlcolor   = linknavy,
}

\newtheorem{theorem}{Theorem}
\newtheorem{lemma}{Lemma}
\newtheorem{proposition}{Proposition}

\newcommand{\R}{\mathbb{R}}
\newcommand{\E}{\mathbb{E}}

\newcommand{\msv}{\mathsf{V}}                 %
\newcommand{\msx}{\mathsf{X}}                 %

\newcommand{\gauss}{\mathrm{N}}               %
\newcommand{\Id}{\operatorname{Id}}           %
\newcommand{\Unif}{\mathrm{Unif}}             %
\newcommand{\eqsp}{\;}                        %
\newcommand{\KL}[2]{\mathrm{KL}\!\left( #1 \,\middle\|\, #2 \right)}
\newcommand{\NFE}{\mathrm{NFE}}
\newcommand{\SNR}{\mathrm{SNR}}
\newcommand{\sg}[1]{\mathrm{sg}(#1)}          %

\newcommand{\Emat}{\mathbf{E}}                       %
\newcommand{\post}[1][]{p_{\mathbf{x} \mid t}^{#1}}       %
\newcommand{\genkernel}{\mathsf{k}}

\DeclareRobustCommand{\SDMD}{Simplex-DMD\xspace}
\DeclareRobustCommand{\RDMD}{Reinforce-DMD\xspace}

\newcommand{\eqd}{\overset{d}{=}}

\newcommand{\entval}[1]{\,{\scriptsize\textcolor{gray}{(#1)}}}

\definecolor{legA}{RGB}{0,114,178}
\definecolor{legB}{RGB}{213,94,0}
\definecolor{legC}{RGB}{0,158,115}
\definecolor{legD}{RGB}{230,159,0}
\definecolor{legE}{RGB}{204,121,167}
\definecolor{legF}{RGB}{86,180,233}
\definecolor{legG}{RGB}{117,112,179}
\definecolor{legH}{RGB}{90,90,90}

\newtcolorbox{samplebox}[3]{enhanced, breakable,
  colback=linknavy!3, colframe=linknavy, colbacktitle=linknavy, coltitle=white,
  boxrule=0.6pt, arc=2pt, left=6pt, right=6pt, top=4pt, bottom=4pt,
  toptitle=2pt, bottomtitle=2pt, before skip=14pt, after skip=14pt,
  segmentation style={draw=linknavy!35, dashed},
  fonttitle=\small\bfseries, fontupper=\small,
  before upper={\setlength{\parskip}{3pt}},
  title={#1\hfill\mdseries #2\quad\textperiodcentered\quad #3}}
\newcommand{\samplehead}[1]{{\footnotesize\bfseries\textcolor{linknavy}{#1}}\par}
\newcommand{\eostoken}{\textcolor{linknavy}{\texttt{\footnotesize\textless eos\textgreater}}}

\title{Distribution Matching Distillation for \\ Continuous Diffusion Language Models}

\author{%
  \begin{tabular}[t]{@{}p{0.5\textwidth}@{}l@{}}
    Paul Le Van Kiem\thanks{Equal contribution.}
      & Dario Shariatian\footnotemark[1] \\
    \mdseries Inria, PSL Research University
      & \mdseries Cohere \\
    \mdseries\texttt{paul.le-van-kiem@inria.fr}
      & \mdseries\texttt{dario.shariatian@cohere.com} \\[3ex]
    Umut Simsekli
      & Alain Durmus \\
    \mdseries Inria, PSL Research University
      & \mdseries CMAP, Ecole Polytechnique \\
    \mdseries\texttt{umut.simsekli@inria.fr}
      & \mdseries\texttt{alain.durmus@polytechnique.edu}
  \end{tabular}%
}

\iclrfinalcopy

\begin{document}

\maketitle
\lhead{}

\addtocontents{toc}{\protect\setcounter{tocdepth}{-1}}

\begin{abstract}

Continuous diffusion language models generate all tokens in parallel, yet high-quality generation can still require hundreds of network evaluations (NFEs). We study how distributional distillation can reduce this cost by exploiting the student's probabilistic token outputs. Our unified formulation connects the student's output parameterization to the resulting gradient estimators and yields two methods with the same student architecture and reverse-KL matching objective: Simplex-DMD uses continuous token relaxations and pathwise gradients, while Reinforce-DMD uses categorical sampling and REINFORCE with a learned density ratio. We develop both methods for multi-step generation and investigate the training and sampling choices associated with each parameterization. On OpenWebText, for sequences of 1,024 tokens, Simplex-DMD achieves a generative perplexity of 45.6 at a unigram entropy of 5.44 nats in just 4 NFEs, a 49\% reduction relative to the strongest evaluated diffusion baseline at matched entropy and sampling budget. Reinforce-DMD improves the frontier at larger budgets, reaching a generative perplexity of 14.9 at an entropy of 5.00 nats with 256 NFEs, a 20\% reduction under the same comparison protocol. 

\end{abstract}

\section{Introduction}
\label{sec:introduction}

Diffusion language models generate text either through diffusion directly in discrete token space \citep{austin2021structured,lou2024discrete,sahoo2024simple}
or through Gaussian diffusion in a continuous embedding space \citep{dieleman2022continuous}.
Recent advances in continuous diffusion language models have shown that the latter approach is competitive
with its discrete counterparts \citep{chen2026langflow,yang2026replaid,hu2026elf}.
Although both approaches update all sequence positions in parallel,
generating high-quality text can still require hundreds of sequential network evaluations.
This sampling cost motivates diffusion distillation, which aims to reduce the number of evaluations
while preserving generation quality, using a pretrained teacher to supervise a student to generate samples in fewer steps.

Two broad approaches differ in the teacher supervision they provide. Trajectory-based methods supervise student predictions or transitions along the teacher's dynamics. For continuous data, these include consistency and flow map methods \citep{song2023consistency,boffi2025flowmap}, and for discrete data they have been developed for both discrete and continuous diffusion models \citep{deschenaux2025sdtt,hayakawa2025di4c, roos2026categorical,potaptchik2026dfm,lee2026fmlm}.

Distributional methods instead use the teacher to guide the student toward the data distribution, without prescribing its generation trajectories. Developed for continuous data in both one-step and multi-step regimes \citep{luo2023diffinstruct,yin2024onestep,salimans2024multistep}, these approaches have recently been explored for discrete data, like language, with discrete \citep{hoogeboom2026dmmd,li2026idlm,zhu2025dimo} and continuous diffusion models \citep{chen2026dlmonediffusionlanguagemodels}.

We study distributional distillation for continuous diffusion models on discrete data (CDMd). Their denoising networks predict a probability distribution over clean tokens at each position, given noisy embeddings and a noise level. Retaining these outputs for the student leaves a natural choice: use the probability vectors directly as continuous token representations, or sample discrete tokens from them. 
In both cases, we embed and noise the student outputs and match their distribution to the corresponding noised data distribution under a reverse-KL objective. These two parameterizations yield different training procedures. Simplex Distribution Matching Distillation (\SDMD) uses pathwise gradients through continuous token relaxations, with an auxiliary denoiser estimating the score of the noised student distribution, following Diff-Instruct and DMD \citep{luo2023diffinstruct,yin2024onestep}. REINFORCE Distribution Matching Distillation (\RDMD) uses categorical samples and a REINFORCE estimator \citep{williams1992simple}, with a discriminator estimating the density ratio between noised student and data distributions. 
We extend both methods to multi-step generation and study their training behavior and quality-diversity tradeoffs across sampling budgets.

\begin{figure}[t]
    \captionsetup{font=footnotesize}
    \centering
    \begin{subfigure}[t]{0.495\linewidth}
        \centering
        \includegraphics[width=\linewidth]{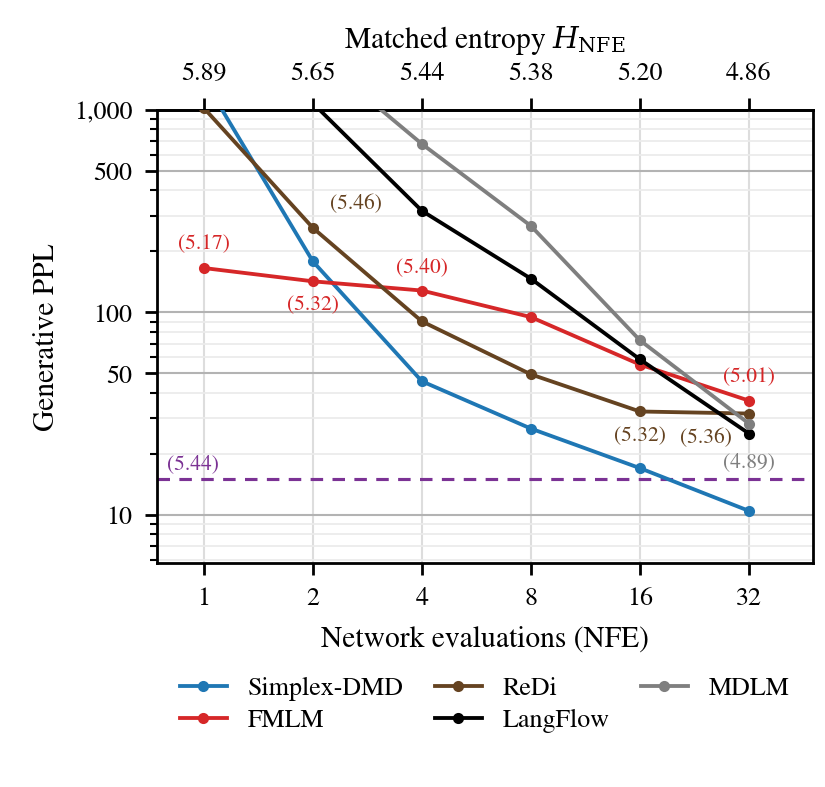}
        \subcaption{Few-step budgets, matched to \SDMD.}
        \label{fig:intro-nfe-fewstep}
    \end{subfigure}%
    \hfill
    \begin{subfigure}[t]{0.495\linewidth}
        \centering
        \includegraphics[width=\linewidth]{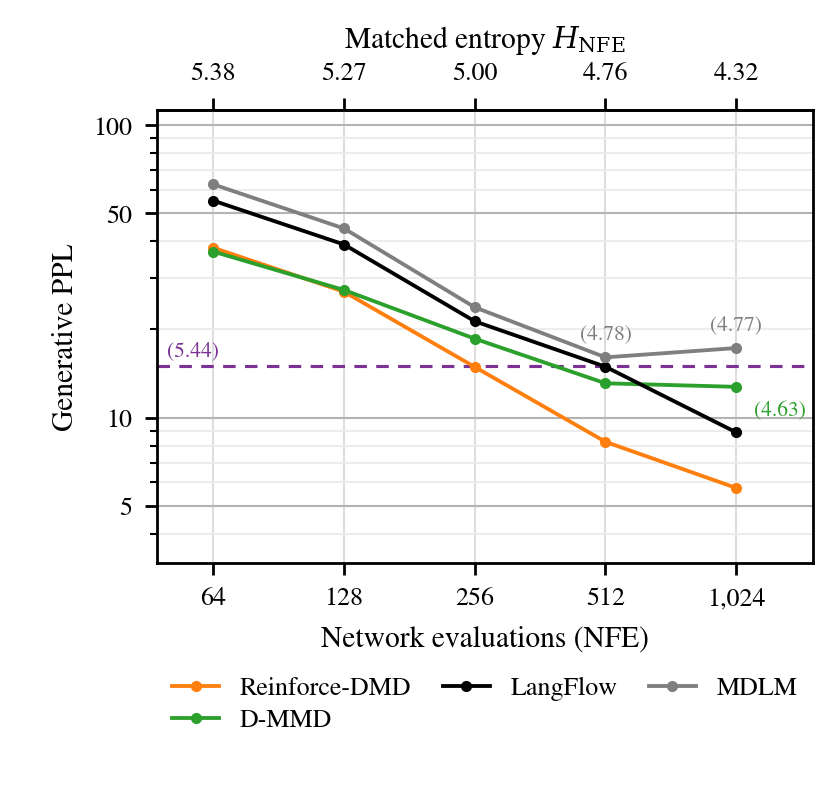}
        \subcaption{Multi-step budgets, matched to \RDMD.}
        \label{fig:intro-nfe-multistep}
    \end{subfigure}
    \caption{Gen PPL on OpenWebText against number of function evaluations (NFE), at the
    matched unigram entropy $H_{\NFE}$ (top axis), defined as in \Cref{tab:exp-proposal-budgets}.
    Each method's log Gen PPL is linearly interpolated at $H_{\NFE}$ between the two
    sampling temperatures whose entropies bracket it.
    LangFlow and MDLM are undistilled models; FMLM, ReDi and D-MMD are distillation methods.
    A method that does not reach $H_{\NFE}$ is shown at the closest measured entropy, with its value in parentheses. The dashed purple line marks the Gen PPL of the data, with its
    entropy in parentheses.}
    \label{fig:intro-nfe}
\end{figure}

Our contributions are as follows:
\begin{enumerate}[wide,label=$\bullet$]
\item We introduce a unified framework for distributional distillation that compares noised student and data distributions
through a generic discrepancy. This formulation recovers objectives used in prior work as particular choices of discrepancy
and extends to multi-step generation.

\item We specialize this framework to reverse-KL matching and derive two distillation methods: \SDMD and \RDMD. Together, they provide continuous-relaxation and categorical-sampling approaches to training the same student architecture for multi-step generation.

\item We distill the same LangFlow teacher \citep{chen2026langflow} (170M parameters) with both methods and evaluate generation of 1,024-token sequences on OpenWebText using generative perplexity--entropy frontiers \citep{pynadath2026generativefrontiersevaluationmatters}. The methods improve on the evaluated baselines at complementary sampling budgets: \SDMD at 2--16 network evaluations and \RDMD at 256 evaluations, matching the best baseline at 128 (\Cref{fig:intro-nfe}). At four evaluations and unigram entropy 5.44, \SDMD reduces Gen PPL from 90.2 for ReDi to 45.6; at 256 evaluations and entropy 5.00, \RDMD reduces it from 18.6 for D-MMD to 14.9.

\end{enumerate}

\paragraph{Notation.}
Let $K$ denote the vocabulary size, let $\msv=\{e^1,\dots,e^K\}$ be the canonical basis of $\R^K$ and
$\msx=\msv^L$ the space of sequences of length $L$.
Capital letters denote random variables and lower-case letters their
realizations. We represent a discrete sample by a matrix
$\mathbf{x}\in\R^{L\times K}$ with rows $\mathbf{x}^\ell\in\msv$
and entries $\mathbf{x}^{\ell k}$.
We write $\Delta_K=\{u\in\R^K:u^k\geq0,\ \sum_{k=1}^K u^k=1\}$
for the probability simplex. For $\mathsf{S}\subseteq\R^n$,
$\mathsf{S}^L$ denotes the matrices with $L$ rows in $\mathsf{S}$.
The embedding matrix $\Emat\in\R^{K\times d}$ maps $\mathbf{x}$ to
$\mathbf{x}\Emat$. We denote the stop-gradient operator by $\sg{\cdot}$.

\section{Preliminaries: continuous diffusion for discrete data}
\label{sec:bg-continuous}

Let $p_{\mathrm{data}}$ be a distribution on $\msx$.
Continuous diffusion models for discrete data (CDMd) map each sample
$\mathbf{x}$ to an embedding $\mathbf{x}\Emat \in \R^{L\times d}$,
where $\Emat$ is fixed or learned jointly with the model
\citep{strudel2022self,li2022diffusion}.
They learn to reverse a Gaussian corruption process in this embedding
space, generating samples by progressively denoising Gaussian noise
before applying a categorical decoder to produce a token sequence.
We describe the forward process and the resulting generative model
within the Variational Diffusion Model (VDM) framework
\citep{kingma2021vdm}.

Given $\mathbf{X}\sim p_{\mathrm{data}}$, the forward process $(Z_t)_{t\in[0,1]}$ progressively corrupts the
embedding $\mathbf{X}\Emat$ with Gaussian noise according to a decreasing schedule $(\alpha_t)_{t\in[0,1]}$,
with $\alpha_1=0$, and $\sigma_t^2=1-\alpha_t^2$.
It is a Markov process in $\R^{L\times d}$ with conditional laws
$q_{t\mid\mathbf{x}}(\cdot\mid\mathbf{x})=\gauss(\alpha_t\mathbf{x}\Emat,\sigma_t^2\Id)$.
For $0\le s<t\le1$, its transition kernels are
$q_{t\mid s}(\cdot\mid z_s)=\gauss(\alpha_{t\mid s}z_s,\sigma_{t\mid s}^2\Id)$,
where $\alpha_{t\mid s}=\alpha_t/\alpha_s$ and $\sigma_{t\mid s}^2=\sigma_t^2-\alpha_{t\mid s}^2\sigma_s^2$.
Here, $\Id$ is the identity on vectorized latent matrices, and we denote by $p_t$ the marginal density of $Z_t$.

Fix a grid $ 0 = t_0 < \dots < t_n = 1 $ and write $ s_i = t_{i-1} $.
The reverse generative model is given by
\begin{equation}
    \label{eq:bg-generative}
    p^\theta(\mathbf{x}) = \int p_1(z_1) \, p^\theta_{\mathbf{x} \mid 0}(\mathbf{x} \mid z_0) \prod_{i=1}^n p^\theta_{s_i \mid t_i}(z_{s_i} \mid z_{t_i})  \, \mathrm{d} z_{t_0 : t_n} \, .
\end{equation}
The transitions in this model are given using a denoising network $ \mathbf{x}^\theta(z_t, t) \in \Delta_K^L $ by
$$
    p^\theta_{s \mid t}(z_s \mid z_t) = q_{s \mid t, \mathbf{x}} \big( z_s \mid z_t, \, \mathbf{x}^\theta(z_t, t) \big) \, ,
$$
where $q_{s\mid t,\mathbf{x}}$ is the conditional density of
$Z_s$ given $Z_t=z_t$ and $\mathbf{X}=\mathbf{x}$, that we extend to simplex-valued inputs.
By Bayes' rule,
\begin{equation}
    \label{eq:gaussian_bridge}
    q_{s \mid t, \mathbf{x}}(z_s \mid z_t, \mathbf{x}) = \frac{q_{t \mid s}(z_t \mid z_s) \, q_{s \mid \mathbf{x}}(z_s \mid \mathbf{x})}{q_{t \mid \mathbf{x}}(z_t \mid \mathbf{x})} = \gauss\!\left(
z_s;\,
\frac{\alpha_{t\mid s}\sigma_s^2}{\sigma_t^2}z_t
+
\frac{\alpha_s\sigma_{t\mid s}^2}{\sigma_t^2}\mathbf{x}\Emat,
\,
\frac{\sigma_s^2\sigma_{t\mid s}^2}{\sigma_t^2}\Id
\right) \, .
\end{equation}
The same network sets the conditional likelihood as
$
    p^\theta_{\mathbf{x} \mid 0}(\mathbf{x} \mid z_0) = \prod_{\ell = 1}^L \langle \mathbf{x}^{\theta,\ell}(z_0, 0), \, \mathbf{x}^\ell \rangle \, .
$
The generative model is fitted by minimizing the negative evidence
lower bound (NELBO); as the time-grid mesh tends to zero, it takes the form
\citep{kingma2021vdm,gulrajani2023likelihood}:
\begin{equation}
    \begin{aligned}
    \mathsf{L}_{\mathrm{c}}(\mathbf{x};\theta)
    ={}&\E_{Z_0\sim q_{0\mid\mathbf{x}}(\cdot\mid\mathbf{x})}
       \left[-\sum_{\ell=1}^L
       \log\langle\mathbf{x}^{\theta,\ell}(Z_0, 0),\mathbf{x}^\ell\rangle
       \right]\\
       &+\frac12\int_0^1[-\SNR'(t)]\,
       \E_{Z_t\sim q_{t\mid\mathbf{x}}(\cdot\mid\mathbf{x})}
       \left[\left\|
           \big(\mathbf{x}^\theta(Z_t, t)-\mathbf{x}\big)\Emat
       \right\|_{\mathrm{F}}^2\right]\,\mathrm{d}t\eqsp,
    \end{aligned}
    \label{eq:nelbo-continuous}
\end{equation}
where $\SNR(t)=\alpha_t^2/\sigma_t^2$. 
The diffusion loss is such that, when $t > 0$, the embedded denoiser output $\mathbf{x}^{\theta}\mathbf{E}$ 
targets the posterior mean in embedding space $(t, z_t) \mapsto m_t(z_t)\Emat$, 
where $ m_t(z_t) = \E[\mathbf{X} | Z_t = z_t] \in \Delta_K^L $ 
is the conditional mean of $\mathbf{X}$ given $Z_t$.
For $t = 0$, the reconstruction term
ensures that the denoiser network's target is $m_0$.
Alternatively, the denoiser network can be trained directly by
cross-entropy \citep{dieleman2022continuous,chen2026langflow}:
\begin{equation}
    \mathsf{L}_{\mathrm{CE}}(\mathbf{x};\theta)
    =\E_{\substack{t\sim\Unif(0,1)\\
                  Z_t\sim q_{t\mid\mathbf{x}}(\cdot\mid\mathbf{x})}}
    \left[-\sum_{\ell=1}^L\log
    \langle\mathbf{x}^{\theta,\ell}(Z_t, t),
    \mathbf{x}^\ell\rangle\right]\eqsp,
    \label{eq:bg-cross-entropy}
\end{equation}
for which the target at optimum is $m_t$.
Although not equivalent to the NELBO above, the cross-entropy loss can yield a
likelihood bound with a suitable weighting \citep{davis2026scaling}.
Embedding normalization, self-conditioning, and learned noise schedules are further design
choices detailed in \Cref{app:continuous}.

\section{Distribution matching distillation for continuous diffusion on discrete data}
\label{sec:method}

Distributional distillation trains a student by matching the distribution
of its generated samples to the data distribution. In our setting,
the student takes as input a latent $Z_T$ at noise level $T \in (0, 1]$ and produces a
distribution over clean token representations. These outputs are then
renoised with the Gaussian forward process of \Cref{sec:bg-continuous}, and
their distributions are matched with the corresponding data marginals
$p_t$ at lower noise levels $t<T$. For the approaches considered, the matching objective involves 
the fixed pretrained CDMd $p^\theta$, used as the distillation teacher.

More precisely, given $Z_T=z_T$, the student network outputs
$\mathbf{x}^\eta(z_T,T)\in\Delta_K^L$, which parameterizes a conditional
kernel $\genkernel_T^\eta(\cdot\mid z_T)$. We consider two parameterizations
of this kernel. The first uses the probability vectors themselves as
simplex-valued token representations, while the second samples one
categorical token at each sequence position:
\begin{align}
&\genkernel_T^\eta(\mathrm{d}\mathbf{x}\mid z_T)
= \delta_{\mathbf{x}^\eta(z_T,T)}(\mathrm{d}\mathbf{x}),
\quad \mathbf{x}\in\Delta_K^L,
&&\quad \text{(continuous relaxation)}, \label{eq:student_relaxation} \\
&\genkernel_T^\eta(\mathbf{x}\mid z_T)
= \prod_{\ell=1}^L
\left\langle
\mathbf{x}^\ell,\mathbf{x}^{\eta,\ell}(z_T,T)
\right\rangle,
\quad \mathbf{x}\in\msx,
&&\quad \text{(discrete sampling)}.
\label{eq:student_sampling}
\end{align}
These two parameterizations share the same distribution-matching objective
but lead to different gradient estimators.
We formulate this objective for one-step generation in a framework unifying existing methods by considering a general class of matching objectives, 
extend it to multi-step generation and derive
\SDMD for the continuous relaxation and \RDMD for discrete sampling.

\subsection{A unified view of one-step distribution matching}
\label{sec:method-unified}

We first consider one-step generation and isolate a common formulation of
distributional distillation. Existing distribution-matching methods can be
viewed as comparing the noised student and data distributions across noise
levels, while differing in the discrepancy used to compare these marginals.
We formalize this view through a generic discrepancy $\mathsf{D}$, before
specializing it to the reverse KL used by our methods.

In the one-step setting, $T=1$ and the student starts from
$Z_1\sim p_1=\gauss(0,\Id)$. Its clean output
$\mathbf{X}^\eta$ is sampled according to
$\genkernel_1^\eta(\cdot\mid Z_1)$, inducing the marginal distribution
$
p_{\mathbf{x}}^\eta
=
\int
\genkernel_1^\eta(\cdot\mid z_1)\,
p_1(z_1)\,\mathrm{d}z_1
.
$

Applying the Gaussian forward process of \Cref{sec:bg-continuous} to $\mathbf{X}^\eta \sim p_{\mathbf{x}}^\eta$ gives a latent process $(Z_t^\eta)_{t\in[0,1]}$. We denote its marginals by $p_t^\eta$, with the corresponding data
marginal being $p_t$. Comparing
$p_{\mathbf{x}}^\eta$ with $p_{\mathrm{data}}$ is difficult; Gaussian noising smooths both distributions
in the same embedding space, where the diffusion teacher $p^{\theta}$
provides quantities that can be leveraged by the loss objectives.
Matching these marginals across noise levels defines the generic one-step distribution-matching objective:
\begin{equation}
    \mathsf{L}_{\mathrm{distill}}(\eta)
    = \int_0^1 
    \E_{Z_t^\eta\sim p_t^\eta}\!\left[
        \mathsf{D}(p_t^\eta,p_t,Z_t^\eta)
    \right]\,\mathrm{d}t \eqsp,
    \label{eq:wip-distill-marginals}
\end{equation}
where $\mathsf{D}$ denotes a local discrepancy between
student and data marginals evaluated at $z_t$.
Different choices of $\mathsf{D}$ recover distribution-matching objectives
used in prior work, as summarized in
\Cref{tab:wip-local-discrepancies}; full objectives and derivations are given in \Cref{app:distillation}.
The gradient of \eqref{eq:wip-distill-marginals} has two contributions: $ (A) $ from the dependence of the sampling law on $ \eta $, and $ (B) $ from the dependence of $ \mathsf{D} $ on $ p_t^\eta $:
\begin{equation}
\begin{aligned}
    \nabla_\eta\mathsf{L}_{\mathrm{distill}}(\eta)
    &= \int_0^1 \bigg[\underbrace{\!\int
        \mathsf{D}(p_t^\eta,p_t,z_t)\,\nabla_\eta p_t^\eta(z_t)
        \,\mathrm{d}z_t}_{(A)} +\underbrace{\!\int
        p_t^\eta(z_t)\,\nabla_\eta\mathsf{D}(p_t^\eta,p_t,z_t)
        \,\mathrm{d}z_t}_{(B)}\bigg]\, \mathrm{d}t.
\end{aligned}
    \label{eq:wip-distill-grad}
\end{equation}

\vspace{-15pt}
\begin{table}[htbp]
    \centering
    \small
    \setlength{\tabcolsep}{4pt}
    \renewcommand{\arraystretch}{0.8} %
    \captionsetup{font=footnotesize}
    \begin{tabular}{@{}>{\raggedright\arraybackslash}p{0.25\linewidth}>{\raggedright\arraybackslash}p{0.37\linewidth}>{\raggedright\arraybackslash}p{\dimexpr0.38\linewidth-4\tabcolsep\relax}@{}}
        \toprule
        Criterion & Local discrepancy $\mathsf{D}$ & Representative methods \\
        \midrule
        \multicolumn{3}{@{}l}{\textit{Continuous diffusion}} \\
        Reverse KL
            & $\log p_t^\eta(z_t)-\log p_t(z_t)$
            & Diff-Instruct \citep{luo2023diffinstruct}, DMD \citep{yin2024onestep} \\
        Score / Velocity / Moment matching
            & $\big\lVert\nabla_{z_t}\log p_t^\eta(z_t)-\nabla_{z_t}\log p_t(z_t)\big\rVert^2$
            & SiD \citep{zhou2024score};
              FGM \citep{huang2024flow};
              Moment matching \citep{salimans2024multistep} \\
        \midrule
        \multicolumn{3}{@{}l}{\textit{Discrete diffusion}} \\
        \par Bregman divergence
            & $\displaystyle\sum_{d_{\mathrm{H}}(z',z_t)=1} d_F\!\left(\frac{p_t^\eta(z')}{p_t^\eta(z_t)},\,\frac{p_t(z')}{p_t(z_t)}\right)$
            & IDLM \citep{li2026idlm}; D-MMD \citep{hoogeboom2026dmmd} \\
        \par Posterior $f$-divergence
            & $\displaystyle\sum_\ell D_f\!\left(\post[\eta,\ell](\cdot\mid z_t),\,\post[\ell](\cdot\mid z_t)\right)$
            & DiMO \citep{zhu2025dimo} \\
        \bottomrule
    \end{tabular}
        \caption{Local discrepancies for distributional distillation.
        Score, velocity and moment discrepancies agree up to
        time-dependent factors. Definitions and derivations are given in \Cref{app:distillation}.}
    \label{tab:wip-local-discrepancies}
    \end{table}

Following Diff-Instruct and DMD \citep{luo2023diffinstruct, yin2024onestep}, we focus in this work on the reverse KL,
choosing $ \mathsf{D}(p_t^\eta,p_t,z_t) = \log\big(p_t^\eta(z_t)/p_t(z_t)\big) $. The loss thus becomes:
\begin{equation}
    \mathsf{L}_{\mathrm{DMD}}(\eta)
    = \int_0^1 \KL{p_t^\eta}{p_t}\,\mathrm{d}t.
    \label{eq:ikl}
\end{equation}
In this case, (B) integrates to zero (\Cref{app:dmd-grad}), leaving only the dependence of the sampled student marginal on $\eta$ in (A).
Estimating this term depends on the student parameterization:
the continuous relaxation admits a pathwise gradient,
whereas discrete sampling requires a score function
estimator. We derive the two estimators in \Cref{sec:method-simplex,sec:method-categorical}.

The previous objective \eqref{eq:wip-distill-marginals} matches the student and data distributions
only after embedding and Gaussian noising. 
It is therefore not immediate that exact matching of the noised
marginals recovers the original discrete data distribution.
The following result shows that this is nevertheless the case, under a simple
geometric condition on the embedding matrix that we satisfy.

\begin{theorem}
\label{thm:emergent}
Assume that the rows of $\Emat$ are distinct and have a common Euclidean norm.
Then
\begin{equation}
    \mathsf{L}_{\mathrm{DMD}}(\eta)=0 \qquad\Longleftrightarrow\qquad p_{\mathbf{x}}^\eta=p_{\mathrm{data}}.
\end{equation}
\end{theorem}
We give the proof in \Cref{app:emergent}.

\subsection{Multi-step generation}
\label{sec:method-fewstep}

The one-step construction of \Cref{sec:method-unified} learns to transport
the terminal marginal $p_1$ to the data distribution. In order to extend the method to multi-step generation, 
we adapt this construction to arbitrary starting noise levels $T\in(0,1]$: given
$Z_T\sim p_T$, the student, conditioned on $T$, produces a clean sample
$\mathbf{X}^{\eta,T}\sim\genkernel_T^\eta(\cdot\mid Z_T)$, and we denote its marginal by $p_{\mathbf{x}}^{\eta, T}$. 
Renoising this sample to a noise level $t<T$ with the forward process of \Cref{sec:bg-continuous} yields the latent states
$(Z_t^{\eta,T})$, inducing marginals $p_t^{\eta,T}$, which
we match to the data marginal $p_t$ using the same reverse-KL
objective as in the one-step case \eqref{eq:ikl}. 
Varying $T$ yields the multi-step objective:
\begin{equation}
    \mathsf{L}_{\mathrm{DMD}}^{\mathrm{ms}}(\eta)
    = \E_{(t,T)\sim\nu} \left[
        \KL{p_t^{\eta,T}}{p_t}
        \right] \, \eqsp,
    \label{eq:fewstep-proposal-loss}
\end{equation}
where the joint distribution $\nu$ over $(t, T)$ is defined as $\nu(\mathrm{d}t,\mathrm{d}T)=\nu_T(\mathrm{d}t)\,\nu_1(\mathrm{d}T)$, with $\nu_s=\Unif(0,s)$ for $s\in(0,1]$. 
Learning these transports for all starting
levels enables multi-step generation by composing them along decreasing
noise levels $1=T_n>T_{n-1}>\cdots>T_0=0$.
Starting from $\widehat Z_{T_n}\sim p_1$, at each level $T_i$ the student draws
$\widehat{\mathbf{X}}^{\eta,T_i}\sim
\genkernel_{T_i}^\eta(\cdot\mid \widehat Z_{T_i})$ and maps it to the next level $T_{i-1}$. 
We explore the following family of DDIM-inspired transitions \citep{song2021ddim}, a standard approach for exploring multi-step performance, which controls the stochasticity of the sampling procedure:
\begin{equation}
    p_{t\mid T}^\eta(z_t\mid z_T)
    =
    \int
        q^\zeta_{t\mid\mathbf{x},T}
        (z_t\mid\mathbf{x},z_T)\,
        \genkernel_T^\eta(\mathrm{d}\mathbf{x}\mid z_T)
    \eqsp,
    \qquad t<T,
    \label{eq:fewstep-transition}
\end{equation}
where, for $0\leq\zeta_{t,T}\leq\sigma_t$,
\begin{equation}
\begin{aligned}
    q^\zeta_{t\mid\mathbf{x},T}
    (z_t\mid\mathbf{x},z_T)
    =
    \gauss\!\left(
        z_t;\,
        \alpha_t\mathbf{x}\Emat
        +\sqrt{\sigma_t^2-\zeta_{t,T}^2}\,
        \frac{z_T-\alpha_T\mathbf{x}\Emat}{\sigma_T},
        \zeta_{t,T}^2\Id
    \right)
    \eqsp.
\end{aligned}
\label{eq:fewstep-transition-kernel}
\end{equation}
The parameter $\zeta_{t,T}$ interpolates between retaining the noise inferred
from the current latent and injecting fresh Gaussian noise.
At $\zeta_{t,T}=0$, the update is
deterministic conditional on $\mathbf{x}$. 
At $\zeta_{t,T}=\sigma_t$, it discards the current latent noise and recovers the
forward-renoising transition used to define $Z_t^{\eta, T}$, whereas
$\zeta_{t,T}=\sigma_t\sigma_{T\mid t}/\sigma_T$ recovers the Gaussian
bridge \eqref{eq:gaussian_bridge}. 
Finally, the generative model takes a final step from $\widehat Z_0$ by applying the categorical decoder parameterized by
$\mathbf{x}^\eta(\widehat Z_0,0)$.
In the forward renoising $\zeta_{t,T}=\sigma_t$ case, these transitions recover the correct diffusion marginals $p_t$ 
if $\widehat Z_T\sim p_T$ and the induced clean student marginal equals $p_{\mathrm{data}}$; 
in general, the diffusion marginals are also recovered if the student kernel equals the exact posterior $p_{\mathbf{x}\mid T}(\cdot\mid z_T)$, see \Cref{app:fewstep-sampling}.

\subsection{\SDMD: continuous relaxation}
\label{sec:method-simplex}

Under the continuous relaxation \eqref{eq:student_relaxation}, the student's
probability vectors are used as the clean sample:
$\mathbf{X}^{\eta,T}=\mathbf{x}^\eta(Z_T,T)$. The noising process is differentiable with respect
to the student parameters, allowing us to optimize the multi-step
distribution-matching objective with a pathwise gradient.

\paragraph{Student gradient.}
We first derive the gradient of the matching objective
\eqref{eq:fewstep-proposal-loss}. Define the reparameterized forward-noising map 
$\pi_t : (\mathbf{x},\varepsilon) \mapsto \alpha_t\mathbf{x}\Emat+\sigma_t\varepsilon$, such that
$Z_t^{\eta,T} \eqd \pi_t(\mathbf{x}^\eta(Z_T,T),\varepsilon)$ for $Z_T \sim p_T$ and $\varepsilon \sim \gauss(0, \Id)$.
This makes explicit the dependence of the noised latent on $\eta$ and how the reparameterization trick \citep{kingma2014autoencoding} applies to \eqref{eq:fewstep-proposal-loss}, similarly to Diff-Instruct and DMD
\citep{luo2023diffinstruct,yin2024onestep}:
\begin{equation}
\begin{aligned}
\nabla_\eta
\mathsf{L}_{\mathrm{DMD}}^{\mathrm{ms}}(\eta)
=
\E_{\substack{
    (t,T)\sim\nu,\,
    Z_T\sim p_T,\,
    \varepsilon\sim\gauss(0,\Id)
}}
\left[
    \Big(
        s^{\eta,T}(Z_t^{\eta,T},t)
        -s(Z_t^{\eta,T},t)
    \Big)^\top
    \frac{\partial Z_t^{\eta,T}}{\partial\eta}
\right]
\eqsp,
\end{aligned}
\label{eq:dmd-grad}
\end{equation}
where $s^{\eta,T}(z,t)=\nabla_z\log p_t^{\eta,T}(z)$ and
$s(z,t)=\nabla_z\log p_t(z)$ are respectively the student and data scores
(see \Cref{app:dmd-grad}).
Thus, the pathwise update requires evaluating the difference between the
score of the noised student distribution and that of the data distribution.

\paragraph{Auxiliary score estimation.}
Neither score in \eqref{eq:dmd-grad} is available exactly, but under Gaussian
noising, each is determined by the posterior mean of its corresponding clean
representation. 
This relation is given by Tweedie's formula \citep{efron2011tweedie}:
\begin{equation}
    s(z,t)
    =
    \frac{
        \alpha_t m_t(z)\Emat-z
    }{\sigma_t^2}
    \eqsp, \qquad
    s^{\eta,T}(z,t)
    =
    \frac{
        \alpha_t m_t^{\eta,T}(z)\Emat-z
    }{\sigma_t^2}
    \eqsp,
    \label{eq:sdmd-student-score}
\end{equation}
where $m_t^{\eta,T}(z)=\E\!\left[\mathbf{X}^{\eta,T}\mid Z_t^{\eta,T}=z\right]$
is the posterior mean of the student distribution. 
The frozen teacher denoiser
$\mathbf{x}^\theta(z,t)$ provides an estimate of the data
posterior mean, and hence an estimate $s^{\theta}(z,t)$ of $s(z,t)$.
We train an auxiliary denoiser
$\mathbf{x}^\phi(z,t,T)\in\Delta_K^L$ to target the posterior mean 
$m_t^{\eta,T}$ of the student distribution, by minimizing the following soft-target cross-entropy:
\begin{equation}
\begin{aligned}
\mathsf{L}_{\mathrm{aux}}(\phi)
=
\E_{\substack{
    (t,T)\sim\nu,\,
    Z_T\sim p_T,\,
    \varepsilon\sim\gauss(0,\Id)
}}
\left[
    \sum_{\ell=1}^L
    \mathrm{CE}\!\left(
        \mathbf{X}^{\eta,T,\ell},
        \mathbf{x}^{\phi,\ell}
        (Z_t^{\eta,T},t,T)
    \right)
\right]
\eqsp.
\end{aligned}
\label{eq:aux-loss}
\end{equation}

We alternate $n_{\mathrm{aux}}$ auxiliary updates on detached
student outputs with one student update.

\subsection{\RDMD: discrete sampling}
\label{sec:method-categorical}

The student clean sample $\mathbf{X}^{\eta,T}$ is now drawn from a discrete distribution given by the transition
$\genkernel_T^\eta(\cdot\mid Z_T)$ of \eqref{eq:student_sampling},
preventing the direct pathwise differentiation used by \SDMD.

\paragraph{Student gradient.} The gradient of the multi-step objective \eqref{eq:fewstep-proposal-loss}
admits the same decomposition as that of the one-step objective derived in \eqref{eq:wip-distill-grad}. 
In our reverse KL matching setting, contribution $(B)$ vanishes, 
so it remains to differentiate the conditional law of
$\mathbf{X}^{\eta,T}$ which gives the exact
score-function identity \citep{williams1992simple}:
\begin{equation}
\begin{aligned}
    \nabla_\eta\mathsf{L}_{\mathrm{DMD}}^{\mathrm{ms}}(\eta)
    &=\E_{\substack{T \sim \nu_1, Z_T\sim p_T\\\mathbf{X}^{\eta,T}\sim\genkernel_T^\eta(\cdot\mid Z_T)}}\!\left[
        C_T(\mathbf{X}^{\eta,T})\,
        \nabla_\eta\log\genkernel_T^\eta(\mathbf{X}^{\eta,T}\mid Z_T)
    \right]\eqsp,\\
    C_T(\mathbf{X}^{\eta, T})
    &=\E_{\substack{t\sim\nu_T\\Z_t^{\eta, T}\sim q_{t\mid\mathbf{x}}(\cdot\mid\mathbf{X}^{\eta, T})}}\!\left[
        \log\frac{p_t^{\eta,T}(Z_t^{\eta, T})}
                          {p_t(Z_t^{\eta, T})}
    \right].
\end{aligned}
    \label{eq:reinforce-grad}
\end{equation}
The cost $C_T$ is the expected log-ratio of the student and data marginals
at time $t$. Unlike the scores in the \SDMD gradient, which
diffusion denoisers can estimate, this log-ratio is not directly available and requires
another estimator.

\paragraph{Auxiliary estimation.} In order to estimate this log-ratio of interest, 
we train a discriminator $D_\psi(z_t, t, T)\in(0,1)$ to
distinguish noised data from noised student samples:
\begin{equation}
    \mathsf{L}_{\mathrm{disc}}(\psi)
    =-\int\Big[
        \E_{Z_t\sim p_t}\log D_\psi(Z_t, t, T)
        +\E_{Z_t^{\eta,T}\sim p_t^{\eta,T}}
          \log\big(1-D_\psi(Z_t^{\eta,T}, t, T)\big)
    \Big]\, \nu(\mathrm{d}t,\mathrm{d}T) \, .
    \label{eq:disc-loss}
\end{equation}
For a fixed student, the population optimum is
$D^\star(z, t, T)=p_t(z)/(p_t(z)+p_t^{\eta,T}(z))$,
whose negative logit equals
$\log(p_t^{\eta,T}/p_t)$ (\Cref{app:discriminator}).
For a pair $0<t<T\leq1$ and independent noise
$\varepsilon\sim\gauss(0,\Id)$, the pointwise cost estimate is:
\begin{equation}
    \widehat C_\psi(\mathbf{X}^{\eta,T};t,T,\varepsilon)
    =
      \log\frac{1-D_\psi(\pi_t(\mathbf{X}^{\eta,T},\varepsilon), t, T)}
                   {D_\psi(\pi_t(\mathbf{X}^{\eta,T},\varepsilon), t, T)}.
    \label{eq:reinforce-cost}
\end{equation}

\paragraph{Variance reduction and teacher supervision.}
For each $(Z_T,T)$, we draw $G\geq2$ independent triplets
$(\mathbf{X}_g^{\eta,T},t_g,\varepsilon_g)$ with
$\mathbf{X}_g^{\eta,T}\sim\genkernel_T^\eta(\cdot\mid Z_T)$,
$t_g\sim \nu_T$, and $\varepsilon_g\sim\gauss(0,\Id)$.
We compute the costs
$\widehat C_g=\widehat C_\psi(\mathbf{X}_g^{\eta,T};t_g,T,\varepsilon_g)$
and the leave-one-out baselines
$\widehat b_g=(G-1)^{-1}\sum_{g'\neq g}\widehat C_{g'}$ \citep{kool2019buy}.
Inspired by the group-based baseline in GRPO,
we center each cost by subtracting $\widehat b_g$ \citep{shao2024deepseekmath},
and include a KL anchor required for stable training in our experiments
(\Cref{sec:exp-new-design}), yielding the per-group surrogate:
\begin{equation}
\begin{aligned}
    \widehat{\mathsf{L}}_{\mathrm{R\text{-}DMD}}(\eta)
    &=\frac1{G}\sum_{g=1}^G
       \sg{\widehat C_g-\widehat b_g}
       \log\genkernel_T^\eta(\mathbf{X}_g^{\eta,T}\mid Z_T)
    +\beta\sum_{\ell=1}^L
       \KL{\mathbf{x}^{\eta,\ell}(Z_T,T)}
          {\mathbf{x}^{\theta,\ell}(Z_T,T)},
\end{aligned}
    \label{eq:teacher-anchor}
\end{equation}
where $\beta\geq0$.
We average this surrogate over sampled $(Z_T, T)$ and their groups.
The leave-one-out baseline preserves the expected cost-based update. We use $G=4$, requiring one student
evaluation and $G$ cost evaluations per group.

We alternate $n_{\mathrm{aux}}$ discriminator updates on
detached student samples with one student update using
\eqref{eq:teacher-anchor}. During student updates, the discriminator,
costs and baselines are detached.

\subsection{Related work}
\label{sec:related-work}

\paragraph{Continuous diffusion for discrete data.}
Diffusion-LM jointly learns embeddings and a diffusion model \citep{li2022diffusion}, SED uses self-conditioning in embedding diffusion \citep{strudel2022self}, and Plaid develops a likelihood-based training framework \citep{gulrajani2023likelihood}.
Alongside these developments, CDCD introduced score interpolation, parameterizing the score through an estimate of the posterior mean \citep{dieleman2022continuous}.
More recently, LangFlow \citep{chen2026langflow} and RePlaid \citep{yang2026replaid} refined noise scheduling and training, achieving language modeling performance competitive with that of discrete diffusion models.

\paragraph{Trajectory distillation.}
\emph{Consistency models} learn to map states along a probability-flow trajectory to a common endpoint, enabling generation with one or a few model evaluations \citep{song2023consistency, song2024improved}.
\emph{Flow map matching} generalizes this construction to maps between arbitrary times, with objectives for distilling pretrained models as well as training directly from data through self-distillation \citep{boffi2025consistency, boffi2025flowmap}.
\emph{Categorical flow maps}, including CFM \citep{roos2026categorical}, DFM \citep{potaptchik2026dfm}, and FMLM \citep{lee2026fmlm}, build on CDMd to extend these ideas to discrete data, enabling one- and few-step generation.
For discrete diffusion, SDTT matches predictions along teacher rollouts \citep{deschenaux2025sdtt}, while Di4C distills transition compositions into mixture models that capture correlations across dimensions \citep{hayakawa2025di4c}.
A complementary approach is taken by Rectified Flow \citep{liu2023rectified} and a discrete adaptation ReDi \citep{yoo2025redi}, which iteratively rectifies the source-target coupling in flow matching to accelerate sampling.

\paragraph{Distributional distillation.}
Distributional distillation methods differ primarily in their choice of local discrepancy and how they estimate the loss gradient.
\Cref{tab:wip-local-discrepancies} summarizes these methods, and \Cref{app:distillation} provides their derivations.

\section{Experiments}
\label{sec:experiments-new}

We evaluate whether \SDMD and \RDMD preserve generation quality and diversity
when distilling a continuous diffusion language model to reduced sampling
budgets. We compare them with diffusion and distillation baselines across
network evaluation budgets, and ablate the main training and sampling choices.

We distill the released LangFlow model \citep{chen2026langflow} on
OpenWebText, using sequences of length $1024$ over the GPT-2 \citep{radford2019language} vocabulary
($K=50{,}257$). We keep the teacher embedding matrix $\Emat$ fixed and
initialize both students and their auxiliary networks from the teacher.
Training runs for $10{,}000$ total iterations with $n_{\mathrm{aux}}=4$ auxiliary
updates per student update, i.e., $2{,}000$ student updates, requiring approximately $40$
H100 GPU hours.

We evaluate generation quality with generative perplexity (Gen PPL, lower is better), computed
by GPT-2 Large, and diversity with unigram entropy (higher is better). Since lower perplexity can
result from reduced diversity, we compare Gen PPL--entropy frontiers obtained
by varying the sampling temperature
\citep{pynadath2026generativefrontiersevaluationmatters}. Baselines include
the LangFlow teacher, the continuous-diffusion distillation method FMLM,
the discrete diffusion models and distillation methods MDLM, D-MMD, IDLM,
SDTT, and ReDi (see \Cref{sec:related-work}), and the autoregressive references
GPT-2 (124M) and OPT (125M) \citep{zhang2022optopenpretrainedtransformer}. For the autoregressive baselines, we use checkpoints from Hugging Face with parameter counts closest to teacher's. All methods use a common evaluation protocol.
\Cref{tab:training-budgets} reports model sizes and training budgets, while
\Cref{app:supp} and \Cref{tab:recipe} provide further evaluation and configuration
details.

\subsection{Generation quality across sampling budgets}
\label{sec:exp-new-frontiers}

\Cref{fig:exp-proposal-frontiers} compares generative frontiers at
representative low and high sampling budgets. At $4$ network evaluations,
\SDMD achieves lower Gen PPL than the evaluated diffusion baselines throughout the
overlapping entropy range. At $256$ evaluations, \RDMD achieves the lowest
Gen PPL among the diffusion methods at matched entropy around
$5.0$. The two methods therefore improve generation quality in
complementary sampling regimes: \SDMD at low budgets and \RDMD at larger
budgets.

\begin{figure}[!htbp]
    \captionsetup{font=footnotesize}
    \centering
    \begin{subfigure}[t]{0.45\linewidth}
        \centering
        \includegraphics[width=\linewidth]{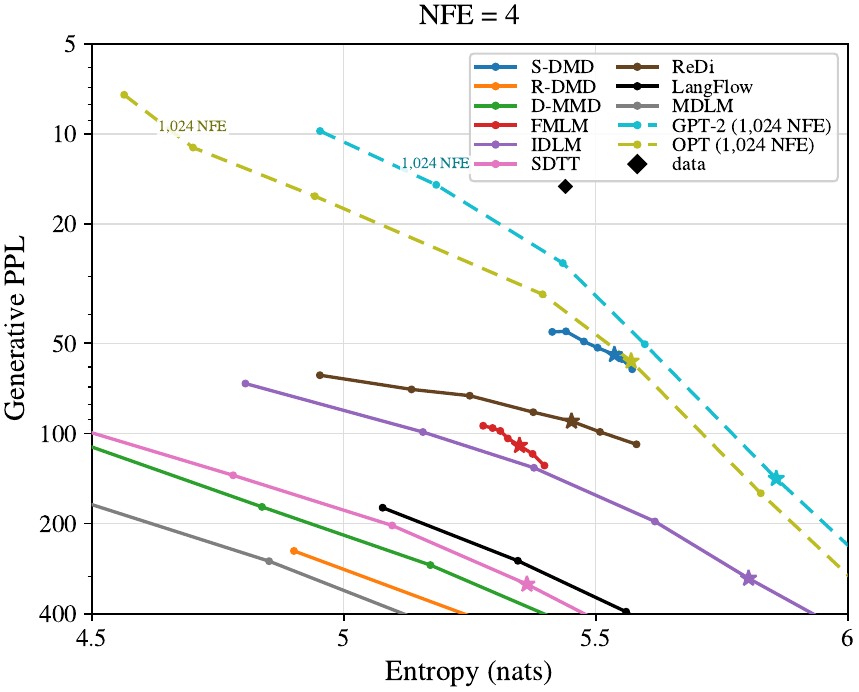}
        \subcaption{4 network evaluations.}
        \label{fig:exp-proposal-frontier-low}
    \end{subfigure}%
    \hspace{0.03\linewidth}%
    \begin{subfigure}[t]{0.45\linewidth}
        \centering
        \includegraphics[width=\linewidth]{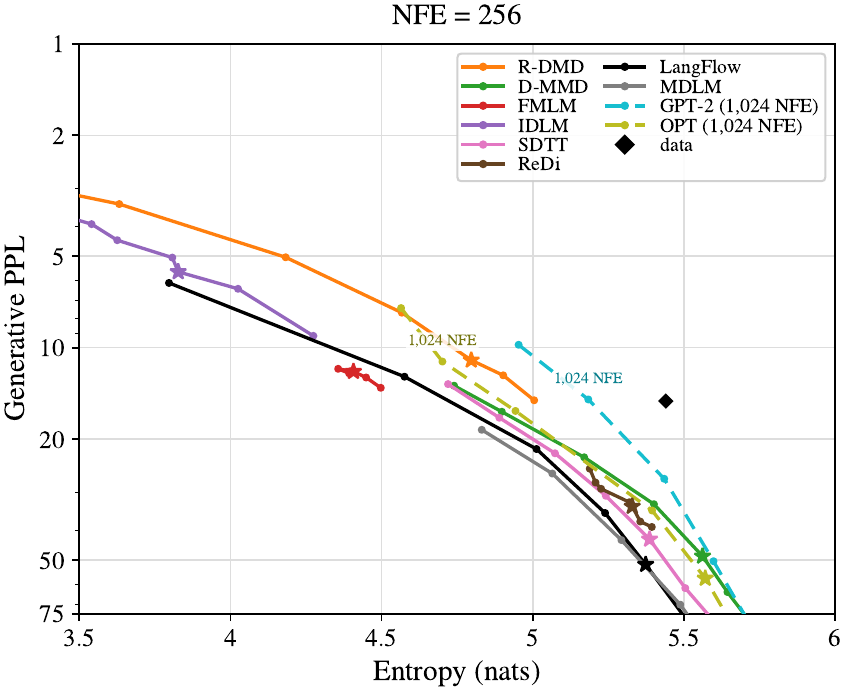}
        \subcaption{256 network evaluations.}
        \label{fig:exp-proposal-frontier-high}
    \end{subfigure}
    \caption{
    Generative frontiers on OpenWebText at $\NFE=4$ and $256$.
Temperature varies from $0.8$ to $1.1$; stars mark temperature $1.0$ and
diamonds the data. Gen PPL uses an inverted log scale. Dashed GPT-2 and OPT use $1{,}024$ evaluations.
\SDMD\ is omitted at $\NFE=256$ due to low diversity.
Full sweeps are in \Cref{app:frontiers}.}
    \label{fig:exp-proposal-frontiers}
\end{figure}

At $\NFE=4$, matching the entropy at $5.44$, \SDMD reaches a
Gen PPL of $45.6$, compared with $90.2$ for ReDi, corresponding to a
$49\%$ reduction. At $\NFE=256$ and matched entropy $5.00$,
\RDMD reaches a Gen PPL of $14.9$, compared with $18.6$ for D-MMD,
a $20\%$ reduction.

With just $4$ network evaluations, \SDMD\ approaches the quality-diversity frontier of OPT-125M which uses $1{,}024$ NFEs, bringing autoregressive-level performance within reach of few-step diffusion generation.
At matched unigram entropy \(H\approx5.54\), their Gen PPL values are \(54.6\) and \(52.1\), respectively.
\RDMD\ also compares favorably with OPT: at \(H\approx5.00\), it achieves Gen PPL \(14.9\) versus \(17.9\), using 256 evaluations.

To compare methods across sampling budgets more systematically,
\Cref{tab:exp-proposal-budgets} reports Gen PPL at a budget-specific
matched-entropy target $H_{\NFE}$. We linearly interpolate log Gen PPL
between measured frontier points that bracket each target.
\SDMD obtains the lowest matched-entropy Gen PPL at $2$, $4$, $8$, and $16$
network evaluations, whereas \RDMD matches D-MMD at $128$ evaluations
and obtains the lowest value at $256$. At $\NFE=2$, FMLM reaches a lower raw Gen PPL
than \SDMD ($142$ versus $178$), but only at a substantially lower entropy
($5.32$ versus $5.65$), and is therefore not a matched-diversity comparison.

\begin{table}[t]
    \centering
    \small
    \renewcommand{\arraystretch}{0.8} %
    \captionsetup{font=footnotesize}
    \setlength{\tabcolsep}{3pt}
    \begin{tabular}{lllllll}
        \toprule
        & \multicolumn{6}{c}{Network evaluations} \\
        \cmidrule(lr){2-7}
        Method & 2 & 4 & 8 & 16 & 128 & 256 \\
        \midrule
        $H_{\mathrm{NFE}}$ & 5.65 & 5.44 & 5.38 & 5.20 & 5.27 & 5.00 \\
        \midrule
        LangFlow \citep{chen2026langflow} & 1077 & 317 & 147 & 58.6 & 39.1 & 21.3 \\
        MDLM \citep{sahoo2024simple} & 1909 & 678 & 267 & 72.8 & 44.3 & 23.8 \\
        \midrule
        FMLM \citep{lee2026fmlm} & 142\entval{5.32} & 128\entval{5.40} & 94.7 & 55.2 & 22.6\entval{4.87} & 13.5\entval{4.50} \\
        D-MMD \citep{hoogeboom2026dmmd} & 1752 & 431 & 135 & 41.7 & 27.3 & 18.6 \\
        IDLM \citep{li2026idlm} & 955 & 145 & 48.4 & 24.3 & 11.6\entval{4.75} & 9.14\entval{4.27} \\
        SDTT \citep{deschenaux2025sdtt} & 1496 & 372 & 103 & 39.4 & 32.6 & 20.1 \\
        ReDi \citep{yoo2025redi} & 261\entval{5.46} & 90.2 & 49.4 & 32.4\entval{5.32} & 30.7 & 25.0\entval{5.19} \\
        \midrule
        \SDMD & \textbf{178} & \textbf{45.6} & \textbf{26.7} & \textbf{17.0} & 3.60\entval{3.82} & 2.51\entval{3.31} \\
        \RDMD & 1112 & 556 & 347 & 68.8 & \textbf{26.9} & \textbf{14.9} \\
        \bottomrule
    \end{tabular}
    \caption{Gen PPL on OpenWebText at matched unigram entropy $H_{\NFE}$;
lower is better. $H_{\NFE}$ is the \SDMD\ entropy closest to the data for
$\NFE=2$--$16$, and the corresponding \RDMD\ entropy for $\NFE=128,256$.
Log Gen PPL is linearly interpolated at $H_{\NFE}$. If a curve does not reach
the target entropy, its closest point is reported with entropy in parentheses.
Bold indicates the lowest matched value.}
    \label{tab:exp-proposal-budgets}
\end{table}

\subsection{Training and sampling choices}
\label{sec:exp-new-design}

We detail the design choices made to obtain the results presented in the previous subsection,
and the related ablations validating them.

\par\medskip\noindent
\begin{minipage}[t]{0.49\linewidth}
\vspace{0pt}
\textbf{Teacher supervision stabilizes \RDMD.}
\Cref{fig:exp-proposal-anchor} examines the effect of the teacher KL anchor introduced in
\eqref{eq:teacher-anchor} at $\NFE=128$. Removing the anchor ($\beta=0$) leads to collapse,
with essentially zero unigram entropy. Positive anchor weights and an annealed schedule prevent
collapse and yield similar generative perplexity--entropy frontiers, with $\beta=0.5$ providing a
modest improvement over the other tested settings. These results indicate that teacher supervision is
needed for stable \RDMD training in the tested configurations, while performance is relatively insensitive
to the anchor weight within the positive range considered. We therefore use $\beta=0.5$ in our default configuration.
\end{minipage}\hfill
\begin{minipage}[t]{0.45\linewidth}
    \vspace{0pt}
    \centering
    \captionsetup{type=figure,font=footnotesize}
    \includegraphics[width=\linewidth]{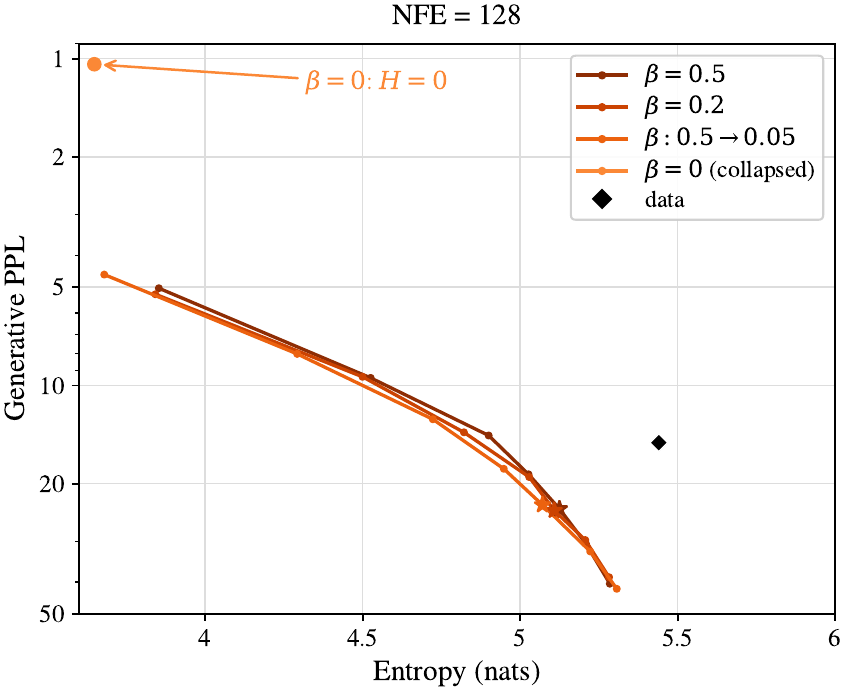}
    \captionof{figure}{Effect of the KL anchor weight $\beta$ on \RDMD\ at
    $\NFE=128$. Without the anchor, generation collapses ($H=0$, off-scale).}
    \label{fig:exp-proposal-anchor}
\end{minipage}
\par\medskip

\begin{figure}[!t]
    \centering
    \captionsetup{font=footnotesize}
    \begin{subfigure}[t]{0.46\linewidth}
        \centering
        \includegraphics[width=0.49\linewidth]{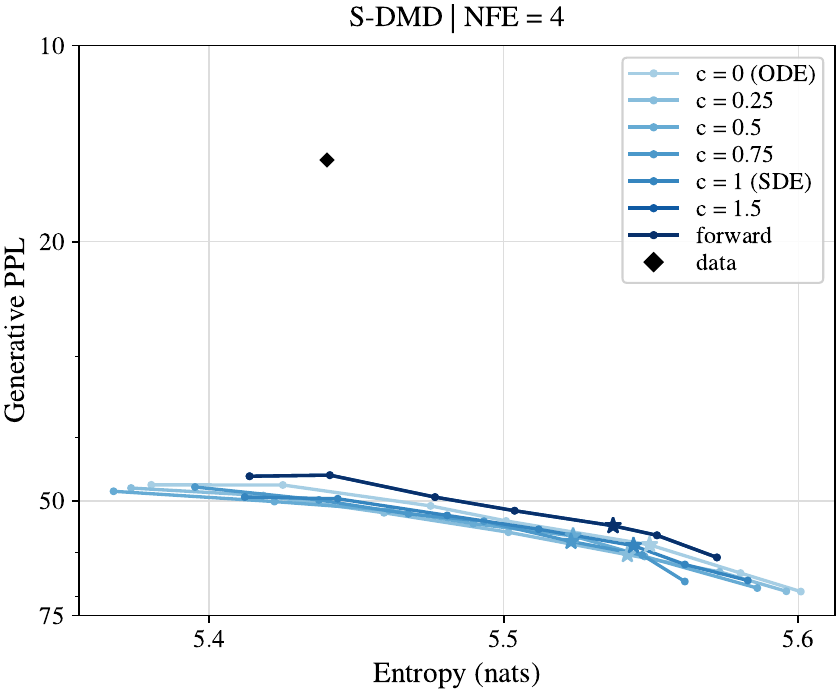}%
        \hfill
        \includegraphics[width=0.49\linewidth]{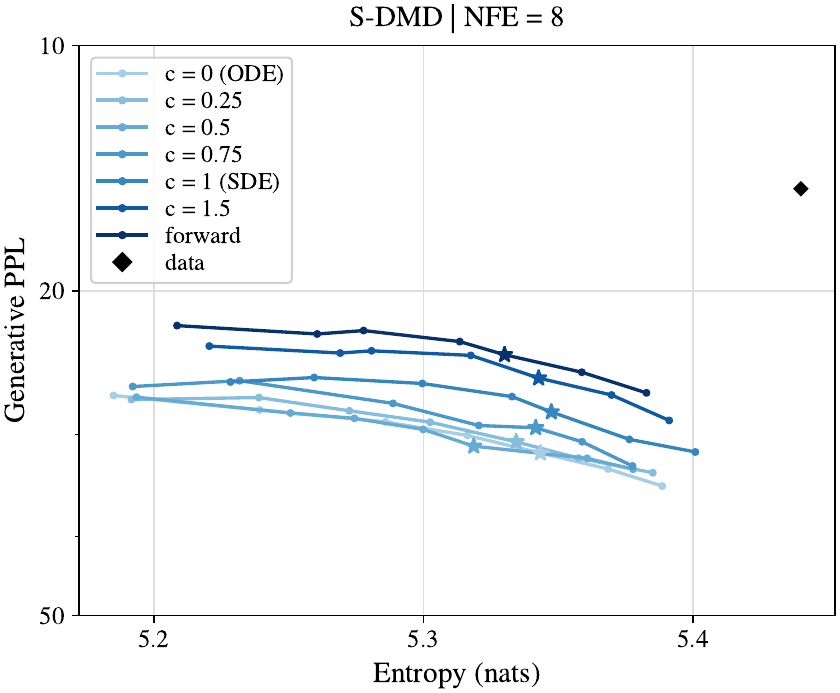}
        \subcaption{\SDMD, $\NFE = 4$ and $8$}
    \end{subfigure}%
    \hspace{0.06\linewidth}%
    \begin{subfigure}[t]{0.46\linewidth}
        \centering
        \includegraphics[width=0.49\linewidth]{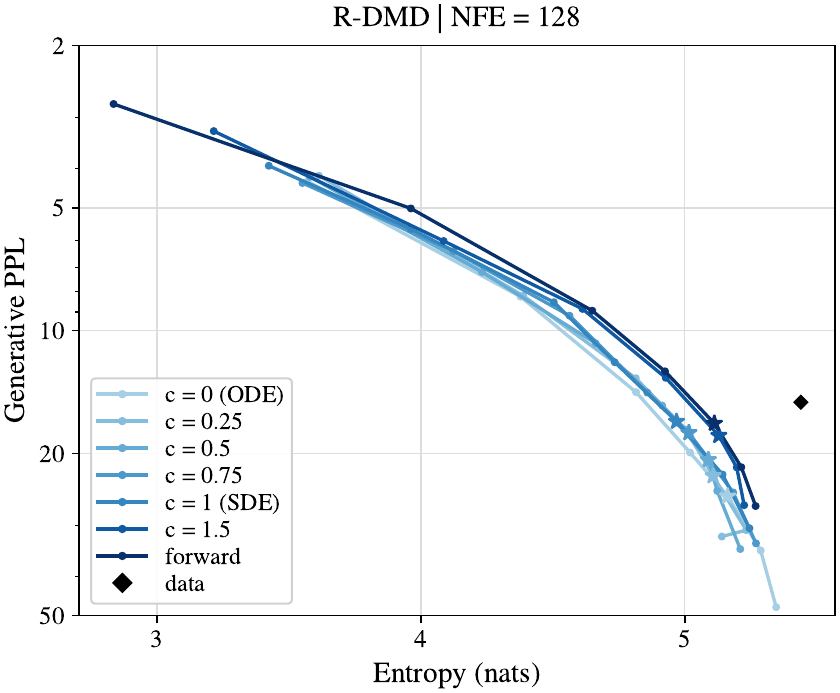}%
        \hfill
        \includegraphics[width=0.49\linewidth]{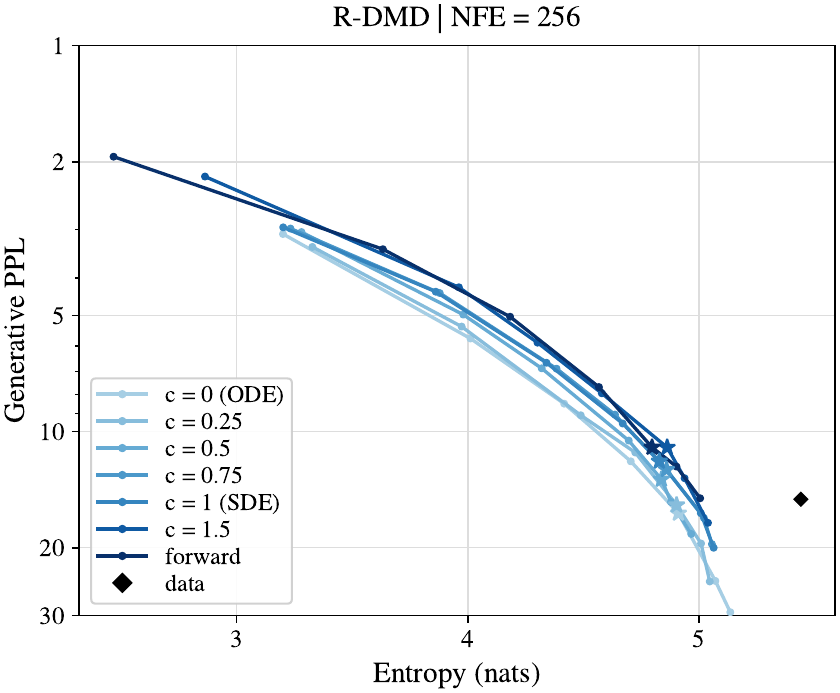}
        \subcaption{\RDMD, $\NFE = 128$ and $256$}
    \end{subfigure}
    \caption{Sampler comparison for $c \in \{0, 0.25, 0.5, 0.75, 1, 1.5\}$ and forward renoising.}
    \label{fig:exp-proposal-samplers}
\end{figure}

\paragraph{Forward renoising performs best among the tested samplers.}
The multi-step construction of \Cref{sec:method-fewstep} defines a family of
sampling transitions through the noise coefficient $\zeta_{t,T}$.
We compare
\begin{equation}
    \zeta_{t,T}
    = c\,\frac{\sigma_t\sigma_{T\mid t}}{\sigma_T},
    \qquad
    c\in\{0,0.25,0.5,0.75,1,1.5\},
\end{equation}
clipping $c$ to respect the constraints on $\zeta_{t,T}$ if necessary,
with forward renoising, $\zeta_{t,T}=\sigma_t$, where a maximum amount of fresh Gaussian noise is added at each sampling
step. Across the swept budgets, forward renoising gives the best observed
Gen PPL--entropy frontier for both \SDMD and \RDMD
(\Cref{fig:exp-proposal-samplers}).

\paragraph{Additional training choices.}
Further ablations in \Cref{app:ablations,app:fewstep} support the remaining
choices of our default configuration. Initializing the student from scratch
leads to collapse in the tested configurations, motivating initialization
from the teacher. Self-conditioning improves \SDMD but provides no benefit for \RDMD in our
experiments, so we use it only for \SDMD.
For \SDMD, a bridge-based training construction
inspired by moment matching \citep{salimans2024multistep} underperforms the
forward-renoising training objective; we discuss its auxiliary score
approximation in \Cref{app:fewstep-bridge}.

\section{Conclusion}
\label{sec:conclusion}

We introduced a distribution-matching framework for distilling continuous
diffusion models on discrete data and derived two methods from a common
reverse-KL objective: \SDMD, based on simplex-valued outputs and pathwise
gradients, and \RDMD, based on categorical sampling and score-function
estimation. On OpenWebText, they improve generation at complementary sampling
budgets, with \SDMD strongest at low budgets and \RDMD at larger budgets among
the evaluated distillation methods. Like other distribution-matching
approaches, our methods require a pretrained teacher and an auxiliary model
during training. Our experiments are also limited to academic model scale;
scaling to substantially larger language models will require studying this
training overhead and extending evaluation beyond the Gen PPL--entropy
protocol used here.

\subsection*{AI use statement}

Generative AI tools assisted with writing and editing the manuscript,
including improving clarity and structure and providing feedback on the
presentation and interpretation of experimental results.
They also assisted with coordinating experiment scripts,
cluster jobs, and data collection.
The authors reviewed and revised all AI-assisted suggestions and take
full responsibility for the scientific claims and final content of the paper.

\bibliography{references}

\begin{thebibliography}{46}
\providecommand{\natexlab}[1]{#1}
\providecommand{\url}[1]{\texttt{#1}}
\expandafter\ifx\csname urlstyle\endcsname\relax
  \providecommand{\doi}[1]{doi: #1}\else
  \providecommand{\doi}{doi: \begingroup \urlstyle{rm}\Url}\fi

\bibitem[Austin et~al.(2021)Austin, Johnson, Ho, Tarlow, and van~den
  Berg]{austin2021structured}
Jacob Austin, Daniel~D. Johnson, Jonathan Ho, Daniel Tarlow, and Rianne van~den
  Berg.
\newblock Structured denoising diffusion models in discrete state-spaces.
\newblock In \emph{Advances in Neural Information Processing Systems
  (NeurIPS)}, 2021.

\bibitem[Boffi et~al.(2025{\natexlab{a}})Boffi, Albergo, and
  Vanden-Eijnden]{boffi2025consistency}
Nicholas~M. Boffi, Michael~S. Albergo, and Eric Vanden-Eijnden.
\newblock How to build a consistency model: Learning flow maps via
  self-distillation.
\newblock In \emph{Advances in Neural Information Processing Systems
  (NeurIPS)}, volume~38, 2025{\natexlab{a}}.

\bibitem[Boffi et~al.(2025{\natexlab{b}})Boffi, Albergo, and
  Vanden-Eijnden]{boffi2025flowmap}
Nicholas~M. Boffi, Michael~S. Albergo, and Eric Vanden-Eijnden.
\newblock Flow map matching with stochastic interpolants: A mathematical
  framework for consistency models.
\newblock \emph{Transactions on Machine Learning Research}, 2025{\natexlab{b}}.

\bibitem[Chen et~al.(2026{\natexlab{a}})Chen, Zhang, and
  Zhou]{chen2026dlmonediffusionlanguagemodels}
Tianqi Chen, Shujian Zhang, and Mingyuan Zhou.
\newblock Dlm-one: Diffusion language models for one-step sequence generation,
  2026{\natexlab{a}}.
\newblock URL \url{https://arxiv.org/abs/2506.00290}.

\bibitem[Chen et~al.(2023)Chen, Zhang, and Hinton]{chen2023analog}
Ting Chen, Ruixiang Zhang, and Geoffrey Hinton.
\newblock Analog bits: Generating discrete data using diffusion models with
  self-conditioning.
\newblock In \emph{International Conference on Learning Representations
  (ICLR)}, 2023.

\bibitem[Chen et~al.(2026{\natexlab{b}})Chen, Liang, Sui, Guo, Cheng, You, and
  Liu]{chen2026langflow}
Yuxin Chen, Chumeng Liang, Hangke Sui, Ruihan Guo, Chaoran Cheng, Jiaxuan You,
  and Ge~Liu.
\newblock {LangFlow}: Continuous diffusion rivals discrete in language
  modeling.
\newblock \emph{arXiv preprint arXiv:2604.11748}, 2026{\natexlab{b}}.

\bibitem[Davis et~al.(2026)Davis, Filippova, Ablin, Turrisi, Shidani, Cuturi,
  and B{\'e}thune]{davis2026scaling}
Oscar Davis, Anastasiia Filippova, Pierre Ablin, Victor Turrisi, Amitis
  Shidani, Marco Cuturi, and Louis B{\'e}thune.
\newblock Scaling categorical flow maps.
\newblock \emph{arXiv preprint arXiv:2605.07820}, 2026.

\bibitem[Deschenaux \& Gulcehre(2025)Deschenaux and
  Gulcehre]{deschenaux2025sdtt}
Justin Deschenaux and Caglar Gulcehre.
\newblock Beyond autoregression: Fast {LLM}s via self-distillation through
  time.
\newblock In \emph{International Conference on Learning Representations
  (ICLR)}, 2025.

\bibitem[Dieleman et~al.(2022)Dieleman, Sartran, Roshannai, Savinov, Ganin,
  Richemond, Doucet, Strudel, Dyer, Durkan, Hawthorne, Leblond, Grathwohl, and
  Adler]{dieleman2022continuous}
Sander Dieleman, Laurent Sartran, Arman Roshannai, Nikolay Savinov, Yaroslav
  Ganin, Pierre~H. Richemond, Arnaud Doucet, Robin Strudel, Chris Dyer, Conor
  Durkan, Curtis Hawthorne, R{\'e}mi Leblond, Will Grathwohl, and Jonas Adler.
\newblock Continuous diffusion for categorical data.
\newblock \emph{arXiv preprint arXiv:2211.15089}, 2022.

\bibitem[Efron(2011)]{efron2011tweedie}
Bradley Efron.
\newblock {Tweedie's} formula and selection bias.
\newblock \emph{Journal of the American Statistical Association}, 106\penalty0
  (496):\penalty0 1602--1614, 2011.

\bibitem[Goodfellow et~al.(2014)Goodfellow, Pouget-Abadie, Mirza, Xu,
  Warde-Farley, Ozair, Courville, and Bengio]{goodfellow2014gan}
Ian~J. Goodfellow, Jean Pouget-Abadie, Mehdi Mirza, Bing Xu, David
  Warde-Farley, Sherjil Ozair, Aaron Courville, and Yoshua Bengio.
\newblock Generative adversarial nets.
\newblock In \emph{Advances in Neural Information Processing Systems
  (NeurIPS)}, 2014.

\bibitem[Gourevitch et~al.(2026)Gourevitch, Janati, Shariatian, Simsekli,
  Moulines, Xing, and Durmus]{gourevitch2026uniform}
Samson Gourevitch, Yazid Janati, Dario Shariatian, Umut Simsekli, Eric
  Moulines, Eric~P. Xing, and Alain Durmus.
\newblock Uniform diffusion models revisited: Leave-one-out denoiser and
  absorbing state reformulation.
\newblock \emph{arXiv preprint arXiv:2605.22765}, 2026.

\bibitem[Gulrajani \& Hashimoto(2023)Gulrajani and
  Hashimoto]{gulrajani2023likelihood}
Ishaan Gulrajani and Tatsunori~B. Hashimoto.
\newblock Likelihood-based diffusion language models.
\newblock In \emph{Advances in Neural Information Processing Systems
  (NeurIPS)}, 2023.

\bibitem[Hayakawa et~al.(2025)Hayakawa, Takida, Imaizumi, Wakaki, and
  Mitsufuji]{hayakawa2025di4c}
Satoshi Hayakawa, Yuhta Takida, Masaaki Imaizumi, Hiromi Wakaki, and Yuki
  Mitsufuji.
\newblock Distillation of discrete diffusion through dimensional correlations.
\newblock In \emph{International Conference on Machine Learning (ICML)}, 2025.

\bibitem[Hoogeboom et~al.(2026)Hoogeboom, Ruhe, Heek, Mensink, and
  Salimans]{hoogeboom2026dmmd}
Emiel Hoogeboom, David Ruhe, Jonathan Heek, Thomas Mensink, and Tim Salimans.
\newblock Beyond single tokens: Distilling discrete diffusion models via
  discrete {MMD}.
\newblock \emph{arXiv preprint arXiv:2603.20155}, 2026.

\bibitem[Hu et~al.(2026)Hu, Qiu, Lu, Zhao, Li, Kim, Andreas, and He]{hu2026elf}
Keya Hu, Linlu Qiu, Yiyang Lu, Hanhong Zhao, Tianhong Li, Yoon Kim, Jacob
  Andreas, and Kaiming He.
\newblock {ELF}: Embedded language flows.
\newblock \emph{arXiv preprint arXiv:2605.10938}, 2026.

\bibitem[Huang et~al.(2024)Huang, Geng, Luo, and Qi]{huang2024flow}
Zemin Huang, Zhengyang Geng, Weijian Luo, and Guo-jun Qi.
\newblock Flow generator matching.
\newblock \emph{arXiv preprint arXiv:2410.19310}, 2024.

\bibitem[Kingma \& Welling(2014)Kingma and Welling]{kingma2014autoencoding}
Diederik~P. Kingma and Max Welling.
\newblock Auto-encoding variational {Bayes}.
\newblock In \emph{International Conference on Learning Representations
  (ICLR)}, 2014.

\bibitem[Kingma et~al.(2021)Kingma, Salimans, Poole, and Ho]{kingma2021vdm}
Diederik~P. Kingma, Tim Salimans, Ben Poole, and Jonathan Ho.
\newblock Variational diffusion models.
\newblock In \emph{Advances in Neural Information Processing Systems
  (NeurIPS)}, 2021.

\bibitem[Kool et~al.(2019)Kool, van Hoof, and Welling]{kool2019buy}
Wouter Kool, Herke van Hoof, and Max Welling.
\newblock Buy 4 {REINFORCE} samples, get a baseline for free!
\newblock In \emph{ICLR 2019 Deep Reinforcement Learning meets Structured
  Prediction Workshop}, 2019.

\bibitem[Lee et~al.(2026)Lee, Yoo, Agarwal, Shah, Huang, Raghunathan, Hong,
  Boffi, and Kim]{lee2026fmlm}
Chanhyuk Lee, Jaehoon Yoo, Manan Agarwal, Sheel Shah, Jerry Huang, Aditi
  Raghunathan, Seunghoon Hong, Nicholas~M. Boffi, and Jinwoo Kim.
\newblock Flow map language models: One-step language modeling via continuous
  denoising.
\newblock \emph{arXiv preprint arXiv:2602.16813}, 2026.

\bibitem[Li et~al.(2026)Li, Gushchin, Abulkhanov, Moulines, Oseledets, Panov,
  and Korotin]{li2026idlm}
David Li, Nikita Gushchin, Dmitry Abulkhanov, Eric Moulines, Ivan Oseledets,
  Maxim Panov, and Alexander Korotin.
\newblock {IDLM}: Inverse-distilled diffusion language models.
\newblock \emph{arXiv preprint arXiv:2602.19066}, 2026.

\bibitem[Li et~al.(2022)Li, Thickstun, Gulrajani, Liang, and
  Hashimoto]{li2022diffusion}
Xiang~Lisa Li, John Thickstun, Ishaan Gulrajani, Percy Liang, and Tatsunori~B.
  Hashimoto.
\newblock Diffusion-{LM} improves controllable text generation.
\newblock In \emph{Advances in Neural Information Processing Systems
  (NeurIPS)}, 2022.

\bibitem[Liu et~al.(2023)Liu, Gong, and Liu]{liu2023rectified}
Xingchao Liu, Chengyue Gong, and Qiang Liu.
\newblock Flow straight and fast: Learning to generate and transfer data with
  {Rectified Flow}.
\newblock In \emph{International Conference on Learning Representations
  (ICLR)}, 2023.

\bibitem[Lou et~al.(2024)Lou, Meng, and Ermon]{lou2024discrete}
Aaron Lou, Chenlin Meng, and Stefano Ermon.
\newblock Discrete diffusion modeling by estimating the ratios of the data
  distribution.
\newblock In \emph{International Conference on Machine Learning (ICML)}, 2024.

\bibitem[Luo et~al.(2023)Luo, Hu, Zhang, Sun, Li, and
  Zhang]{luo2023diffinstruct}
Weijian Luo, Tianyang Hu, Shifeng Zhang, Jiacheng Sun, Zhenguo Li, and Zhihua
  Zhang.
\newblock {Diff-Instruct}: A universal approach for transferring knowledge from
  pre-trained diffusion models.
\newblock In \emph{Advances in Neural Information Processing Systems
  (NeurIPS)}, 2023.

\bibitem[Potaptchik et~al.(2026)Potaptchik, Yim, Saravanan, Holderrieth,
  Vanden-Eijnden, and Albergo]{potaptchik2026dfm}
Peter Potaptchik, Jason Yim, Adhi Saravanan, Peter Holderrieth, Eric
  Vanden-Eijnden, and Michael~S. Albergo.
\newblock Discrete flow maps.
\newblock \emph{arXiv preprint arXiv:2604.09784}, 2026.

\bibitem[Pynadath et~al.(2026)Pynadath, Shi, and
  Zhang]{pynadath2026generativefrontiersevaluationmatters}
Patrick Pynadath, Jiaxin Shi, and Ruqi Zhang.
\newblock Generative frontiers: Why evaluation matters for diffusion language
  models.
\newblock \emph{arXiv preprint arXiv:2604.02718}, 2026.

\bibitem[Radford et~al.(2019)Radford, Wu, Child, Luan, Amodei, and
  Sutskever]{radford2019language}
Alec Radford, Jeff Wu, Rewon Child, David Luan, Dario Amodei, and Ilya
  Sutskever.
\newblock Language models are unsupervised multitask learners.
\newblock Technical report, OpenAI, 2019.

\bibitem[Roos et~al.(2026)Roos, Davis, Eijkelboom, Bronstein, Welling, Ceylan,
  Ambrogioni, and van~de Meent]{roos2026categorical}
Daan Roos, Oscar Davis, Floor Eijkelboom, Michael Bronstein, Max Welling,
  {\.I}smail~{\.I}lkan Ceylan, Luca Ambrogioni, and Jan-Willem van~de Meent.
\newblock Categorical flow maps.
\newblock \emph{arXiv preprint arXiv:2602.12233}, 2026.

\bibitem[Sahoo et~al.(2024)Sahoo, Arriola, Schiff, Gokaslan, Marroquin, Chiu,
  Rush, and Kuleshov]{sahoo2024simple}
Subham~Sekhar Sahoo, Marianne Arriola, Yair Schiff, Aaron Gokaslan, Edgar
  Marroquin, Justin~T. Chiu, Alexander Rush, and Volodymyr Kuleshov.
\newblock Simple and effective masked diffusion language models.
\newblock In \emph{Advances in Neural Information Processing Systems
  (NeurIPS)}, 2024.

\bibitem[Sahoo et~al.(2025)Sahoo, Deschenaux, Gokaslan, Wang, Chiu, and
  Kuleshov]{sahoo2025duality}
Subham~Sekhar Sahoo, Justin Deschenaux, Aaron Gokaslan, Guanghan Wang, Justin
  Chiu, and Volodymyr Kuleshov.
\newblock The diffusion duality.
\newblock In \emph{International Conference on Machine Learning (ICML)}, 2025.

\bibitem[Salimans et~al.(2024)Salimans, Mensink, Heek, and
  Hoogeboom]{salimans2024multistep}
Tim Salimans, Thomas Mensink, Jonathan Heek, and Emiel Hoogeboom.
\newblock Multistep distillation of diffusion models via moment matching.
\newblock In \emph{Advances in Neural Information Processing Systems
  (NeurIPS)}, 2024.

\bibitem[Schiff et~al.(2025)Schiff, Sahoo, Phung, Wang, Rush, Kuleshov,
  Dalla-Torre, Boshar, de~Almeida, and Pierrot]{schiff2025simple}
Yair Schiff, Subham~Sekhar Sahoo, Hao Phung, Guanghan Wang, Alexander Rush,
  Volodymyr Kuleshov, Hugo Dalla-Torre, Sam Boshar, Bernardo~P. de~Almeida, and
  Thomas Pierrot.
\newblock Simple guidance mechanisms for discrete diffusion models.
\newblock In \emph{International Conference on Learning Representations
  (ICLR)}, 2025.

\bibitem[Shao et~al.(2024)Shao, Wang, Zhu, Xu, Song, Bi, Zhang, Zhang, Li, Wu,
  and Guo]{shao2024deepseekmath}
Zhihong Shao, Peiyi Wang, Qihao Zhu, Runxin Xu, Junxiao Song, Xiao Bi, Haowei
  Zhang, Mingchuan Zhang, Y.~K. Li, Y.~Wu, and Daya Guo.
\newblock {DeepSeekMath}: Pushing the limits of mathematical reasoning in open
  language models.
\newblock \emph{arXiv preprint arXiv:2402.03300}, 2024.

\bibitem[Song et~al.(2021)Song, Meng, and Ermon]{song2021ddim}
Jiaming Song, Chenlin Meng, and Stefano Ermon.
\newblock Denoising diffusion implicit models.
\newblock In \emph{International Conference on Learning Representations
  (ICLR)}, 2021.

\bibitem[Song \& Dhariwal(2024)Song and Dhariwal]{song2024improved}
Yang Song and Prafulla Dhariwal.
\newblock Improved techniques for training consistency models.
\newblock In \emph{International Conference on Learning Representations
  (ICLR)}, 2024.

\bibitem[Song et~al.(2023)Song, Dhariwal, Chen, and
  Sutskever]{song2023consistency}
Yang Song, Prafulla Dhariwal, Mark Chen, and Ilya Sutskever.
\newblock Consistency models.
\newblock In \emph{International Conference on Machine Learning (ICML)}, 2023.

\bibitem[Strudel et~al.(2022)Strudel, Tallec, Altch{\'e}, Du, Ganin, Mensch,
  Grathwohl, Savinov, Dieleman, Sifre, and Leblond]{strudel2022self}
Robin Strudel, Corentin Tallec, Florent Altch{\'e}, Yilun Du, Yaroslav Ganin,
  Arthur Mensch, Will Grathwohl, Nikolay Savinov, Sander Dieleman, Laurent
  Sifre, and R{\'e}mi Leblond.
\newblock Self-conditioned embedding diffusion for text generation.
\newblock \emph{arXiv preprint arXiv:2211.04236}, 2022.

\bibitem[Williams(1992)]{williams1992simple}
Ronald~J. Williams.
\newblock Simple statistical gradient-following algorithms for connectionist
  reinforcement learning.
\newblock \emph{Machine Learning}, 8:\penalty0 229--256, 1992.

\bibitem[Yang et~al.(2026)Yang, Guo, Zhang, Sahoo, Chen, Vahdat, Mardani, and
  Thickstun]{yang2026replaid}
Zhihan Yang, Wei Guo, Shuibai Zhang, Subham~Sekhar Sahoo, Yongxin Chen, Arash
  Vahdat, Morteza Mardani, and John Thickstun.
\newblock Continuous diffusion scales competitively with discrete diffusion for
  language.
\newblock \emph{arXiv preprint arXiv:2605.18530}, 2026.

\bibitem[Yin et~al.(2024)Yin, Gharbi, Zhang, Shechtman, Durand, Freeman, and
  Park]{yin2024onestep}
Tianwei Yin, Micha{\"e}l Gharbi, Richard Zhang, Eli Shechtman, Fr{\'e}do
  Durand, William~T. Freeman, and Taesung Park.
\newblock One-step diffusion with distribution matching distillation.
\newblock In \emph{IEEE/CVF Conference on Computer Vision and Pattern
  Recognition (CVPR)}, 2024.

\bibitem[Yoo et~al.(2025)Yoo, Kim, and Hong]{yoo2025redi}
Jaehoon Yoo, Wonjung Kim, and Seunghoon Hong.
\newblock {ReDi}: Rectified discrete flow.
\newblock In \emph{Advances in Neural Information Processing Systems
  (NeurIPS)}, 2025.

\bibitem[Zhang et~al.(2022)Zhang, Roller, Goyal, Artetxe, Chen, Chen, Dewan,
  Diab, Li, Lin, Mihaylov, Ott, Shleifer, Shuster, Simig, Koura, Sridhar, Wang,
  and Zettlemoyer]{zhang2022optopenpretrainedtransformer}
Susan Zhang, Stephen Roller, Naman Goyal, Mikel Artetxe, Moya Chen, Shuohui
  Chen, Christopher Dewan, Mona Diab, Xian Li, Xi~Victoria Lin, Todor Mihaylov,
  Myle Ott, Sam Shleifer, Kurt Shuster, Daniel Simig, Punit~Singh Koura, Anjali
  Sridhar, Tianlu Wang, and Luke Zettlemoyer.
\newblock Opt: Open pre-trained transformer language models, 2022.
\newblock URL \url{https://arxiv.org/abs/2205.01068}.

\bibitem[Zhou et~al.(2024)Zhou, Zheng, Wang, Yin, and Huang]{zhou2024score}
Mingyuan Zhou, Huangjie Zheng, Zhendong Wang, Mingzhang Yin, and Hai Huang.
\newblock Score identity distillation: Exponentially fast distillation of
  pretrained diffusion models for one-step generation.
\newblock In \emph{International Conference on Machine Learning (ICML)}, 2024.

\bibitem[Zhu et~al.(2025)Zhu, Wang, Lathuili{\`e}re, and
  Kalogeiton]{zhu2025dimo}
Yuanzhi Zhu, Xi~Wang, St{\'e}phane Lathuili{\`e}re, and Vicky Kalogeiton.
\newblock {Di[M]O}: Distilling masked diffusion models into a one-step
  generator.
\newblock \emph{arXiv preprint arXiv:2503.15457}, 2025.

\end{thebibliography}

\newpage
\appendix

\addtocontents{toc}{\protect\setcounter{tocdepth}{2}}
\tableofcontents
\clearpage

\section{Notation}
\label{app:notation}

\paragraph{Data, representations, and dimensions.}
The vocabulary is $\msv$ with $K$ tokens, and $\msx=\msv^L$ is the space of sequences of length $L$.
We identify token $k$ with the canonical basis vector $\mathbf{e}_k\in\R^K$ and discrete sequences with matrices of one-hot rows.
The probability simplex is $\Delta_K=\{u\in\R^K:u_k\geq0,\ \sum_k u_k=1\}$; relaxed sequences and collections of token probabilities belong to $\Delta_K^L$.
We write $\mathbf{x}^\ell$ for row $\ell$ and $x_k^\ell$ for its $k$-th coordinate.
A collection of token marginals does not by itself specify a joint sequence distribution.
The clean data law is $p_{\mathrm{data}}$.
The embedding matrix $\Emat\in\R^{K\times d}$ has rows $E_k$, where $d$ is the embedding dimension, and maps $\mathbf{x}$ to $\mathbf{x}\Emat\in\R^{L\times d}$.
Uppercase symbols such as $\mathbf{X}$ and $Z_t$ denote random variables; lowercase symbols denote their realizations.

\paragraph{Gaussian diffusion and noise levels.}
The data forward process has marginals $p_t$ and conditional laws
\[
    Z_t=\alpha_t\mathbf{X}\Emat+\sigma_t\varepsilon,
    \qquad
    q_{t\mid\mathbf{x}}(\cdot\mid\mathbf{x})
    =\gauss(\alpha_t\mathbf{x}\Emat,\sigma_t^2\Id),
    \qquad \varepsilon\sim\gauss(0,\Id).
\]
Here $\Id$ is the identity on vectorized latent matrices, $\sigma_t^2=1-\alpha_t^2$, and the idealized endpoints are $(\alpha_0,\sigma_0)=(1,0)$ and $(\alpha_1,\sigma_1)=(0,1)$, so $p_1=\gauss(0,\Id)$.
For $s<t$, $q_{t\mid s}$ denotes the forward transition, with $\alpha_{t\mid s}=\alpha_t/\alpha_s$ and $\sigma_{t\mid s}^2=\sigma_t^2-\alpha_{t\mid s}^2\sigma_s^2$.
The Gaussian bridge $q_{s\mid t,\mathbf{x}}$ conditions on both the later latent and the clean representation; $p^\theta_{s\mid t}$ denotes a learned teacher reverse transition.

\paragraph{Posteriors, denoisers, and scores.}
The joint data posterior is $\post(\cdot\mid z_t)$, with token marginal $\post[\ell](\cdot\mid z_t)$ and mean $m_t(z_t)=\E[\mathbf{X}\mid Z_t=z_t]\in\Delta_K^L$.
The pretrained teacher has parameters $\theta$, denoiser $\mathbf{x}^\theta(z_t,t)$, and generative law $p^\theta$; $p_t$ always denotes the data forward marginal.
Student parameters are $\eta$, auxiliary-denoiser parameters are $\phi$, and discriminator parameters are $\psi$.
The auxiliary denoiser $\mathbf{x}^\phi(z,t,T)$ targets $m_t^{\eta,T}(z)=\E[\mathbf{X}^{\eta,T}\mid Z_t^{\eta,T}=z]$.
The data and student scores are $s(z,t)=\nabla_z\log p_t(z)$ and $s^{\eta,T}(z,t)=\nabla_z\log p_t^{\eta,T}(z)$; $s^\theta$ denotes the teacher estimate of the data score.
The scalar time $s$ in a transition is distinct from the score function $s(z,t)$.

\paragraph{Student outputs and training marginals.}
Given $Z_T$, the student predicts $\mathbf{x}^\eta(Z_T,T)\in\Delta_K^L$, which parameterizes the clean-output kernel $\genkernel_T^\eta(\cdot\mid Z_T)$.
For \SDMD this kernel is the point mass $\delta_{\mathbf{x}^\eta(Z_T,T)}$; for \RDMD it is the product of categorical laws with these token probabilities.
We write $\mathbf{X}^{\eta,T}\sim\genkernel_T^\eta(\cdot\mid Z_T)$ and $p_{\mathbf{x}}^{\eta,T}$ for its marginal when $Z_T\sim p_T$.
Forward noising gives $Z_t^{\eta,T}$ with law $p_t^{\eta,T}$, where
$\pi_t(\mathbf{x},\varepsilon)=\alpha_t\mathbf{x}\Emat+\sigma_t\varepsilon$ denotes the reparameterized noising map.
The one-step notation $\mathbf{X}^\eta$, $p_{\mathbf{x}}^\eta$, and $p_t^\eta$ omits the fixed starting level $T=1$.

\paragraph{Time sampling and matching objectives.}
In the multi-step objective, $T$ is the starting noise level and $t<T$ is the level at which noised distributions are compared.
We sample $(t,T)\sim\nu$, with
$\nu(\mathrm{d}t,\mathrm{d}T)=\nu_T(\mathrm{d}t)\nu_1(\mathrm{d}T)$ and $\nu_s=\Unif(0,s)$.
The local discrepancy is $\mathsf{D}$, and $\mathsf{L}_{\mathrm{distill}}$ denotes the generic distribution-matching objective.
For reverse KL, the one-step and multi-step losses are
\[
    \mathsf{L}_{\mathrm{DMD}}(\eta)
    =\int_0^1\KL{p_t^\eta}{p_t}\,\mathrm{d}t,
    \qquad
    \mathsf{L}_{\mathrm{DMD}}^{\mathrm{ms}}(\eta)
    =\E_{(t,T)\sim\nu}\!\left[\KL{p_t^{\eta,T}}{p_t}\right].
\]
The auxiliary-denoiser and discriminator objectives are $\mathsf{L}_{\mathrm{aux}}$ and $\mathsf{L}_{\mathrm{disc}}$, respectively; $n_{\mathrm{aux}}$ is the number of auxiliary updates per student update.

\paragraph{Discrete-sampling cost and variance reduction.}
The discriminator $D_\psi(z,t,T)\in(0,1)$ estimates the probability that a noised sample comes from data rather than the student.
Its negative logit, $\log((1-D_\psi)/D_\psi)$, estimates the log-ratio $\log(p_t^{\eta,T}/p_t)$.
The expected log-ratio cost is $C_T$, and $\widehat C_\psi$ denotes its pointwise estimate using a sampled time and Gaussian noise.
For a group of $G$ samples sharing $(Z_T,T)$, $g$ indexes a sample, $\widehat C_g$ is its estimated cost, and $\widehat b_g=(G-1)^{-1}\sum_{g'\ne g}\widehat C_{g'}$ is its leave-one-out baseline.
The coefficient $\beta\geq0$ weights the tokenwise KL anchor from the student to the frozen teacher.

\paragraph{Multi-step sampling and evaluation.}
The inference grid is $1=T_n>T_{n-1}>\cdots>T_0=0$.
Hatted variables $\widehat Z_{T_i}$ and $\widehat{\mathbf{X}}^{\eta,T_i}$ denote states and clean proposals along the composed sampler; their laws need not equal the corresponding training marginals.
The student transition is $p_{t\mid T}^\eta$, formed by composing $\genkernel_T^\eta$ with the DDIM-inspired kernel $q^\zeta_{t\mid\mathbf{x},T}$.
The noise standard deviation $\zeta_{t,T}$ controls fresh noise: $0$ gives a deterministic update conditional on the clean proposal, $\sigma_t\sigma_{T\mid t}/\sigma_T$ gives the Gaussian bridge, and $\sigma_t$ gives forward renoising.
The sampler experiments use the multiplier $c$ in $\zeta_{t,T}=c\sigma_t\sigma_{T\mid t}/\sigma_T$.
Sampling temperature rescales output logits and is separate from the starting noise level $T$ (the sample appendix also uses $T$ locally for temperature).
$\NFE$ denotes the number of network evaluations per generated sample.
Gen PPL is generative perplexity under GPT-2-large; entropy is the mean per-sequence Shannon entropy of empirical token frequencies, measured in nats.

\paragraph{Operators and appendix conventions.}
We use $\KL{P}{Q}$ for Kullback--Leibler divergence, $\mathrm{CE}(u,v)=-\sum_k u_k\log v_k$ for cross-entropy, $\sg{\cdot}$ for stop-gradient, $\delta_x$ for a point mass, and $\eqd$ for equality in distribution.
In \Cref{app:emergent}, $V$ is the set of one-hot sequences, $\mathcal{E}(\mathbf{x})=\mathbf{x}\Emat$, $\rho$ is a clean law, and $Q_t\rho$ is its embedded, noised law; $f_\#\rho$ denotes the law of $f(X)$ for $X\sim\rho$.
In the comparison of continuous distillation objectives in \Cref{app:distillation}, clean samples are written directly in Euclidean space, so $m_t$ there is a mean in that space; $s_t$ and $v_t$ denote the score and velocity fields.
Other locally defined symbols retain the meanings specified in their respective derivations.
\section{Background: objectives and parameterizations}
\label{app:bg-derivations}

We expand the continuous and discrete formulations recalled in \Cref{sec:bg-continuous}.

\subsection{Continuous diffusion on discrete data: objective and design choices}
\label{app:continuous}

We review four design choices for continuous diffusion on token embeddings:
training objectives, embedding normalization,
self-conditioning and noise schedules, highlighting how existing models implement them.

\paragraph{The objective.}
We distinguish likelihood-based training, used by Diffusion-LM
\citep{li2022diffusion}, Plaid \citep{gulrajani2023likelihood}, and RePlaid
\citep{yang2026replaid}, from direct cross-entropy training.
In the notation of \Cref{sec:bg-continuous}, the continuous-time negative
evidence lower bound consists of prior, reconstruction, and diffusion terms:
\begin{equation}
\begin{aligned}
    \mathsf{L}_{\mathrm{c}}(\mathbf{x};\theta)
    ={}&\mathsf{L}_{\mathrm{prior}}(\mathbf{x})
    +\E_{Z_0\sim q_{0\mid\mathbf{x}}(\cdot\mid\mathbf{x})}
      \left[-\sum_{\ell=1}^L
      \log\langle\mathbf{x}^{\theta,\ell}(Z_0,0),\mathbf{x}^\ell\rangle\right]\\
    &+\frac12\int_0^1[-\SNR'(t)]\,
      \E_{Z_t\sim q_{t\mid\mathbf{x}}(\cdot\mid\mathbf{x})}
      \left[\left\|\big(\mathbf{x}^\theta(Z_t,t)-\mathbf{x}\big)
      \Emat\right\|_{\mathrm F}^2\right]\,\mathrm{d}t,
\end{aligned}
\label{eq:app-full-nelbo}
\end{equation}
where $\SNR(t)=\alpha_t^2/\sigma_t^2$. For a standard Gaussian generative
prior, the prior term is explicit:
\begin{equation}
\begin{aligned}
    \mathsf{L}_{\mathrm{prior}}(\mathbf{x})
    &=\KL{q_{1\mid\mathbf{x}}(\cdot\mid\mathbf{x})}
           {\gauss(0,\Id)} \, .
\end{aligned}
\label{eq:app-prior-loss}
\end{equation}
Since $\alpha_1=0$ and $\sigma_1=1$ in \Cref{sec:bg-continuous},
$q_{1\mid\mathbf{x}}(\cdot\mid\mathbf{x})=\gauss(0,\Id)$ and this term
vanishes. It is nonzero only for schedules with a finite terminal
signal-to-noise ratio. The
reconstruction term trains the categorical decoder at time zero, while
the diffusion term penalizes denoising errors in embedding space.

Alternatively, direct cross-entropy training supervises the clean token
at every sampled noise level:
\begin{equation}
    \mathsf{L}_{\mathrm{CE}}(\mathbf{x};\theta)
    =\E_{\substack{t\sim\Unif(0,1)\\
        Z_t\sim q_{t\mid\mathbf{x}}(\cdot\mid\mathbf{x})}}
    \left[-\sum_{\ell=1}^L
        \log\langle\mathbf{x}^{\theta,\ell}(Z_t,t),
        \mathbf{x}^\ell\rangle\right].
    \label{eq:app-ce-loss}
\end{equation}
This approach is used by CDCD \citep{dieleman2022continuous} and LangFlow
\citep{chen2026langflow}, with model-specific representations and
parameterizations. Analog Bits uses squared-error regression by default,
but also investigates sigmoid and softmax cross-entropy variants
\citep[Appendix B]{chen2023analog}. SED instead combines a mean-squared
error diffusion loss in a fixed embedding space with a separate
reconstruction cross-entropy loss for its trainable readout
\citep[Section 3, Eqs.~(4)--(5)]{strudel2022self}.
Both objectives above are averaged over
$\mathbf{x}\sim p_{\mathrm{data}}$ during training. Unlike the diffusion
term of the NELBO, this cross-entropy directly supervises token
probabilities rather than their embedded mean; the unweighted objective
above is not the NELBO.

\paragraph{Embedding normalization.}
These objectives also put different pressures on learned embeddings.
The diffusion term alone vanishes if all token embeddings coincide,
whereas cross-entropy favors distinguishable tokens and can encourage
increasing their separation relative to the noise. To control this scale,
embeddings are constrained to a sphere, with each row of $\Emat$ having
norm $1$ or $\sqrt{d}$, where $d$ is the embedding dimension
\citep{dieleman2022continuous,chen2026langflow,yang2026replaid}.
This prevents the embedding norms from vanishing or diverging, although
normalization alone does not prevent different tokens from sharing the
same embedding.

\paragraph{Self-conditioning.}
Self-conditioning lets the denoiser refine an earlier prediction of the
clean sample alongside the current noisy latent. Introduced in Analog Bits
\citep{chen2023analog} and adopted for token embeddings in SED
\citep{strudel2022self}, it conditions on predicted analog-bit vectors
and continuous embeddings, respectively. In our token-probability
parameterization, we express this mechanism by augmenting the network
input as $\mathbf{x}^\theta(z_t,t,\hat{\mathbf{x}})$, where
$\hat{\mathbf{x}}$ is a previous prediction in $\Delta_K^L$ or a zero
matrix indicating that no prediction is available.
At sampling time, the prediction from the preceding denoising step is
reused without an additional network evaluation. Along the descending
grid $t_n>\dots>t_0$, this gives
\begin{equation}
    \hat{\mathbf{x}}^{(i)}
    =\mathbf{x}^\theta\big(z_{t_i},t_i,\hat{\mathbf{x}}^{(i+1)}\big),
    \qquad \hat{\mathbf{x}}^{(n+1)}=0.
\end{equation}
For $i\geq1$, this prediction parameterizes the transition from $t_i$
to $t_{i-1}$ and is passed to the next network evaluation.

During training, a noisy state is sampled directly, so no preceding
prediction is available. On a randomly selected subset of training
examples, a preliminary forward pass supplies the conditioning input:
\begin{equation}
    \hat{\mathbf{x}}
    =\sg{\mathbf{x}^\theta(z_t,t,0)}.
\end{equation}
The training loss is then evaluated using
$\mathbf{x}^\theta(z_t,t,\hat{\mathbf{x}})$, with gradients flowing only
through this second pass. The remaining examples use
$\hat{\mathbf{x}}=0$, so the network also learns to denoise without a
previous prediction, as required at the first sampling step. The fraction
of self-conditioned training examples is a model-specific choice.

\paragraph{The noise schedule.}
The noise schedule determines how quickly the forward process removes
information from the clean embeddings. We parameterize it by the increasing
negative log signal-to-noise ratio $\gamma(t)=-\log\SNR(t)$, so that
\begin{equation}
    \alpha_t=\sqrt{\operatorname{sigmoid}(-\gamma(t))},
    \qquad
    \sigma_t=\sqrt{\operatorname{sigmoid}(\gamma(t))}.
    \label{eq:app-noise-schedule}
\end{equation}
Thus increasing $\gamma(t)$ decreases the signal and increases the noise
in $Z_t=\alpha_t\mathbf{X}\Emat+\sigma_t\varepsilon$.
The endpoint $\alpha_1=0$ of \Cref{sec:bg-continuous}
correspond to $\gamma(1)=+\infty$; practical
schedules clip $\gamma$ to finite values $\gamma_{\text{min}}, \gamma_{\text{max}}$.
In practice, the networks receive the noise level as $\gamma(t)$ rather
than $t$: the notation $\mathbf{x}^\theta(z_t,t)$ of the main text stands
for $\mathbf{x}^\theta(z_t,\gamma(t))$, and likewise for the student and
auxiliary networks.

LangFlow \citep{chen2026langflow} learns a schedule that aims to distribute
information loss evenly over time. It fits a scaled Gumbel cumulative
distribution function to the observed token cross-entropy as a function
of $\gamma$, learning its location $\mu$, scale $b>0$, and amplitude by
squared-error regression. The corresponding quantile schedule is
\begin{equation}
    \gamma(t)=\mu-b\log(-\log t),\qquad 0<t<1.
    \label{eq:gumbel-schedule}
\end{equation}
This allocates more noise levels to regions where the fitted
cross-entropy changes rapidly; in practice, the noise range is clipped.

RePlaid \citep{yang2026replaid} instead learns two endpoints and a monotone
interpolation between them:
\begin{equation}
    \gamma(t)=\gamma_0+(\gamma_1-\gamma_0)\tilde\gamma(t),
    \qquad \tilde\gamma(0)=0,\quad\tilde\gamma(1)=1.
    \label{eq:monotone-schedule}
\end{equation}
The endpoints are optimized using the diffusion loss, while the monotone
network $\tilde\gamma$ is trained to reduce the Monte Carlo variance of
its estimator. For fixed endpoints and a denoiser taking $\gamma(t)$ as
input, the continuous-time bound is invariant to the interior schedule,
so its shape can be adjusted to improve estimation efficiency
\citep{kingma2021vdm}.

In our experiments, we reuse the teacher's pretrained noise schedule and
keep it fixed throughout distillation and sampling.

\subsection{Discrete diffusion: the bound, the plug-in and the concrete score}
\label{app:discrete}

Discrete diffusion progressively corrupts a token sequence and learns to
reverse this process to generate data. We consider two choices of the
replacement distribution $\pi$: a point mass on an additional absorbing
$\mathtt{[MASK]}$ token for masked diffusion, or the uniform distribution
over the vocabulary for uniform diffusion. The process starts directly
from the clean sequence, $Z_0=\mathbf{X}$, so $p_0=p_{\mathrm{data}}$.
We describe the forward process, the reverse model and its training bound,
then explain how denoiser predictions parameterize reverse transitions
and how this relates to concrete scores in continuous time.

\paragraph{Forward process.}
The forward process corrupts each position independently by replacing its
token with a draw from $\pi$. Let $(\alpha_t)_{t\in[0,1]}$ be a decreasing
schedule with $\alpha_0=1$ and $\alpha_1=0$. Between times $s<t$, a token
is retained with probability $\alpha_t/\alpha_s$ and otherwise resampled
from $\pi$, giving the transition
\begin{equation}
    q^\ell_{t \mid s}(z_t^\ell \mid z_s^\ell) \, = \, \mathrm{Cat}\Big( z_t^\ell ; \, \tfrac{\alpha_t}{\alpha_s} \, z_s^\ell \, + \, \big( 1 - \tfrac{\alpha_t}{\alpha_s} \big) \pi \Big) \, , \qquad 0 \leq s < t \leq 1 \, .
    \label{eq:discrete-transition}
\end{equation}
In particular, setting $s=0$ gives the conditional law given the clean token:
\begin{equation}
    q^\ell_{t \mid \mathbf{x}}(z_t^\ell \mid \mathbf{x}^\ell) \, = \, \mathrm{Cat}\big( z_t^\ell ; \, \alpha_t \mathbf{x}^\ell + (1 - \alpha_t) \pi \big) \, ,
    \label{eq:discrete-marginal}
\end{equation}
These conditional laws and the forward transitions determine the bridges
$q_{s\mid t,\mathbf{x}}$ by Bayes' rule. At $t=1$, every position follows
$\mathrm{Cat}(\pi)$ independently of the clean sequence, so the terminal
distribution is the product prior $p_1=\mathrm{Cat}(\pi)^{\otimes L}$.

\paragraph{Generative model and KL bound.}
Generation starts from this prior and applies learned reverse transitions
along a grid $0=t_0<\dots<t_n=1$, with $s_i=t_{i-1}$.
Each transition predicts the positions independently conditional on the
full current sequence, yielding
\begin{equation}
    p^\theta(\mathbf{x}) \, = \, \sum_{z_{t_1 : t_n}} p_1(z_1) \prod_{i = 1}^n p^\theta_{s_i \mid t_i}(z_{s_i} \mid z_{t_i}) \, , \qquad p^\theta_{s \mid t}(z_s \mid z_t) = \prod_{\ell = 1}^L p^{\theta, \ell}_{s \mid t}(z_s^\ell \mid z_t) \, ,
    \label{eq:discrete-generative}
\end{equation}
where $z_{s_1}=\mathbf{x}$ is the generated sequence at time zero.
To relate the generated distribution to the data distribution, let
$p_{s\mid t}$ denote the exact reverse transition induced by the forward
process with $\mathbf{X}\sim p_{\mathrm{data}}$.
The exact and learned reverse chains share the same terminal prior.
Applying the chain rule for KL divergence to their trajectory laws,
then marginalizing to time zero, bounds the discrepancy between data and
generated samples by the accumulated transition errors:
\begin{equation}
    \KL{p_{\mathrm{data}}}{p^\theta} \, \leq \, \sum_{i = 1}^n \E_{Z_{t_i} \sim p_{t_i}} \Big[ \KL{p_{s_i \mid t_i}(\cdot\mid Z_{t_i})}{p^\theta_{s_i \mid t_i}(\cdot\mid Z_{t_i})} \Big] \, .
    \label{eq:discrete-bound}
\end{equation}
In general, the exact reverse transition contains dependencies between
positions that the factorized model cannot represent. For a fixed $z_t$,
each transition error separates into tokenwise marginal errors and a
term accounting for these dependencies:
\begin{equation}
    \KL{p_{s \mid t}(\cdot\mid z_t)}{p^\theta_{s \mid t}(\cdot\mid z_t)} \, = \, \sum_{\ell = 1}^L \KL{p^\ell_{s \mid t}(\cdot\mid z_t)}{p^{\theta, \ell}_{s \mid t}(\cdot\mid z_t)} \; + \; \mathcal{I}_{s \mid t}(z_t) \, ,
    \label{eq:discrete-tokenwise}
\end{equation}
where $ p^\ell_{s \mid t}(\cdot\mid z_t) = \sum_{z_s^{-\ell}} p_{s \mid t}(z_s \mid z_t) $ is the $ \ell $-th token marginal of the exact transition and $ \mathcal{I}_{s \mid t}(z_t) = \KL{p_{s \mid t}(\cdot\mid z_t)}{\prod_\ell p^\ell_{s \mid t}(\cdot\mid z_t)} $ measures the conditional dependence between positions under the exact
reverse transition.
Only the marginal errors depend on $\theta$. Thus, within the factorized
family, the optimal transition matches each exact token marginal; the
remaining term $\mathcal{I}_{s\mid t}(z_t)$ is the irreducible error due
to factorization at this time step. This motivates parameterizing the
model through predictions that recover these marginals.

\paragraph{The bridge plug-in.}
The bound asks for a reverse transition, whereas the network produces a point of $ \Delta_K^L $.
Nearly every work joins the two by \emph{plugging in}, evaluating the bridge as though the prediction were the clean token \citep{sahoo2024simple, schiff2025simple},
\begin{equation}
    p^{\theta, \ell}_{s \mid t}(\cdot\mid z_t) \, = \, q^\ell_{s \mid t, \mathbf{x}} \big(\cdot\mid z_t^\ell, \, \mathbf{x}^{\theta, \ell}(z_t, t) \big) \, .
    \label{eq:discrete-plugin}
\end{equation}
Whether this costs anything is the question of whether the token marginals of \eqref{eq:discrete-tokenwise} are themselves of that form, and they are.

\begin{proposition}[\citealp{gourevitch2026uniform}]
\label{prop:plugin-exact}
Let $ \mathbf{x}^{\mathrm{loo}, \ell}(z_t, t) = \E [ \mathbf{x}^\ell \mid z_t^{-\ell} ] $ be the leave-one-out denoiser, and let the bridge be extended to a simplex-valued first argument by its Bayes-ratio form, as both processes admit.
Then, for every position $ \ell $,
\begin{equation}
    p^\ell_{s \mid t}(\cdot\mid z_t) \, = \, q^\ell_{s \mid t, \mathbf{x}} \big(\cdot\mid z_t^\ell, \, \mathbf{x}^{\mathrm{loo}, \ell}(z_t, t) \big) \, .
    \label{eq:discrete-exactness}
\end{equation}
When $ \pi $ is uniform this is the only simplex-valued field with that property; when $ \pi $ is the mask, the leave-one-out denoiser agrees with the denoiser $ \E [ \mathbf{x}^\ell \mid z_t ] $ at every masked position, and the latter has it too.
\end{proposition}

The plug-in at $ \mathbf{x}^\theta = \mathbf{x}^{\mathrm{loo}} $ therefore cancels the sum in \eqref{eq:discrete-tokenwise}, leaving only the factorization error.
The object it fits is the leave-one-out denoiser, which does not read the position it predicts.

\paragraph{Continuous time.}
Refining the grid removes the factorization error, two positions both jumping within one interval being a second-order event, so \eqref{eq:discrete-bound} becomes tight over the family \eqref{eq:discrete-generative} in the limit.
That limit is written on ratios rather than on transitions.
For a state $ z_t $, a position $ \ell $ and a token $ k $, write $ z_t^{\ell \to k} $ for $ z_t $ with its $ \ell $-th token replaced by $ k $; these are the states at Hamming distance one from $ z_t $.
The \emph{concrete score} of the marginal $ p_t $ collects the corresponding ratios,
\begin{equation}
    s^\ell_t(z_t)_k \, = \, \frac{p_t(z_t^{\ell \to k})}{p_t(z_t)} \, ,
    \label{eq:concrete-score}
\end{equation}
and the same ratios for $ q_{t \mid \mathbf{x}} $ are closed-form, only the $ \ell $-th factor of \eqref{eq:discrete-marginal} changing,
\begin{equation}
    s^\ell_t(z_t \mid \mathbf{x})_k \, = \, \frac{\big\langle \alpha_t \mathbf{x}^\ell + (1 - \alpha_t) \pi , \, e_k \big\rangle}{\big\langle \alpha_t \mathbf{x}^\ell + (1 - \alpha_t) \pi , \, z_t^\ell \big\rangle} \, .
    \label{eq:concrete-score-conditional}
\end{equation}
The second is defined for any point of $ \Delta_K^L $ in place of $ \mathbf{x} $, so the plug-in of \eqref{eq:discrete-plugin} carries over to it,
\begin{equation}
    s^{\theta, \ell}_t(z_t) \, = \, s^\ell_t \big( z_t \mid \mathbf{x}^\theta(z_t, t) \big) \, ,
    \label{eq:concrete-score-model}
\end{equation}
and returns \eqref{eq:concrete-score} exactly at $ \mathbf{x}^\theta = \mathbf{x}^{\mathrm{loo}} $, as in \Cref{prop:plugin-exact} \citep{gourevitch2026uniform}.
The limit of \eqref{eq:discrete-bound} compares the two scores through a Bregman divergence, weighted by the rate at which the reverse process leaves $ Z_t^\ell $ \citep{lou2024discrete},
\begin{equation}
    \mathsf{L}_{\mathrm{d}}(\theta) \, = \, \int_0^1 R_t \, \E_{Z_t \sim p_t} \bigg[ \sum_{\ell = 1}^L \big\langle \pi , \, Z_t^\ell \big\rangle \sum_{e_k \neq Z_t^\ell} d_F \Big( s^\ell_t(Z_t)_k \, , \; s^{\theta, \ell}_t(Z_t)_k \Big) \bigg] \, \mathrm{d} t \, ,
    \label{eq:discrete-ct}
\end{equation}
with $ R_t = - \alpha_t' / \alpha_t $ the substitution rate of \eqref{eq:discrete-transition}, the reverse rate towards $ z_t^{\ell \to k} $ being $ R_t \langle \pi, z_t^\ell \rangle \, s^\ell_t(z_t)_k $,
generated by $ F(u) = u \log u - u $, so that $ d_F(a, b) = a \log (a / b) - a + b $.
The data score \eqref{eq:concrete-score} is not available, but it is the average of \eqref{eq:concrete-score-conditional} over the denoising posterior, and a Bregman divergence is insensitive to that substitution up to a constant in $ \theta $: drawing $ \mathbf{X} \sim \post(\cdot\mid Z_t) $ and scoring against $ s^\ell_t(Z_t \mid \mathbf{X}) $ is the form actually minimized.
Nothing in \eqref{eq:discrete-ct} requires $ s^\theta_t $ to come from a denoiser: parameterizing the concrete score directly by a network is SEDD \citep{lou2024discrete}.
\section{Distributional distillation: discrepancies and gradients}
\label{app:distillation}

This appendix derives the objectives summarized in \Cref{tab:wip-local-discrepancies}, the gradient each method descends, and the identities relating them.
All methods are considered in the one-step setting of \Cref{sec:method-unified}: $ T = 1 $, the student starts from $ Z_1 \sim p_1 $, the terminal law of the diffusion considered, and its clean output $ \mathbf{X}^\eta \sim \genkernel_1^\eta(\cdot \mid Z_1) $ has law $ p_{\mathbf{x}}^\eta $, whose noised marginals are $ p_t^\eta $.

\subsection{Continuous distillation baselines: SiD, FGM and moment matching}
\label{app:continuous-baselines}
\label{app:sid}
\label{app:fgm-sid}

Score identity distillation, flow generator matching and moment matching distill continuous diffusion models on continuous data, the case $ \Emat = \Id $ of \Cref{sec:bg-continuous}.
Each is stated in a different field --- scores, velocities and conditional means --- and each estimates the gradient of the same objective differently.
We restate them on one path and identify what each one computes.

\paragraph{Path and fields.}
The student is deterministic, $ \mathbf{X}^\eta = \mathbf{x}^\eta(Z_1, 1) \in \R^d $ with $ Z_1 \sim p_1 = \gauss(0, \Id) $, and its output is corrupted along the path of \Cref{sec:bg-continuous},
\begin{equation}
    Z_t^\eta = \alpha_t \, \mathbf{X}^\eta + \sigma_t \varepsilon \, ,
    \qquad \varepsilon \sim \gauss(0, \Id) \, ,
    \label{eq:baseline-path}
\end{equation}
with $ p_t^\eta $ the law of $ Z_t^\eta $ and $ p_t $ that of the data under the same kernel $ q_{t \mid \mathbf{x}} $.
At each noise level the student law is described by any one of three fields, its conditional mean, its score and its velocity,
\begin{equation}
\begin{aligned}
    m_t^\eta(z_t) &= \E \big[ \mathbf{X}^\eta \mid Z_t^\eta = z_t \big] ,
    &s_t^\eta(z_t) &= \nabla_{z_t} \log p_t^\eta(z_t) , \\
    v_t^\eta(z_t) &= \E \big[ \dot{\alpha}_t \mathbf{X}^\eta + \dot{\sigma}_t \varepsilon \mid Z_t^\eta = z_t \big] ,
\end{aligned}
\label{eq:local-discrepancy-fields}
\end{equation}
and $ m_t, s_t, v_t $ denote the same fields under the data law.
Tweedie's identity, $ s_t^\eta = ( \alpha_t m_t^\eta - z_t ) / \sigma_t^2 $, and its velocity counterpart, $ v_t^\eta = \dot{\alpha}_t m_t^\eta + \dot{\sigma}_t ( z_t - \alpha_t m_t^\eta ) / \sigma_t $, make the three differences proportional at each noise level,
\begin{equation}
    s_t^\eta - s_t = \frac{\alpha_t}{\sigma_t^2} \big( m_t^\eta - m_t \big) ,
    \qquad
    v_t^\eta - v_t = \Big( \dot{\alpha}_t - \alpha_t \frac{\dot{\sigma}_t}{\sigma_t} \Big) \big( m_t^\eta - m_t \big) ,
    \label{eq:local-discrepancy-relations}
\end{equation}
so matching scores, velocities or conditional means is one criterion under three names, up to a weight in $ t $.
We work with the score.

\paragraph{The objective and its two gradient terms.}
The squared discrepancy between the two fields, integrated over noise levels, is
\begin{equation}
    \mathsf{L}(\eta) = \int_0^1 w(t) \, \E_{Z_t \sim p_t^\eta} \Big[ \big\Vert s_t^\eta(Z_t) - s_t(Z_t) \big\Vert^2 \Big] \, \mathrm{d} t \, .
    \label{eq:baseline-loss}
\end{equation}
Here $ \eta $ enters twice, through the law of the state and through the field evaluated at it.
Freezing one at a time gives two losses,
\begin{equation}
\begin{aligned}
    \mathsf{L}_1(\eta) \; &= \; \int_0^1 w(t) \, \E_{Z_1, \varepsilon} \Big[ \big\Vert \sg{s_t^\eta}(Z_t^\eta) - s_t(Z_t^\eta) \big\Vert^2 \Big] \, \mathrm{d} t \, , \\[3pt]
    \mathsf{L}_2(\eta) \; &= \; \int_0^1 w(t) \, \E_{Z_t \sim \sg{p_t^\eta}} \Big[ \big\Vert s_t^\eta(Z_t) - s_t(Z_t) \big\Vert^2 \Big] \, \mathrm{d} t \, ,
\end{aligned}
\label{eq:baseline-split}
\end{equation}
where $ \sg{\cdot} $ freezes the parameters of a field, or of the law a state is drawn from, but not the value at which the field is read.
Both equal $ \mathsf{L}(\eta) $, and their gradients are the two terms of \eqref{eq:wip-distill-grad},
\begin{equation}
    \nabla_\eta \mathsf{L}_1 = (A) \, , \qquad \nabla_\eta \mathsf{L}_2 = (B) \, , \qquad \nabla_\eta \mathsf{L} = (A) + (B) \, .
    \label{eq:baseline-two-losses}
\end{equation}
For the reverse KL the second vanishes, $ \E_{p_t^\eta} [ \nabla_\eta \log p_t^\eta ] = \nabla_\eta \! \int p_t^\eta = 0 $.
Here it does not, and $ \mathsf{L}_2 $ cannot be differentiated as it stands: the auxiliary network estimates $ s_t^\eta $ at the current parameters, not its derivative $ \nabla_\eta s_t^\eta $.

\paragraph{Trading the student field for the sample that produced it.}
The unknown field is removed from one side of the squared norm by an identity on the student's own path.

\begin{lemma}
\label{lem:field-swap}
Let $ \sigma_t > 0 $ on the support of $ w $. Then $ \mathsf{L}(\eta) = \mathsf{L}^{\mathrm{proj}}(\eta) $, where
\begin{equation}
    \mathsf{L}^{\mathrm{proj}}(\eta) = \int_0^1 w(t) \, \E_{Z_1, \varepsilon} \Big[ \big( s_t^\eta - s_t \big)^{\!\top} \big( s_t ( \cdot \mid \mathbf{X}^\eta ) - s_t \big) (Z_t^\eta) \Big] \, \mathrm{d} t \, ,
    \label{eq:projected-loss}
\end{equation}
and $ s_t(z_t \mid x) = \nabla_{z_t} \log q_{t \mid \mathbf{x}}(z_t \mid x) = ( \alpha_t x - z_t ) / \sigma_t^2 $ is the conditional score.
\end{lemma}

\begin{proof}
$ p_t^\eta(z_t) \, s_t^\eta(z_t) = \nabla_{z_t} p_t^\eta(z_t) = \int \nabla_{z_t} q_{t \mid \mathbf{x}}(z_t \mid x) \, p_{\mathbf{x}}^\eta(x) \, \mathrm{d} x = \int q_{t \mid \mathbf{x}}(z_t \mid x) \, s_t(z_t \mid x) \, p_{\mathbf{x}}^\eta(x) \, \mathrm{d} x $, so that $ \E_{Z_t \sim p_t^\eta} [ f(Z_t)^\top s_t^\eta(Z_t) ] = \E_{Z_1, \varepsilon} [ f(Z_t^\eta)^\top s_t(Z_t^\eta \mid \mathbf{X}^\eta) ] $ for any square-integrable $ f $.
Take $ f = s_t^\eta - s_t $ and subtract $ \E [ f^\top s_t ] $ from both sides.
\end{proof}

The marginal field on the left is unknown; the conditional field on the right is explicit, and differentiable in $ \eta $ through the sample that produced the state.
\Cref{eq:projected-loss} is the form SiD calls projected score matching \citep{zhou2024score}.
Freezing the field in it, as in \eqref{eq:baseline-split}, isolates the part of its gradient due to the sampling operation,
\begin{equation}
    \mathsf{L}_1^{\mathrm{proj}}(\eta) = \int_0^1 w(t) \, \E_{Z_1, \varepsilon} \Big[ \big( \sg{s_t^\eta} - s_t \big)^{\!\top} \big( s_t ( \cdot \mid \mathbf{X}^\eta ) - s_t \big) (Z_t^\eta) \Big] \, \mathrm{d} t \, ,
    \label{eq:projected-split}
\end{equation}
where $ \eta $ now reaches the objective twice over, through the state and through $ \mathbf{X}^\eta $ inside the conditional score.
That second route is what the swap has bought: $ \mathsf{L}_1^{\mathrm{proj}} $ and $ \mathsf{L}_1 $ agree in value, both being $ \mathsf{L}(\eta) $, but not in gradient, part of $ (B) $ having moved into the sampling operation.
The gap between them is exactly the missing term.

\begin{proposition}[\citealp{huang2024flow}]
\label{prop:b-term}
$ \nabla_\eta \mathsf{L}_2(\eta) = 2 \, \nabla_\eta \big( \mathsf{L}_1^{\mathrm{proj}}(\eta) - \mathsf{L}_1(\eta) \big) $.
\end{proposition}

FGM states this in the velocity field and proves it from the identity of \Cref{lem:field-swap} \citep{huang2024flow}.
The exact gradient of \eqref{eq:baseline-loss} is therefore the gradient of a computable loss,
\begin{equation}
    \nabla_\eta \mathsf{L}(\eta) \, = \, \nabla_\eta \big( 2 \, \mathsf{L}_1^{\mathrm{proj}}(\eta) - \mathsf{L}_1(\eta) \big) \, ,
    \label{eq:baseline-exact-grad}
\end{equation}
in which the student field appears only frozen, at the current parameters, which is what the auxiliary network provides.

\paragraph{What each method computes.}
The three baselines descend $ \mathsf{L}_1^{\mathrm{proj}} $ and $ \mathsf{L}_1 $ in different proportions.
FGM takes the combination \eqref{eq:baseline-exact-grad} and so descends the exact gradient \citep{huang2024flow}.
SiD descends $ \mathsf{L}_1^{\mathrm{proj}} - \alpha \mathsf{L}_1 $ for a scalar $ \alpha $ \citep{zhou2024score}, whose gradient is, by \Cref{prop:b-term},
\begin{equation}
    \nabla_\eta \big( \mathsf{L}_1^{\mathrm{proj}} - \alpha \, \mathsf{L}_1 \big) \, = \, ( 1 - \alpha ) \, (A) \, + \, \tfrac{1}{2} \, (B) \, .
    \label{eq:sid-grad}
\end{equation}
FGM is therefore SiD at $ \alpha = 1/2 $, doubled, the factor being absorbed in the step size.
The values reported as best for SiD, $ \alpha \in [0.75, 1.2] $ \citep{zhou2024score}, sit past that point; at $ \alpha = 1 $ the pathwise term disappears and only $ \tfrac{1}{2} (B) $ is left, and beyond it the pathwise term is descended with the opposite sign.
The one-step version of moment matching descends \eqref{eq:projected-split}, the case $ \alpha = 0 $, with $ Z_t^\eta $ frozen in addition \citep{salimans2024multistep}.
The three are one family, read off a single pair of losses.

\subsection{Discrete distillation as a Bregman divergence between concrete scores}
\label{app:bregman-score}

D-MMD \citep{hoogeboom2026dmmd} and IDLM \citep{li2026idlm} distill discrete diffusion models, and their few-step constructions differ: D-MMD draws the next state from the bridge, IDLM rather uses the forward process.
At one step they descend the same objective, derived once here.

\paragraph{The objective.}
The student draws $ Z_1 \sim p_1 $ and emits $ \mathbf{X}^\eta \sim \genkernel_1^\eta(\cdot \mid Z_1) $, of sequence law $ p_{\mathbf{x}}^\eta $, and the criterion is the reverse Kullback--Leibler divergence $ \KL{p_{\mathbf{x}}^\eta}{p_{\mathrm{data}}} $.
Diffuse both laws along \eqref{eq:discrete-transition}.
Their reverse chains are Markov on the same grid and agree at $ t = 1 $, so \eqref{eq:discrete-bound} holds with the two exchanged,
\begin{equation}
    \KL{p_{\mathbf{x}}^\eta}{p_{\mathrm{data}}} \, \leq \, \sum_{i = 1}^n \E_{Z_{t_i} \sim p^\eta_{t_i}} \Big[ \KL{p^\eta_{s_i \mid t_i}(\cdot\mid Z_{t_i})}{p_{s_i \mid t_i}(\cdot\mid Z_{t_i})} \Big] \, ,
    \label{eq:dmd-discrete-bound}
\end{equation}
and \Cref{prop:plugin-exact} applies to both sides, each token marginal being a bridge with a leave-one-out denoiser plugged in --- the student's on the left, the data's on the right.
The passage to continuous time of \Cref{app:discrete} runs unchanged and returns a Bregman divergence between the two concrete scores, read at the states the student visits,
\begin{equation}
    \KL{p_{\mathbf{x}}^\eta}{p_{\mathrm{data}}} \, \leq \, \int_0^1 R_t \, \E_{Z_t \sim p^\eta_t} \bigg[ \sum_{\ell = 1}^L \big\langle \pi , \, Z_t^\ell \big\rangle \sum_{e_k \neq Z_t^\ell} d_F \Big( s^{\eta, \ell}_t(Z_t)_k \, , \; s^\ell_t(Z_t)_k \Big) \bigg] \, \mathrm{d} t \, .
    \label{eq:discrete-dmd}
\end{equation}
Neither score is available in distillation: the teacher stands in for $ s_t $, and an auxiliary model, fitted on the student's own samples, for $ s^\eta_t $.

\paragraph{Why a difference of two training losses computes it.}
Neither method evaluates \eqref{eq:discrete-dmd}.
Both subtract two per-sample training losses of the kind \Cref{app:discrete} ends on, evaluated on student samples.
For a simplex-valued field and the concrete score $ \hat s $ it induces through \eqref{eq:concrete-score-model}, that loss is \citep{lou2024discrete}
\begin{equation}
    \ell_t(\hat s ; z_t, \mathbf{x}) \, = \, \sum_{\ell = 1}^L \big\langle \pi , \, z_t^\ell \big\rangle \sum_{e_k \neq z_t^\ell} \Big[ \hat s^\ell(z_t)_k \, - \, s^\ell_t(z_t \mid \mathbf{x})_k \, \log \hat s^\ell(z_t)_k \Big] \, ,
    \label{eq:dse}
\end{equation}
the integrand of \eqref{eq:discrete-ct} up to a term free of $ \hat s $: an evidence bound rather than a divergence, minimized over $ \mathbf{X} \sim \post[\eta](\cdot\mid z_t) $ at $ \hat s = s^\eta_t $.
Subtracting two of them removes that free term and leaves the divergence,
\begin{equation}
    \E_{\mathbf{X} \sim \post[\eta](\cdot\mid z_t)} \Big[ \ell_t(s^\theta_t ; z_t, \mathbf{X}) - \ell_t(s^\eta_t ; z_t, \mathbf{X}) \Big] \, = \, \sum_{\ell = 1}^L \big\langle \pi , \, z_t^\ell \big\rangle \sum_{e_k \neq z_t^\ell} d_F \Big( s^{\eta, \ell}_t(z_t)_k \, , \; s^{\theta, \ell}_t(z_t)_k \Big) \, ,
    \label{eq:dmd-difference}
\end{equation}
which is the integrand of \eqref{eq:discrete-dmd} with the teacher in place of the data.
IDLM writes it as the teacher's loss on student samples minus its minimum over auxiliary fields \citep{li2026idlm}; D-MMD adds a term pulling the auxiliary towards the teacher, which leaves the fixed point where it is \citep{hoogeboom2026dmmd}.
Under the mask this divergence takes a simpler form still, derived in \Cref{app:masked-distillation}.

\subsection{Masked Diffusion Distillation: collapse to a posterior KL}
\label{app:masked-distillation}

Take $ \pi = e_m $, the point mass on the mask symbol.
The weight $ \langle \pi, z_t^\ell \rangle $ of \eqref{eq:discrete-dmd} is then $ 1 $ at masked positions and $ 0 $ elsewhere, so only masked positions are scored.
At one of them, for every $ k \neq m $, the concrete score is the denoiser up to a factor that depends on $ t $ alone,
\begin{equation}
    s^\ell_t(z_t)_k \, = \, \frac{\alpha_t}{1 - \alpha_t} \, m^\ell_t(z_t)_k \, ,
    \label{eq:masked-score}
\end{equation}
where $ m^\ell_t(z_t) = \E [ \mathbf{X}^\ell \mid Z_t = z_t ] $ is the data denoiser, equal at masked positions to the leave-one-out denoiser of \Cref{prop:plugin-exact}.
The same holds for the student score with $ m^{\eta, \ell}_t(z_t) = \E [ \mathbf{X}^{\eta, \ell} \mid Z_t^\eta = z_t ] $, the denoiser of the noised student distribution, not the student generator.
Replacing the student and data scores in \eqref{eq:discrete-dmd} accordingly leaves the loss
\begin{equation}
    \KL{p_{\mathbf{x}}^\eta}{p_{\mathrm{data}}} \, \leq \, \int_0^1 \frac{- \alpha_t'}{1 - \alpha_t} \, \E_{Z_t \sim p^\eta_t} \bigg[ \sum_{\ell \, : \, Z_t^\ell = e_m} \KL{m^{\eta, \ell}_t(Z_t)}{m^\ell_t(Z_t)} \bigg] \, \mathrm{d} t \, ,
    \label{eq:masked-posterior-kl}
\end{equation}
where $ m^\eta_t $ and $ m_t $ are approximated in practice by an auxiliary denoiser and the teacher $ \mathbf{x}^\theta $,
yielding the discrepancy \Cref{tab:wip-local-discrepancies} lists for DiMO \citep{zhu2025dimo} and the loss D-MMD implements in its masked instance \citep{hoogeboom2026dmmd}.
Nothing of the sort happens under the uniform law, where every position is scored and the denominator of \eqref{eq:concrete-score-conditional} keeps the head in it.
\section{Method: derivations and proofs}
\label{app:method-derivations}

The derivations behind \Cref{sec:method}: the gradient the objective is trained with, the
proof that the vertex constraint is recovered, the discriminator that supplies the log-ratio,
and the alternative objective we distill as a baseline.

\subsection{Distribution matching distillation: gradient derivation}
\label{app:dmd-grad}

We derive the pathwise gradient used by \SDMD in
\eqref{eq:dmd-grad}. The key observation is that differentiating the
reverse KL produces both a contribution through the generated latent and
an explicit density derivative; the latter has zero expectation.
Throughout, the embedding matrix $\Emat$, the noise schedule, and the
sampling distribution $\nu$ are fixed. We assume sufficient regularity
to interchange differentiation and integration.

\paragraph{Reparameterized student samples.}
For $(t,T)\sim\nu$, draw $Z_T\sim p_T$ and independent noise
$\varepsilon\sim\gauss(0,\Id)$. Under the continuous relaxation,
\begin{equation}
    \mathbf{X}^{\eta,T}=\mathbf{x}^\eta(Z_T,T),
    \qquad
    Z_t^{\eta,T}=\pi_t(\mathbf{X}^{\eta,T},\varepsilon)
    =\alpha_t\mathbf{x}^\eta(Z_T,T)\Emat+\sigma_t\varepsilon.
\end{equation}
This construction realizes the marginal $p_t^{\eta,T}$ on a sampling
space whose law does not depend on $\eta$. The multi-step objective is
therefore
\begin{equation}
    \mathsf{L}_{\mathrm{DMD}}^{\mathrm{ms}}(\eta)
    =\E_{\substack{(t,T)\sim\nu\\
        Z_T\sim p_T,\,\varepsilon\sim\gauss(0,\Id)}}
    \left[\log\frac{p_t^{\eta,T}(Z_t^{\eta,T})}
                        {p_t(Z_t^{\eta,T})}\right].
\end{equation}

\paragraph{Differentiating the reverse KL.}
Fix $(t,T)$ and write
$s^{\eta,T}(z,t)=\nabla_z\log p_t^{\eta,T}(z)$ and
$s(z,t)=\nabla_z\log p_t(z)$. Holding $Z_T$ and $\varepsilon$ fixed,
the chain rule gives
\begin{equation}
\begin{aligned}
    \nabla_\eta\log\frac{p_t^{\eta,T}(Z_t^{\eta,T})}
                            {p_t(Z_t^{\eta,T})}
    ={}&\Big(s^{\eta,T}(Z_t^{\eta,T},t)
             -s(Z_t^{\eta,T},t)\Big)^\top
          \frac{\partial Z_t^{\eta,T}}{\partial\eta}\\
       &+\left.\nabla_\eta\log p_t^{\eta,T}(z)
          \right|_{z=Z_t^{\eta,T}}.
\end{aligned}
\end{equation}
In the last term, the density is differentiated at a fixed argument $z$.
Its expectation vanishes because $p_t^{\eta,T}$ integrates to one:
\begin{equation}
\begin{aligned}
    \E_{Z_t^{\eta,T}\sim p_t^{\eta,T}}
    \left[\left.\nabla_\eta\log p_t^{\eta,T}(z)
        \right|_{z=Z_t^{\eta,T}}\right]
    &=\int_{\R^{L\times d}}\nabla_\eta p_t^{\eta,T}(z)\,\mathrm{d}z\\
    &=\nabla_\eta 1=0.
\end{aligned}
\end{equation}
This is the cancellation of contribution $(B)$ in
\eqref{eq:wip-distill-grad}; it also holds for the categorical student,
although that parameterization does not admit the pathwise derivative.

Averaging the remaining term yields
\begin{equation}
    \nabla_\eta\mathsf{L}_{\mathrm{DMD}}^{\mathrm{ms}}(\eta)
    =\E_{\substack{(t,T)\sim\nu\\
        Z_T\sim p_T,\,\varepsilon\sim\gauss(0,\Id)}}
    \left[
        \Big(s^{\eta,T}(Z_t^{\eta,T},t)
            -s(Z_t^{\eta,T},t)\Big)^\top
        \frac{\partial Z_t^{\eta,T}}{\partial\eta}
    \right],
\end{equation}
which is \eqref{eq:dmd-grad}. Latent coordinates are flattened when
forming the score--Jacobian product. The latent Jacobian includes the
embedding map and the signal coefficient:
\begin{equation}
    \frac{\partial Z_t^{\eta,T}}{\partial\eta}
    =\alpha_t\frac{\partial
        [\mathbf{x}^\eta(Z_T,T)\Emat]}{\partial\eta}.
\end{equation}
The one-step gradient follows by fixing $T=1$, sampling $t$ uniformly,
and identifying $p_t^{\eta,1}$ with $p_t^\eta$.

\paragraph{Estimated scores.}
The identity above uses exact scores. In practice, the frozen teacher
and the auxiliary denoiser provide their estimates through
\eqref{eq:sdmd-student-score}. During a student update, the estimated
score difference is detached and gradients flow only through the
reparameterized latent. Thus the implemented update approximates the
exact pathwise gradient, with accuracy determined by the score estimates.

\subsection{Recovering the discrete data distribution}
\label{app:emergent}

This subsection proves \Cref{thm:emergent}: matching the embedded, noised marginals recovers the clean data distribution, even when the student outputs continuous probability vectors.
We then explain how the same argument identifies the transports from noisy data marginals to the clean data distribution used for multi-step generation.

\paragraph{Clean and noised distributions.}
We use the clean student law $p_{\mathbf{x}}^\eta$ from \Cref{sec:method-unified} and identify discrete sequences with the one-hot vertex set $V=\{\mathbf{e}_1,\ldots,\mathbf{e}_K\}^L$ of $\Delta_K^L$.
The data law $p_{\mathrm{data}}$ is supported on $V$, whereas a relaxed student can place mass throughout $\Delta_K^L$.
For a clean distribution $\rho$ on this simplex, let $Q_t\rho$ denote the law of
\[
    Z_t=\alpha_t\mathbf{X}\Emat+\sigma_t\varepsilon,
    \qquad \mathbf{X}\sim\rho,\quad
    \varepsilon\sim\gauss(0,\Id),
\]
with independent Gaussian noise.
Thus $Q_t p_{\mathbf{x}}^\eta=p_t^\eta$ and $Q_t p_{\mathrm{data}}=p_t$.
Writing the objective as a function of the clean law gives
\begin{equation}
    \mathsf{L}_{\mathrm{DMD}}(\rho)
    =\int_0^1\KL{Q_t\rho}{p_t}\,\mathrm{d}t,
    \qquad
    \mathsf{L}_{\mathrm{DMD}}(\eta)
    =\mathsf{L}_{\mathrm{DMD}}(p_{\mathbf{x}}^\eta),
    \label{eq:emergent-functional}
\end{equation}
which is exactly \eqref{eq:ikl} for the student law.
Here $\rho$ denotes a clean distribution; $\nu$ retains its meaning as the time-sampling distribution in the multi-step objective.
For the proofs, write $\mathcal{E}(\mathbf{x})=\mathbf{x}\Emat$ for the position-wise embedding map.
We assume that $\alpha_t$ and $\sigma_t$ are measurable, finite, and strictly positive for $t\in(0,1)$; endpoint values do not affect the integral, and the loss may be infinite.

\paragraph{Why equal-norm embeddings suffice.}
Constraining the token embeddings to a common sphere is standard in the lineage of \Cref{sec:bg-continuous}: CDCD $ L_2 $-normalizes them before every use and backpropagates through the normalization \citep{dieleman2022continuous}, LangFlow places them on a sphere of radius $ \sqrt{d} $ \citep{chen2026langflow}, and RePlaid constrains each row to unit length \citep{yang2026replaid}.
If the rows are also distinct, this normalization ensures that no token embedding is a convex combination of the others, as the following proposition shows.

\begin{proposition}
\label{prop:sphere}
Distinct rows sharing a common Euclidean norm are in convex position.
\end{proposition}

\begin{proof}
Let $ \Vert E_k \Vert = R $ for all $ k $ and suppose $ E_i = \sum_{k \neq i} a_k E_k $ with $ a \in \Delta_{K-1} $.
For $ k \neq i $, $ \Vert E_i - E_k \Vert^2 = 2R^2 - 2 \langle E_i, E_k \rangle $ is positive by distinctness, so $ \langle E_i, E_k \rangle < R^2 $.
Taking the inner product of the supposed combination with $ E_i $ gives $ R^2 = \sum_{k \neq i} a_k \langle E_i, E_k \rangle < R^2 $.
\end{proof}

More generally, the argument below only requires distinct rows in convex position; it does not require the embedding map to be injective on the whole simplex.

\paragraph{Proof of the theorem.}
We first show that a token embedding can only come from its one-hot vector. We then undo Gaussian noising to identify the clean distribution.

\begin{lemma}
\label{lem:lifting}
Under the assumption of \Cref{thm:emergent}, $ x \in \Delta_K $ and $ x \Emat = E_i $ imply $ x = \mathbf{e}_i $.
Hence $ \mathcal{E}^{-1}(\mathcal{E}(V)) = V $, and $ \mathcal{E} $ is injective on $ V $.
\end{lemma}

\begin{proof}
From $ E_i = x_i E_i + \sum_{k \neq i} x_k E_k $, if $ x_i < 1 $ then dividing by $ 1 - x_i $ gives $ E_i = \sum_{k \neq i} \frac{x_k}{1 - x_i} E_k $, a convex combination of the other rows, contradicting convex position.
Hence $ x_i = 1 $.
Applying this at every position gives the second claim, and injectivity on $ V $ follows from distinctness of the rows.
\end{proof}

\begin{lemma}[Identification at one noise level]
\label{lem:one-level}
Under the assumption of \Cref{thm:emergent}, for any probability measure $ \rho $ on $ \Delta_K^L $ and any $ t \in (0,1) $, $ Q_t \rho = Q_t p_{\mathrm{data}} $ holds if and only if $ \rho = p_{\mathrm{data}} $.
\end{lemma}

\begin{proof}
Only the forward implication needs an argument.
Both measures are Gaussian smoothings, so their characteristic functions factor as $ \widehat{(\alpha_t \mathcal{E})_\# \rho} \cdot \hat{g} $ with $ \hat{g}(u) = e^{- \sigma_t^2 \Vert u \Vert^2 / 2} $ nowhere zero.
Dividing by $ \hat{g} $ and using $ \alpha_t \neq 0 $ gives $ \mathcal{E}_\# \rho = \mathcal{E}_\# p_{\mathrm{data}} $.
That measure is carried by $ \mathcal{E}(V) $, so $ \rho $ is carried by $ \mathcal{E}^{-1}(\mathcal{E}(V)) = V $ by \Cref{lem:lifting}.
On that finite set $ \mathcal{E} $ is injective, so equality of the embedded laws is equality of the mass at every vertex sequence.
\end{proof}

\begin{proof}[Proof of \Cref{thm:emergent}]
We prove that $ \mathsf{L}_{\mathrm{DMD}}(\rho) = 0 $ if and only if $ \rho = p_{\mathrm{data}} $, for any probability measure $ \rho $ on $ \Delta_K^L $; the theorem is the case $ \rho = p_{\mathbf{x}}^\eta $.
The reverse implication is immediate, $ Q_t \rho $ being a function of $ \rho $ alone.
Conversely, nonnegativity of the KL divergence forces $ Q_t \rho = Q_t p_{\mathrm{data}} $ for almost every $ t \in (0,1) $, and any such $ t $ with \Cref{lem:one-level} closes the argument.
\end{proof}

\paragraph{Implications for multi-step generation.}
In \Cref{sec:method-fewstep}, the student starts from $Z_T\sim p_T$ and produces a clean sample with law $p_{\mathbf{x}}^{\eta,T}$.
For each fixed $T>0$, \Cref{lem:one-level} applies to this law just as in the one-step case: matching $p_t^{\eta,T}$ to $p_t$ at any noise level $0<t<T$ implies $p_{\mathbf{x}}^{\eta,T}=p_{\mathrm{data}}$.
Consequently, a zero multi-step objective \eqref{eq:fewstep-proposal-loss} gives, for $\nu_1$-almost every starting level $T$,
\begin{equation}
    \int \genkernel_T^\eta(\mathrm{d}\mathbf{x}\mid z_T)\,
        p_T(\mathrm{d}z_T)
    =p_{\mathrm{data}}(\mathrm{d}\mathbf{x}).
\end{equation}
Indeed, nonnegativity of the KL divergence implies matching for $\nu_T$-almost every $t<T$, and any such level identifies the clean law.
Thus the objective learns a family of transports from $p_T$ to $p_{\mathrm{data}}$, indexed by the starting noise level.

\subsection{The optimal discriminator and the log-ratio}
\label{app:discriminator}

We show that the population optimum of \eqref{eq:disc-loss} recovers
the log-density ratio required by \RDMD. This is the standard optimal
discriminator argument of \citet{goodfellow2014gan}, applied to the
noised data and student distributions at each pair $(t,T)$.

\begin{proposition}
\label{prop:optimal-discriminator}
Fix the student parameters $\eta$ and a pair $0<t<T\leq1$ with
$\sigma_t>0$. Over measurable functions
$D:\R^{L\times d}\to(0,1)$, the objective
\begin{equation}
    \ell_{t,T}(D)
    =-\E_{Z_t\sim p_t}\!\left[\log D(Z_t)\right]
     -\E_{Z_t^{\eta,T}\sim p_t^{\eta,T}}
       \!\left[\log\big(1-D(Z_t^{\eta,T})\big)\right]
\end{equation}
has the unique minimizer, up to sets of Lebesgue measure zero,
\begin{equation}
    D^\star(z,t,T)
    =\frac{p_t(z)}{p_t(z)+p_t^{\eta,T}(z)}.
\end{equation}
Consequently,
\begin{equation}
    \log\frac{1-D^\star(z,t,T)}{D^\star(z,t,T)}
    =\log\frac{p_t^{\eta,T}(z)}{p_t(z)}.
\end{equation}
\end{proposition}

\begin{proof}
Gaussian noising with $\sigma_t>0$ makes both densities strictly positive.
Writing the objective as
\begin{equation}
    \ell_{t,T}(D)
    =\int_{\R^{L\times d}}
      \left[-p_t(z)\log D(z)
            -p_t^{\eta,T}(z)\log\big(1-D(z)\big)\right]
      \,\mathrm{d}z,
\end{equation}
we can minimize the integrand separately at each $z$. Set
$a=p_t(z)>0$ and $b=p_t^{\eta,T}(z)>0$. For $u\in(0,1)$,
let $f(u)=-a\log u-b\log(1-u)$. Its derivatives are
\begin{equation}
    f'(u)=-\frac{a}{u}+\frac{b}{1-u},
    \qquad
    f''(u)=\frac{a}{u^2}+\frac{b}{(1-u)^2}>0.
\end{equation}
Thus $f$ is strictly convex, and its unique minimizer solves
$f'(u)=0$, giving $u=a/(a+b)$. Substitution yields $D^\star$ and
$(1-D^\star)/D^\star=b/a$, which proves the log-ratio identity.
\end{proof}

Averaging these objectives over $(t,T)\sim\nu$ gives
\eqref{eq:disc-loss}. Hence an unrestricted discriminator conditioned on
both noise levels has the stated optimum for $\nu$-almost every pair.

\paragraph{Leave-one-out baseline.}
\label{app:reinforce-baseline}
For a fixed discriminator, the leave-one-out baseline $\widehat b_g$
is independent of $\mathbf{X}_g^{\eta,T}$ conditionally on $(Z_T,T)$.
Since the conditional score function has zero expectation, subtracting
this baseline preserves the expected cost-based update \citep{kool2019buy}.

\subsection{Auxiliary score estimation with bridge renoising}
\label{app:fewstep-bridge}

The multi-step formulation in \Cref{sec:method-fewstep} can also use the bridge $q_{t\mid T,\mathbf{x}}$ in place of forward noising. For \SDMD, consider $Z_T\sim p_T$ and $Z_t\sim q_{t\mid T,\mathbf{x}}(\cdot\mid Z_T,\mathbf{x}^\eta(Z_T,T))$, and denote the resulting marginal by $p_t^{\eta,T\to t}$. Taking the conditional expectation of the Gaussian bridge score, with the coefficients of \eqref{eq:gaussian_bridge}, gives
\[
    \nabla_{z_t}\log p_t^{\eta,T\to t}(z_t)
    =\frac{
        \alpha_t\sigma_{T\mid t}^2\E[\mathbf{x}^\eta(Z_T,T)\mid Z_t=z_t]\Emat
        +\alpha_{T\mid t}\sigma_t^2\E[Z_T\mid Z_t=z_t]
        -\sigma_T^2z_t
    }{\sigma_t^2\sigma_{T\mid t}^2},
\]
for positive bridge variance, where both expectations are under this joint law. Unlike forward noising, the exact bridge score therefore involves the additional conditional mean $\E[Z_T\mid Z_t]$, which is not estimated by the auxiliary token denoiser.

In practice, we retain the usual auxiliary score estimator by approximating the conditional law of $Z_T$ given $Z_t=z_t$ with the forward Gaussian transition
$\mathcal N(\alpha_{T\mid t}z_t,\sigma_{T\mid t}^2\Id)$.
Although this conditional law does not generally hold under bridge renoising, it gives
$\E[Z_T\mid Z_t=z_t]\approx\alpha_{T\mid t}z_t$.
Using $\sigma_T^2=\alpha_{T\mid t}^2\sigma_t^2+\sigma_{T\mid t}^2$, the score expression then reduces to the usual Tweedie form,
\[
    \nabla_{z_t}\log p_t^{\eta,T\to t}(z_t)
    \approx
    \frac{
        \alpha_t\E[\mathbf{x}^\eta(Z_T,T)\mid Z_t=z_t]\Emat-z_t
    }{\sigma_t^2},
\]
whose conditional token mean is estimated by the auxiliary denoiser. Thus, retaining this estimator under bridge renoising introduces an approximation, whereas it follows exactly from the forward-noising construction.

Estimating the exact bridge score would require an additional prediction task for $\E[Z_T\mid Z_t]$, for example through a separate network or output head, or an auxiliary model predicting the combined bridge mean directly. We leave such estimators to future work. Related experiments on continuous data, in Appendix D of \citet{salimans2024multistep}, reported poor results when fitting the exact bridge score. We do not evaluate bridge renoising during \RDMD\ training.

\subsection{Multi-step generation}
\label{app:fewstep-sampling}

We justify here the DDIM-inspired family of transitions introduced in
\Cref{sec:method-fewstep}. The following proposition gives two conditions
under which a transition preserves the exact diffusion marginals.

\begin{proposition}
\label{prop:fewstep-marginal-consistency}

Fix $0\leq t<T\leq1$ and suppose that
$\widehat Z_T\sim p_T$. Given $\widehat Z_T$, sample
\[
\widehat{\mathbf X}
\sim
\genkernel_T^\eta(\cdot\mid\widehat Z_T),
\]
and then
\[
\widehat Z_t
\sim
q^\zeta_{t\mid\mathbf{x},T}
(\cdot\mid\widehat{\mathbf X},\widehat Z_T).
\]

Then either of the following conditions implies
$\widehat Z_t\sim p_t$:

\begin{enumerate}
    \item $\zeta_{t,T}=\sigma_t$ and the marginal distribution of
    $\widehat{\mathbf X}$ is $p_{\mathrm{data}}$;

    \item
    $\genkernel_T^\eta(\cdot\mid z_T)
    =p_{\mathbf{x}\mid T}(\cdot\mid z_T)$ for $p_T$-almost every $z_T$.
    In this case the conclusion holds for every
    $0\leq\zeta_{t,T}\leq\sigma_t$.
\end{enumerate}
\end{proposition}

\begin{proof}
For $\zeta_{t,T}=\sigma_t$, the coefficient multiplying the noise inferred
from $\widehat Z_T$ vanishes, and the transition becomes
\[
\widehat Z_t
=
\alpha_t\widehat{\mathbf X}\Emat
+\sigma_t\varepsilon,
\qquad
\varepsilon\sim\gauss(0,\Id).
\]
If $\widehat{\mathbf X}\sim p_{\mathrm{data}}$, this is exactly the forward
construction defining $p_t$.

For the second statement, sampling
$\widehat{\mathbf X}$ from the exact posterior given
$\widehat Z_T\sim p_T$ reconstructs the exact joint law of
$(\mathbf X,Z_T)$ under the forward process. Hence we may write
\[
\widehat Z_T
=
\alpha_T\widehat{\mathbf X}\Emat
+\sigma_T\varepsilon_T,
\qquad
\varepsilon_T\sim\gauss(0,\Id),
\]
with $\varepsilon_T$ independent of $\widehat{\mathbf X}$.
A draw from \eqref{eq:fewstep-transition-kernel} can therefore be written as
\[
\widehat Z_t
=
\alpha_t\widehat{\mathbf X}\Emat
+\sqrt{\sigma_t^2-\zeta_{t,T}^2}\,\varepsilon_T
+\zeta_{t,T}\varepsilon,
\]
where $\varepsilon\sim\gauss(0,\Id)$ is independent of
$(\widehat{\mathbf X},\varepsilon_T)$.
The two Gaussian noise terms combine into a
$\gauss(0,\sigma_t^2\Id)$ variable independent of
$\widehat{\mathbf X}$, so that
\[
\widehat Z_t
\mid \widehat{\mathbf X}
\sim
\gauss(
\alpha_t\widehat{\mathbf X}\Emat,
\sigma_t^2\Id).
\]
Since $\widehat{\mathbf X}\sim p_{\mathrm{data}}$, marginalizing over
$\widehat{\mathbf X}$ gives $\widehat Z_t\sim p_t$.
\end{proof}

Applying \Cref{prop:fewstep-marginal-consistency} recursively along a
decreasing grid shows that, if its assumptions hold at every transition,
then
\[
\widehat Z_{T_i}\sim p_{T_i},
\qquad i=0,\ldots,n.
\]
In particular, for forward renoising, exact distribution matching at each
starting level implies the required clean-marginal condition under the
geometric assumption of \Cref{thm:emergent}.
\section{Supplementary experiments}
\label{app:supp}

\subsection{Implementation details}
\label{app:practical}

\paragraph{Training budgets and model sizes.}
\Cref{tab:training-budgets} separates teacher pretraining from the
additional updates used to train the distilled students.
Both the student and auxiliary models retain the LangFlow architecture
\citep{chen2026langflow}; we use an embedding-space diffusion transformer with $12$ blocks,
hidden dimension $768$, $12$ attention heads and rotary position embeddings, conditioned on the
noise level through adaLN modulation, with an output head over the vocabulary that is not tied to $\Emat$.
Each network has $170.8$M parameters, of which the $38.6$M of $\Emat$ are frozen and shared with the teacher.
For \RDMD, the discriminator keeps this trunk and replaces the vocabulary head with a scalar output per position.
Optimization-step counts alone do not imply equal training compute:
batch sizes and the number of network evaluations per update can differ,
and our methods alternate auxiliary and student updates.
Each student is trained for $10{,}000$ optimizer updates in total, comprising $8{,}000$ auxiliary updates and $2{,}000$ student updates.
Our training GPU budget comprises four nodes with eight NVIDIA H100 GPUs each ($32$ GPUs in total).

\begin{table}[H]
    \centering
    \small
    \setlength{\tabcolsep}{4pt}
    \renewcommand{\arraystretch}{1.2}
    \begin{tabular}{@{}
        >{\raggedright\arraybackslash}p{0.30\linewidth}
        >{\raggedright\arraybackslash}p{0.14\linewidth}
        >{\raggedright\arraybackslash}p{0.13\linewidth}
        >{\raggedright\arraybackslash}p{0.20\linewidth}
        >{\raggedright\arraybackslash}p{0.13\linewidth}@{}}
        \toprule
        Method & Teacher & Teacher updates & Distillation updates & Student params. \\
        \midrule
        SDTT \citep{deschenaux2025sdtt} & MDLM & 1M
            & $r\times10$k\newline $r=1,\ldots,7$ & $\approx$170M \\
        IDLM \citep{li2026idlm} & MDLM & 1M
            & 1M & 169.63M \\
        FMLM \citep{lee2026fmlm} & FLM & 1.5M
            & 1M & $\approx$170M \\
        ReDi \citep{yoo2025redi} & DUO & $\sim$500k
            & $\leq$1M & $\approx$165M \\
        \midrule
        \SDMD (ours) & LangFlow & 1M & 10k & $\approx$170M \\
        \RDMD (ours) & LangFlow & 1M & 10k & $\approx$170M \\
        \bottomrule
    \end{tabular}
    \caption{Training budgets and model sizes for the distilled checkpoints.
    Teacher updates are pretraining steps; distillation updates are additional
    optimizer updates. For our methods, the 10k total includes 8k auxiliary and 2k student updates. SDTT uses successive round checkpoints, indexed by $r$.
    The ReDi student budget is an upper bound. Parameter counts refer to one sampling model,
    including embeddings, and exclude auxiliary training networks.}
    \label{tab:training-budgets}
\end{table}

For SDTT, the teacher is the MDLM teacher released with SDTT;
IDLM uses \texttt{kuleshov-group/mdlm-owt} and a learning rate of $10^{-6}$.
Both teachers are MDLM models \citep{sahoo2024simple}.
The FMLM checkpoint is distilled from its flow language model (FLM) teacher \citep{lee2026fmlm}.
ReDi uses the DUO \citep{sahoo2025duality} checkpoint \texttt{s-sahoo/duo}
(\texttt{duo.ckpt}). Both of our students use the same released
LangFlow checkpoint \citep{chen2026langflow}. For all the baselines, we use their default samplers presented in their associated paper.

\paragraph{Training and sampling configuration.}
\Cref{tab:recipe} summarizes the training and sampling configuration of
each student; the ablations of \Cref{sec:exp-new-design} and below vary its
components.
At inference, we divide the temporal interval $[0,1]$ uniformly, using $T_i=i/n$ for $i=1,\ldots,n$, and traverse this grid from $T_n = 1$ to $T_1$ and apply a final discrete decoder step at time $T_1$, yielding $n$ total NFEs.
Sampling temperature is applied at the network output: we divide the token logits by the temperature before the softmax that produces the clean token probabilities.

\begin{table}[htbp]
    \centering
    \small
    \begin{tabular}{lcc}
        \toprule
        Component & \SDMD & \RDMD \\
        \midrule
        Student initialization & teacher & teacher \\
        Auxiliary initialization & teacher & teacher \\
        Renoising kernel & $ q_{t \mid \mathbf{x}} $ & $ q_{t \mid \mathbf{x}} $ \\
        Self-conditioning & on & off \\
        KL anchor $ \beta $ & \textendash & $ 0.5 $ \\
        \midrule
        Sampler & forward & forward \\
        \bottomrule
    \end{tabular}
    \caption{Training and sampling configuration of each student. A dash marks a component the objective does not have. The last row is applied at inference.}
    \label{tab:recipe}
\end{table}

\paragraph{Discriminator architecture.}
The \RDMD discriminator uses the same transformer backbone as the teacher, including rotary positional embeddings and noise-level conditioning through adaptive layer normalization.
It operates directly on noised latents in the teacher's embedding space, omitting the token-embedding layer and replacing the vocabulary prediction head with a scalar head.
We initialize the backbone from the teacher checkpoint and initialize the scalar head separately.
The head applies noise-conditioned adaptive layer normalization followed by a projection to one logit per position; averaging these logits over sequence positions yields a single scalar logit per sequence.

\paragraph{Optimization.}
We use a global training batch size of $256$ and AdamW at a constant learning rate of $ 3 \times 10^{-4} $ after $ 2500 $ warmup steps, without weight decay, in \texttt{bf16}, with gradients clipped at $ 1.0 $ and an exponential moving average of decay $ 0.9999 $.
We use $n_{\mathrm{aux}}=4$: four auxiliary updates on detached student samples, followed by one student update.
At each optimization step, we sample the time pair $ (t, T) \sim \nu $.

\paragraph{Evaluation harness.}
We evaluate on OpenWebText with sequences of length $1024$, sweeping sampling temperatures $\{0.80, 0.85, 0.90, 0.95, 1.00, 1.05, 1.10\}$ applied to the output logits and sampling budgets $\{1, 2, 4, \ldots, 1024\}$, with the autoregressive references evaluated only at $1024$ NFE.
Generative perplexity is computed by exponentiating the mean token negative log-likelihood under GPT-2-large.
We evaluate with a single random seed, as in the default experimental setup of \citet[Appendix I.1]{gourevitch2026uniform}.
We pair perplexity with diversity measured as the mean per-sequence Shannon entropy of empirical GPT-2 token frequencies, in nats, to obtain a quality--diversity frontier at each sampling budget, following \citet{pynadath2026generativefrontiersevaluationmatters}.

\subsection{Full generative frontiers}
\label{app:frontiers}

\Cref{fig:frontier-all} extends the method comparison of
\Cref{fig:exp-proposal-frontiers} to all evaluated sampling budgets.

\begin{figure}[htbp]
    \centering
    \begin{subfigure}[t]{0.32\linewidth}
        \centering
        \includegraphics[width=\linewidth]{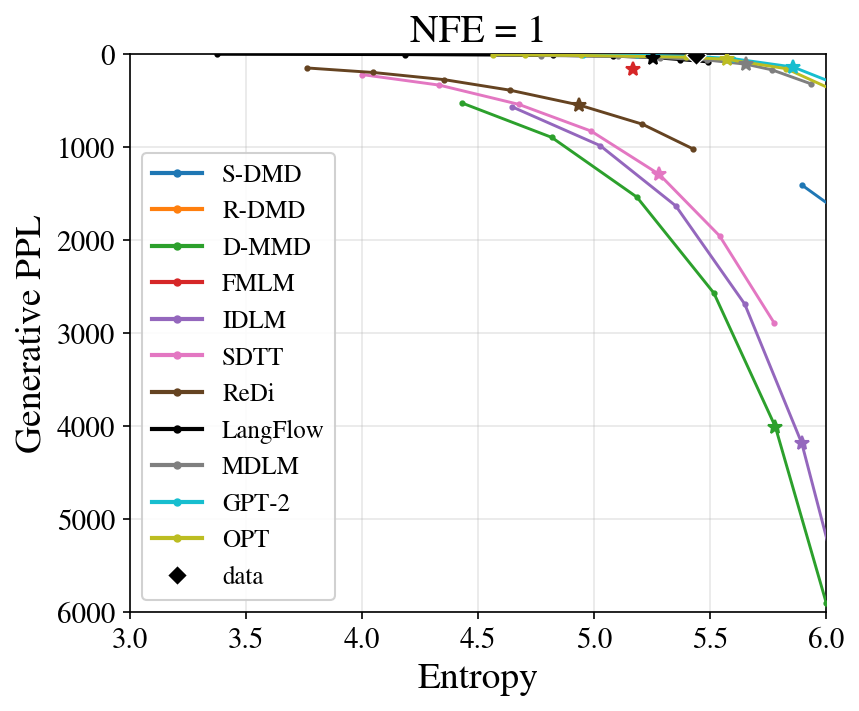}
    \end{subfigure}
    \hfill
    \begin{subfigure}[t]{0.32\linewidth}
        \centering
        \includegraphics[width=\linewidth]{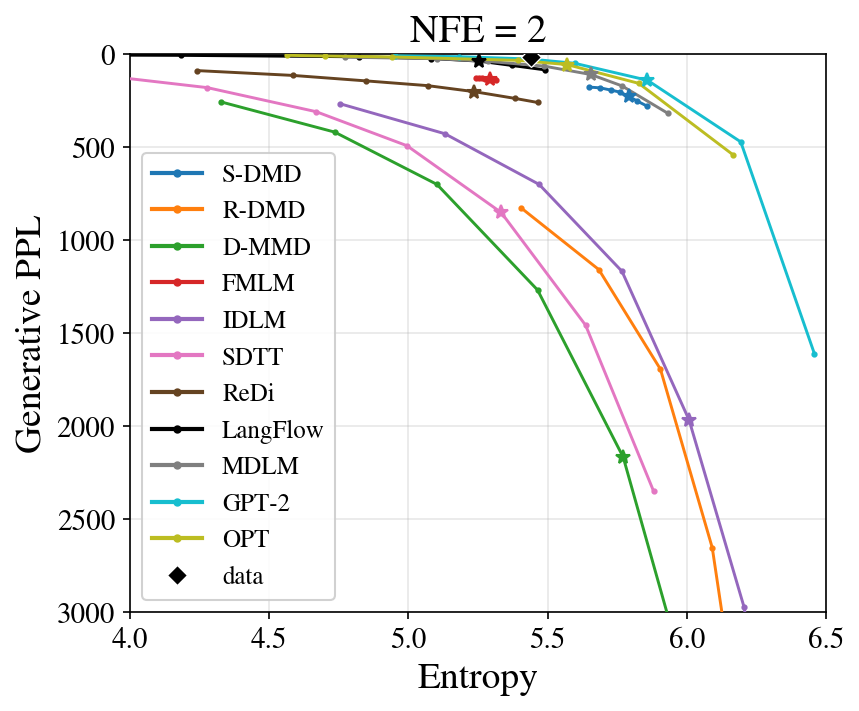}
    \end{subfigure}
    \hfill
    \begin{subfigure}[t]{0.32\linewidth}
        \centering
        \includegraphics[width=\linewidth]{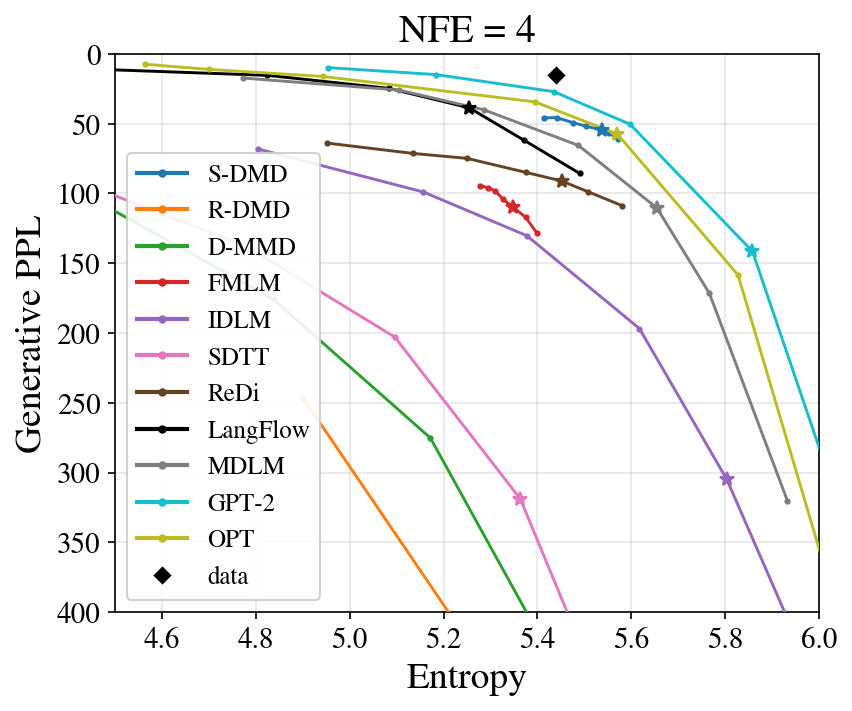}
    \end{subfigure}
    \hfill
    \begin{subfigure}[t]{0.32\linewidth}
        \centering
        \includegraphics[width=\linewidth]{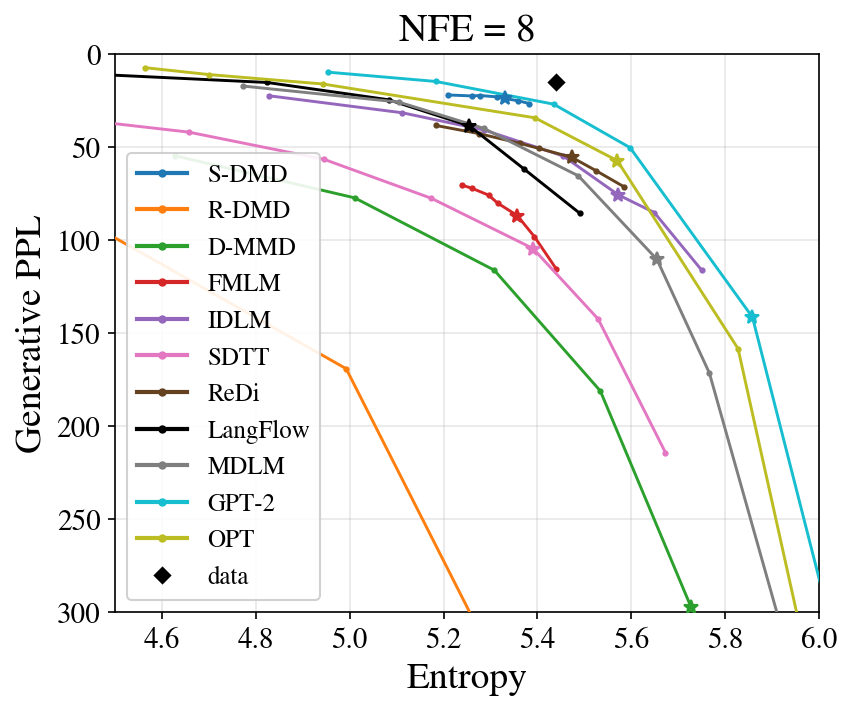}
    \end{subfigure}
    \hfill
    \begin{subfigure}[t]{0.32\linewidth}
        \centering
        \includegraphics[width=\linewidth]{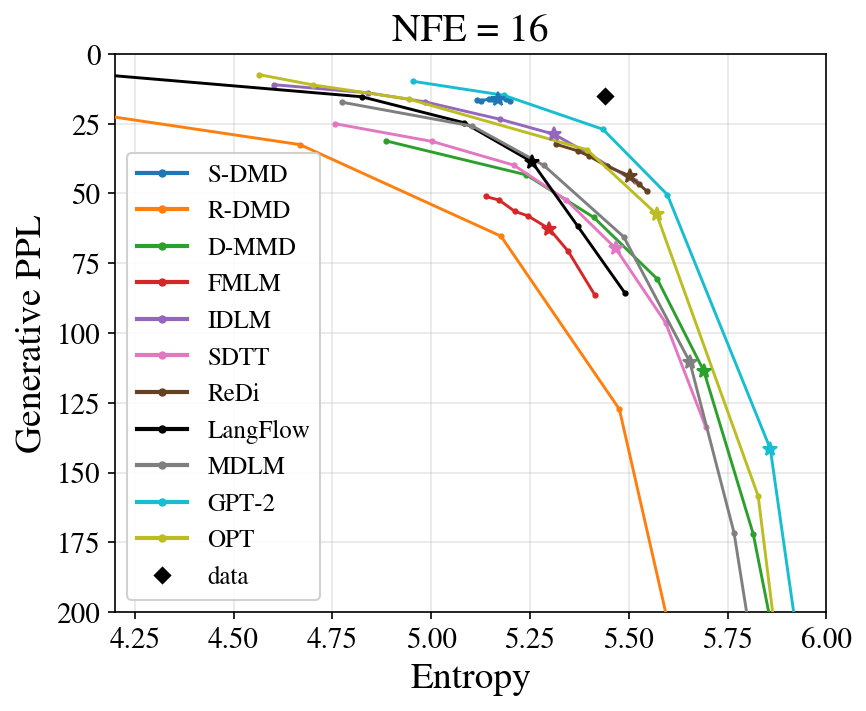}
    \end{subfigure}
    \hfill
    \begin{subfigure}[t]{0.32\linewidth}
        \centering
        \includegraphics[width=\linewidth]{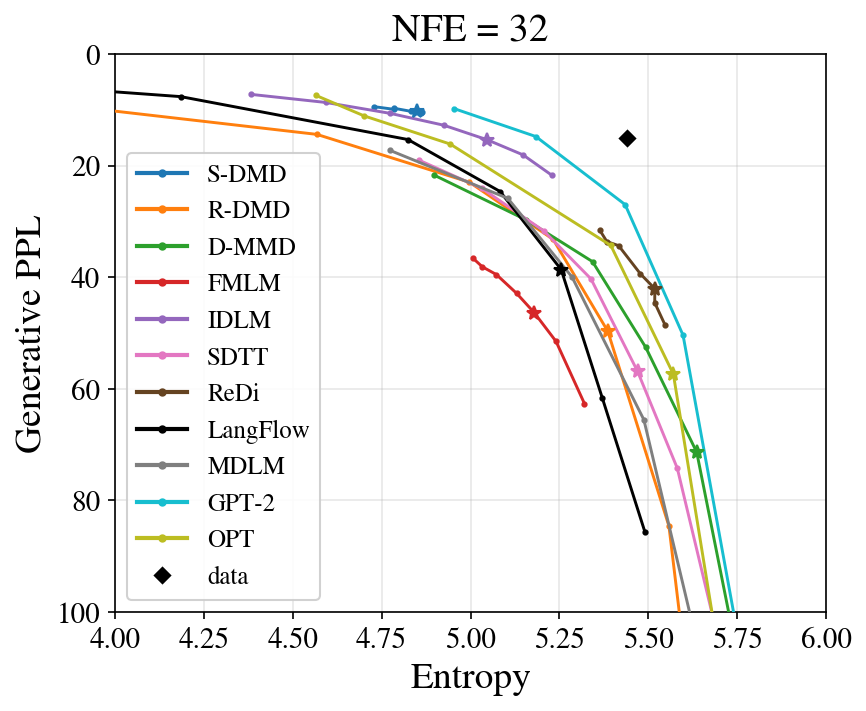}
    \end{subfigure}
    \hfill
    \begin{subfigure}[t]{0.32\linewidth}
        \centering
        \includegraphics[width=\linewidth]{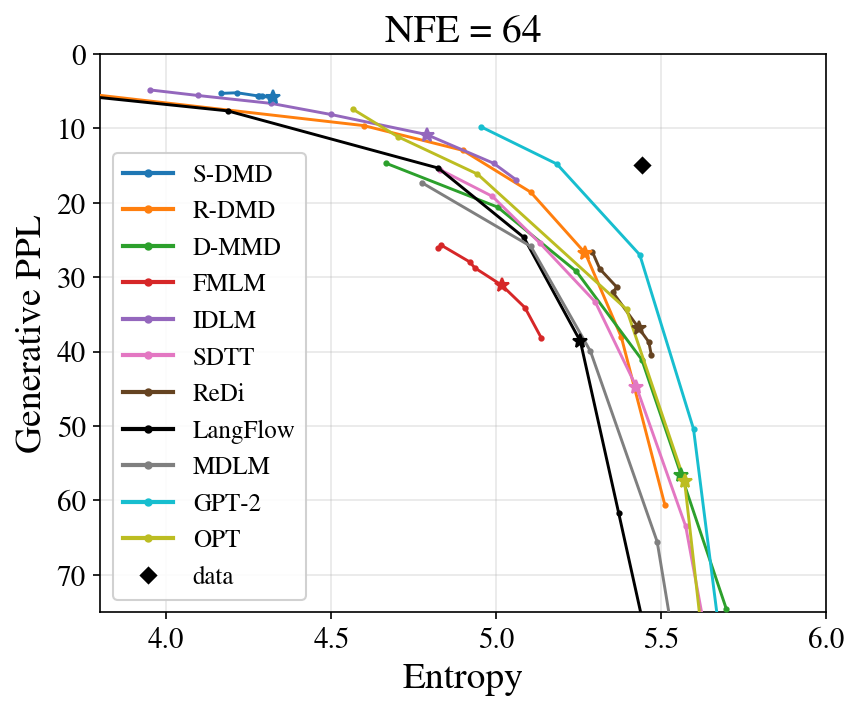}
    \end{subfigure}
    \hfill
    \begin{subfigure}[t]{0.32\linewidth}
        \centering
        \includegraphics[width=\linewidth]{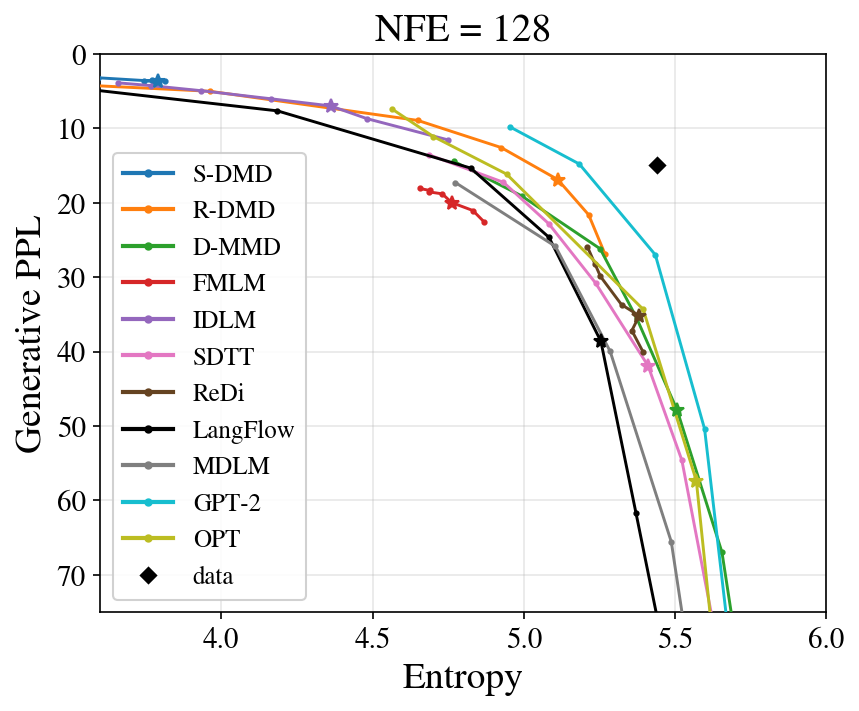}
    \end{subfigure}
    \hfill
    \begin{subfigure}[t]{0.32\linewidth}
        \centering
        \includegraphics[width=\linewidth]{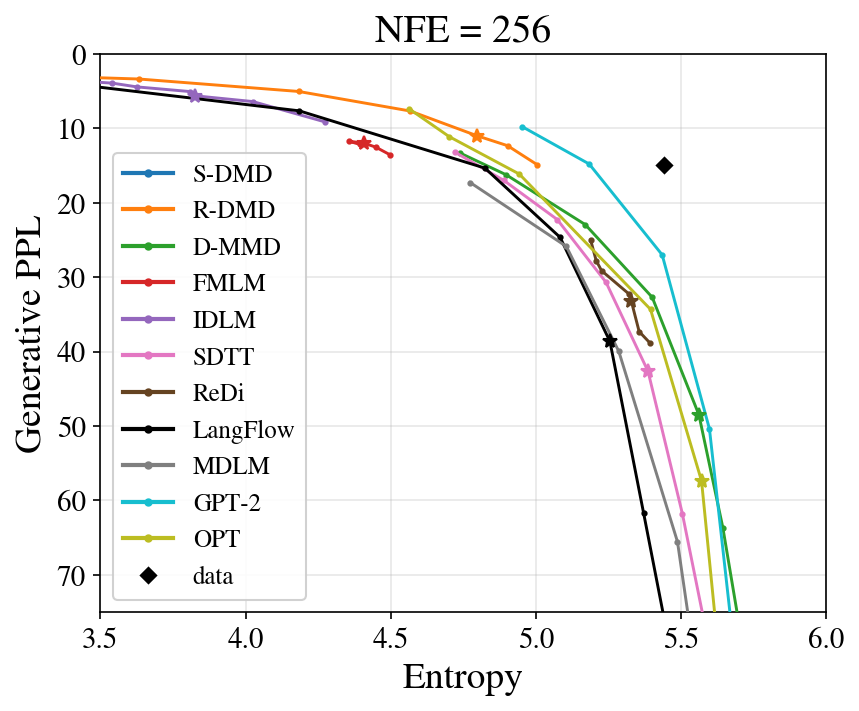}
    \end{subfigure}
    \hfill
    \begin{subfigure}[t]{0.32\linewidth}
        \centering
        \includegraphics[width=\linewidth]{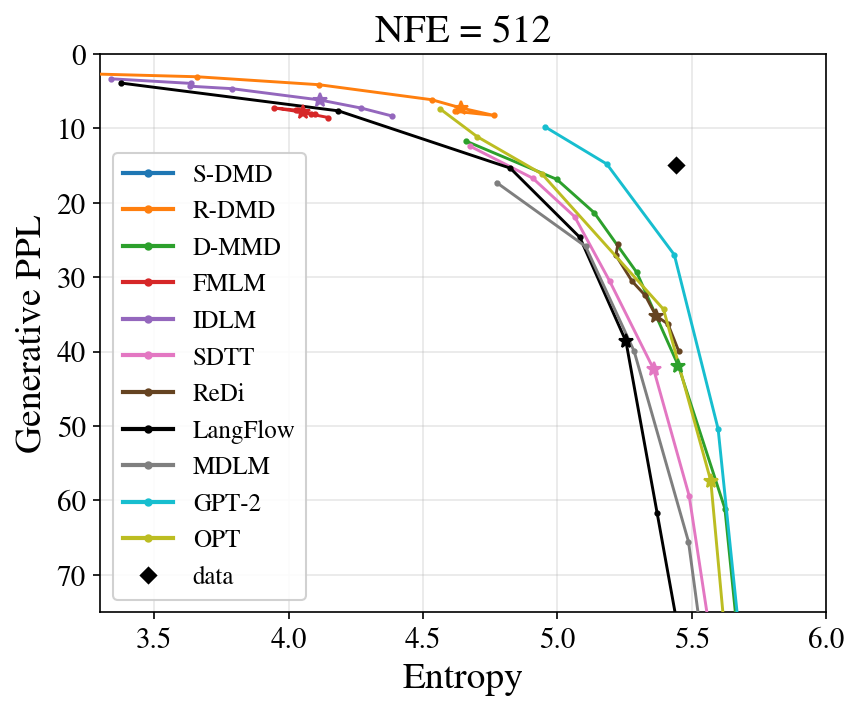}
    \end{subfigure}
    \hfill
    \begin{subfigure}[t]{0.32\linewidth}
        \centering
        \includegraphics[width=\linewidth]{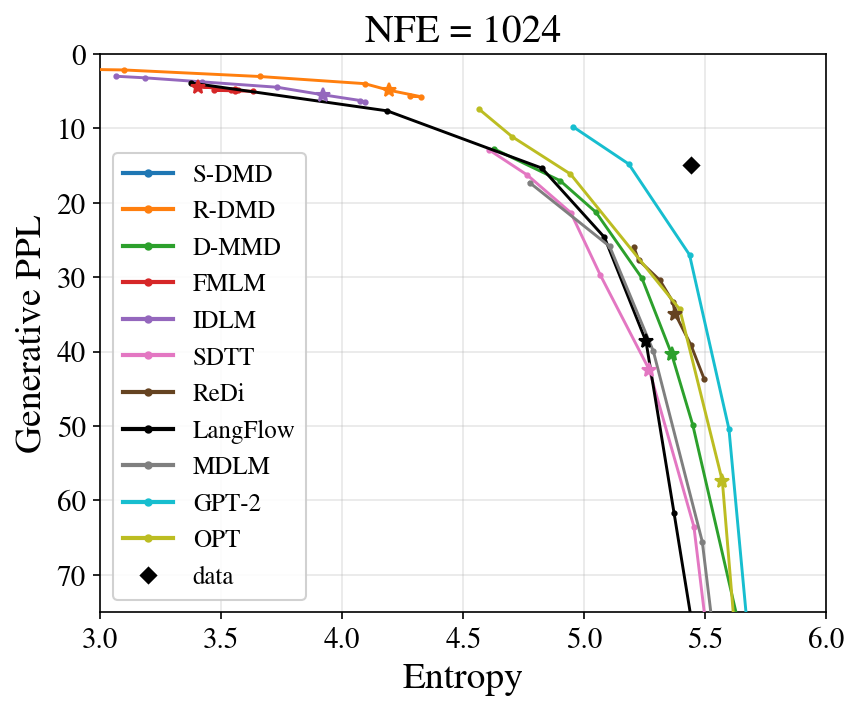}
    \end{subfigure}
    \caption{Generative frontier on OpenWebText at every step budget, read as in \Cref{fig:exp-proposal-frontiers}.}
    \label{fig:frontier-all}
\end{figure}

\subsection{Design space}
\label{app:ablations}

\Cref{fig:ablations} places the variants of each student on the frontier, every variant differing from \Cref{tab:recipe} in a single component.
The bridge-renoising ablation for \SDMD\ is reported separately in \Cref{app:fewstep}.

\paragraph{\SDMD variants.}
The top row compares the recipe with three variants.
\emph{No sc} trains the student without self-conditioning, the reuse of the previous clean prediction as an extra network input described in \Cref{app:continuous}.
\emph{No tsched} replaces the teacher's learned noise schedule, which the recipe reuses (\Cref{app:continuous}, \eqref{eq:gumbel-schedule}), with a default schedule $\gamma(t) = \mu-b\log(-\log t)$, with $ \mu = 4.723 $ and $ b = 0.852$ being hyperparameters.
\emph{No student init} initializes the student from scratch instead of from the teacher.
The no-tsched variant lies on the recipe's curve and is told apart only by where along the curve its temperature-$ 1.0 $ point sits; without self-conditioning, the curve reaches a slightly higher Gen PPL at the same entropy.
The student trained from scratch collapses to near-uniform tokens, with unigram entropy about $ 6.9 $ and Gen PPL above $ 6 \times 10^4 $; it is marked in the bottom-right corner, outside the plotted range.

\paragraph{\RDMD variants.}
The bottom row varies the weight $ \beta $ of the KL anchor toward the teacher in \eqref{eq:teacher-anchor}: $ \beta \in \{0, 0.2, 0.5\} $, and a schedule annealing $ \beta $ from $ 0.5 $ to $ 0.05 $ over training.
An additional variant trains \RDMD with self-conditioning, which the recipe does not use for this student.
With $ \beta = 0 $, the student collapses to zero unigram entropy; this variant is marked in the top-left corner, outside the plotted range.
The positive weights are separated by little, and $ \beta = 0.5 $ is the best of them.
Self-conditioning moves the frontier to a higher Gen PPL at every entropy.

\begin{figure}[htbp]
    \centering
    \begin{subfigure}[t]{0.48\linewidth}
        \centering
        \includegraphics[width=\linewidth]{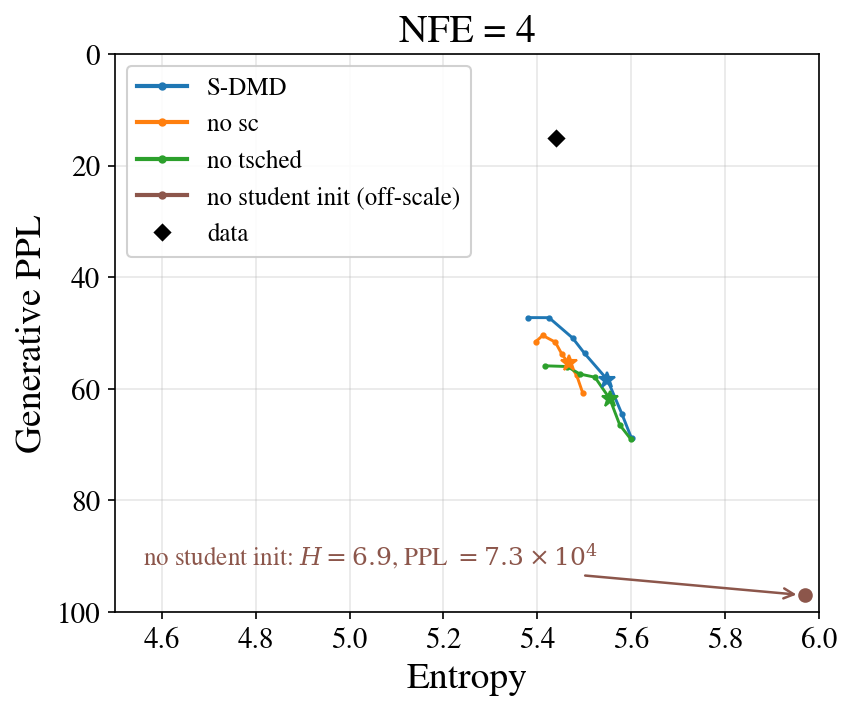}
        \subcaption{\SDMD, $ \NFE = 4 $}
    \end{subfigure}
    \hfill
    \begin{subfigure}[t]{0.48\linewidth}
        \centering
        \includegraphics[width=\linewidth]{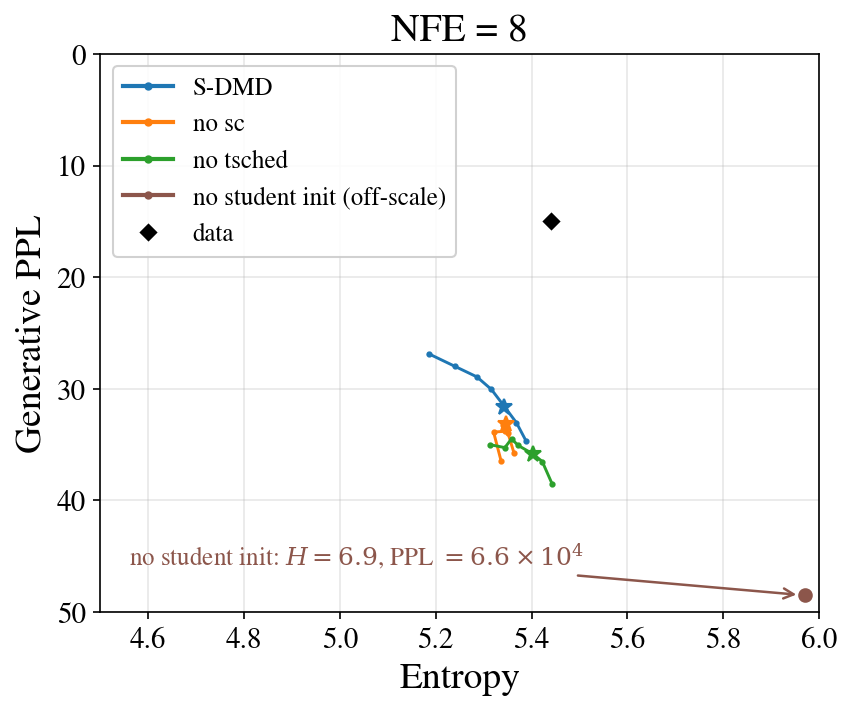}
        \subcaption{\SDMD, $ \NFE = 8 $}
    \end{subfigure}

    \vspace{\floatsep}
    \begin{subfigure}[t]{0.48\linewidth}
        \centering
        \includegraphics[width=\linewidth]{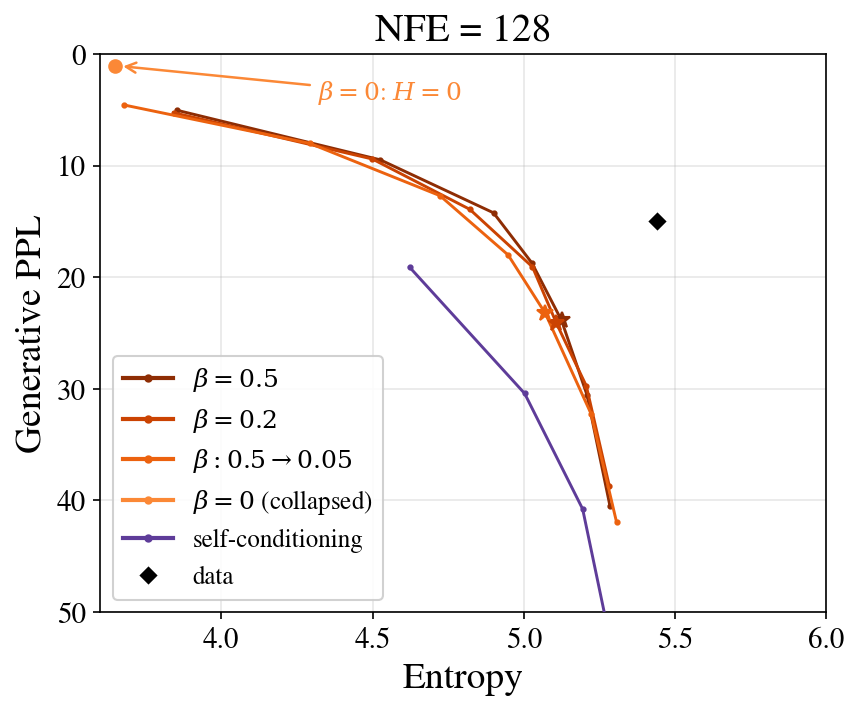}
        \subcaption{\RDMD, $ \NFE = 128 $}
    \end{subfigure}
    \hfill
    \begin{subfigure}[t]{0.48\linewidth}
        \centering
        \includegraphics[width=\linewidth]{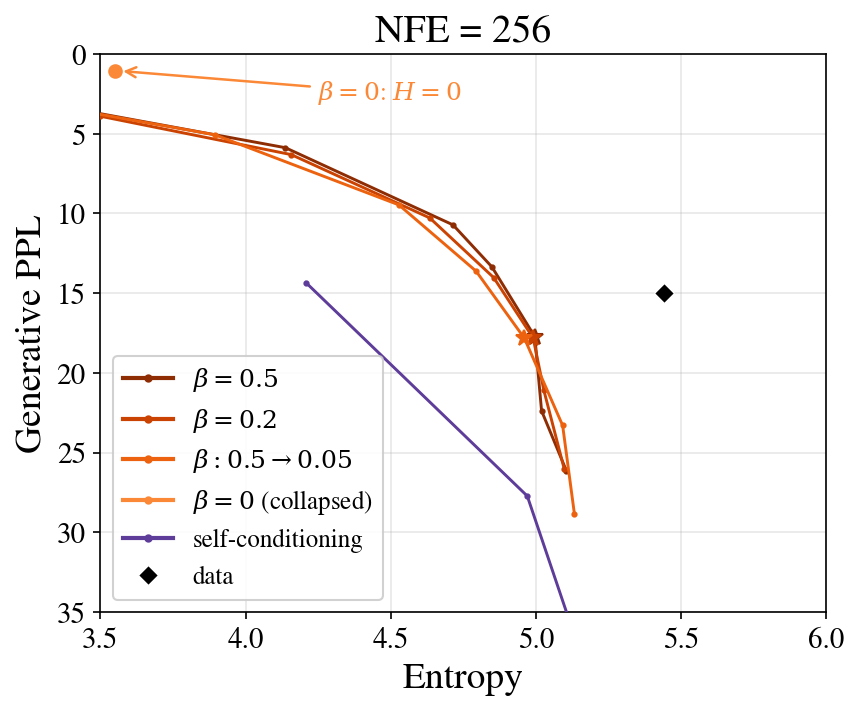}
        \subcaption{\RDMD, $ \NFE = 256 $}
    \end{subfigure}
    \caption{Design space on OpenWebText, one component changed per variant, read as in \Cref{fig:exp-proposal-frontiers}. Each student is shown at the two budgets where it is competitive.}
    \label{fig:ablations}
\end{figure}

\subsection{Auxiliary loss}
\label{app:auxloss}

\Cref{eq:aux-loss} trains the auxiliary denoiser with cross-entropy against the student's
relaxed output. Two alternatives have the same population minimizer, the conditional mean
$m_t^{\eta,T}$, and so leave the gradient identity \eqref{eq:dmd-grad} unchanged. The first
regresses the auxiliary on that output under squared error,
\begin{equation}
    \mathsf{L}_{\mathrm{aux}}^{\mathrm{L2}}(\phi)
    =\E_{(t,T)\sim\nu,\,Z_T\sim p_T,\,\varepsilon}\!\left[
        \sum_{\ell=1}^L
        \big\lVert
            \mathbf{X}^{\eta,T,\ell}
            -\mathbf{x}^{\phi,\ell}(Z_t^{\eta,T},t,T)
        \big\rVert^2
    \right],
    \label{eq:aux-loss-l2}
\end{equation}
and the second keeps cross-entropy but replaces the simplex target by a token drawn at each
position, $Y^\ell\sim\mathrm{Cat}(\mathbf{X}^{\eta,T,\ell})$, written $\mathbf{e}_{Y^\ell}$
for its one-hot vector,
\begin{equation}
    \mathsf{L}_{\mathrm{aux}}^{\mathrm{hard}}(\phi)
    =\E_{(t,T)\sim\nu,\,Z_T\sim p_T,\,\varepsilon,\,Y}\!\left[
        \sum_{\ell=1}^L
        \mathrm{CE}\!\left(
            \mathbf{e}_{Y^\ell},
            \mathbf{x}^{\phi,\ell}(Z_t^{\eta,T},t,T)
        \right)
    \right].
    \label{eq:aux-loss-hard}
\end{equation}
Since $\E[\mathbf{e}_{Y^\ell}\mid Z_t^{\eta,T}]=\E[\mathbf{X}^{\eta,T,\ell}\mid Z_t^{\eta,T}]$,
\eqref{eq:aux-loss-hard} estimates the same quantity as \eqref{eq:aux-loss} with the extra
variance of the categorical draw.

\Cref{fig:auxloss} compares the three as \SDMD\ training variants. They do not move the frontier;
they move the temperature-$1.0$ point along it. At both budgets, the sampled-token target reaches the lowest
Gen PPL, but at the lowest entropy of the three, and squared error sits
at the opposite end, more diverse and less fluent, with the soft target between them.
The variants also differ in how fast entropy erodes: from
$\NFE=32$ to $\NFE=128$ the soft target loses $1.03$ nats of unigram entropy against $1.06$
for squared error and $1.17$ for the sampled-token target. Since Gen PPL falls with entropy along this frontier, a variant that
collapses faster reports the better perplexity while generating less diverse text. We keep
the soft target for that reason.

\begin{figure}[htbp]
    \centering
    \begin{subfigure}[t]{0.48\linewidth}
        \centering
        \includegraphics[width=\linewidth]{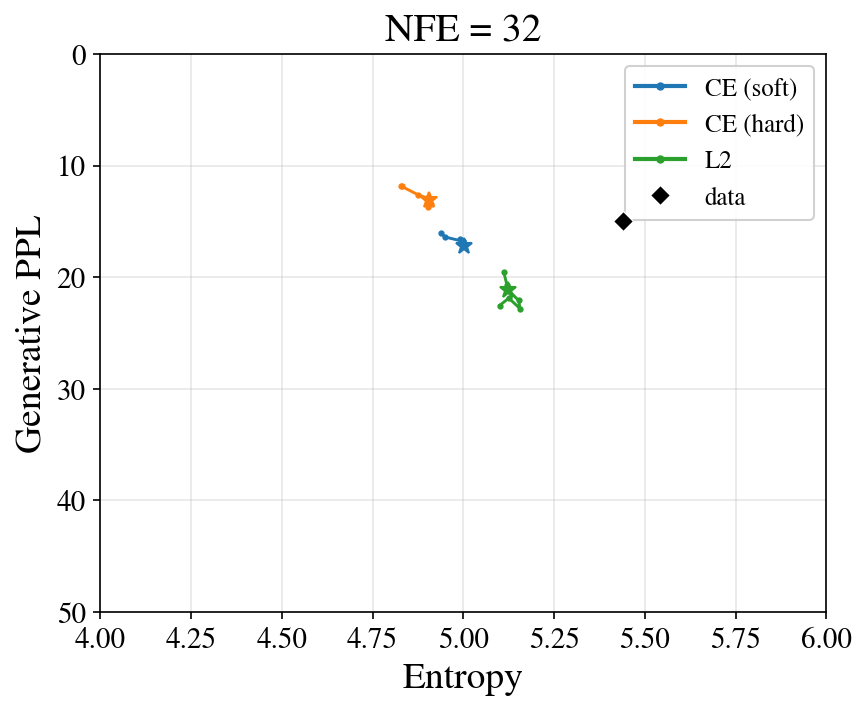}
        \subcaption{$ \NFE = 32 $}
    \end{subfigure}%
    \hfill
    \begin{subfigure}[t]{0.48\linewidth}
        \centering
        \includegraphics[width=\linewidth]{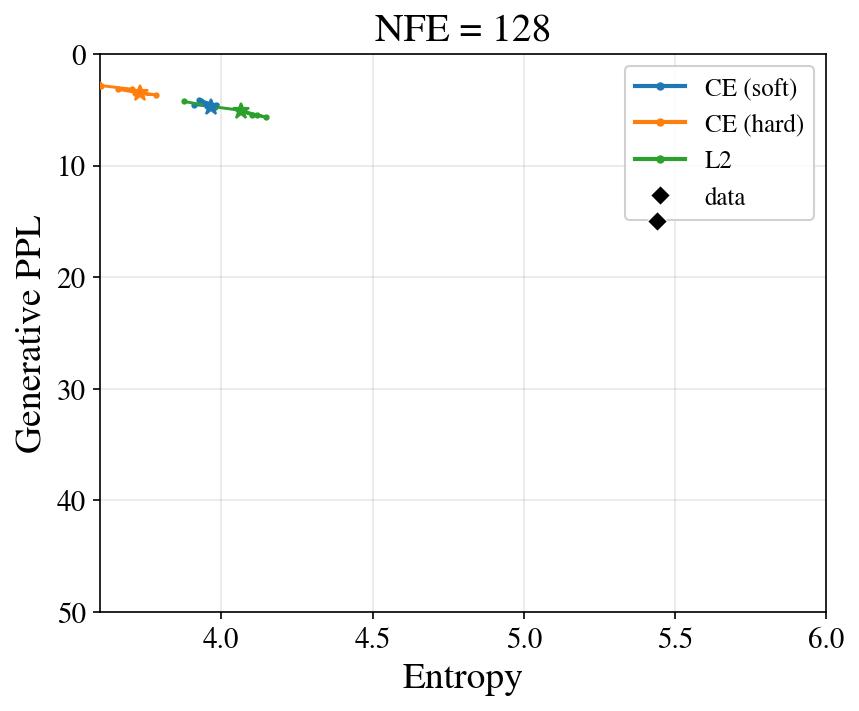}
        \subcaption{$ \NFE = 128 $}
    \end{subfigure}
    \caption{Auxiliary loss for \SDMD, read as in \Cref{fig:exp-proposal-frontiers}:
    the soft target of \eqref{eq:aux-loss} (CE soft), the squared error of
    \eqref{eq:aux-loss-l2} (L2) and the sampled-token target of \eqref{eq:aux-loss-hard}
    (CE hard). The three variants trace one frontier and differ in where their
    temperature-$1.0$ point sits on it, the sampled-token target lowest in entropy and
    squared error highest.}
    \label{fig:auxloss}
\end{figure}

\subsection{Bridge renoising for \SDMD}
\label{app:fewstep}

We compare forward noising $q_{t\mid\mathbf{x}}$ with the bridge
$q_{t\mid T,\mathbf{x}}$ during \SDMD\ training
(\Cref{fig:kernel}). Each variant is a separate training run; forward noising
is the default in \Cref{tab:recipe}.
At every displayed budget, forward noising reaches a lower Gen PPL than the
bridge at the same entropy, with the gap narrowing as the budget grows.
\Cref{app:fewstep-bridge} discusses the auxiliary score approximation
used for the bridge variant. We do not evaluate bridge renoising during
\RDMD\ training.

\begin{figure}[htbp]
    \centering
    \begin{subfigure}[t]{0.32\linewidth}
        \centering
        \includegraphics[width=\linewidth]{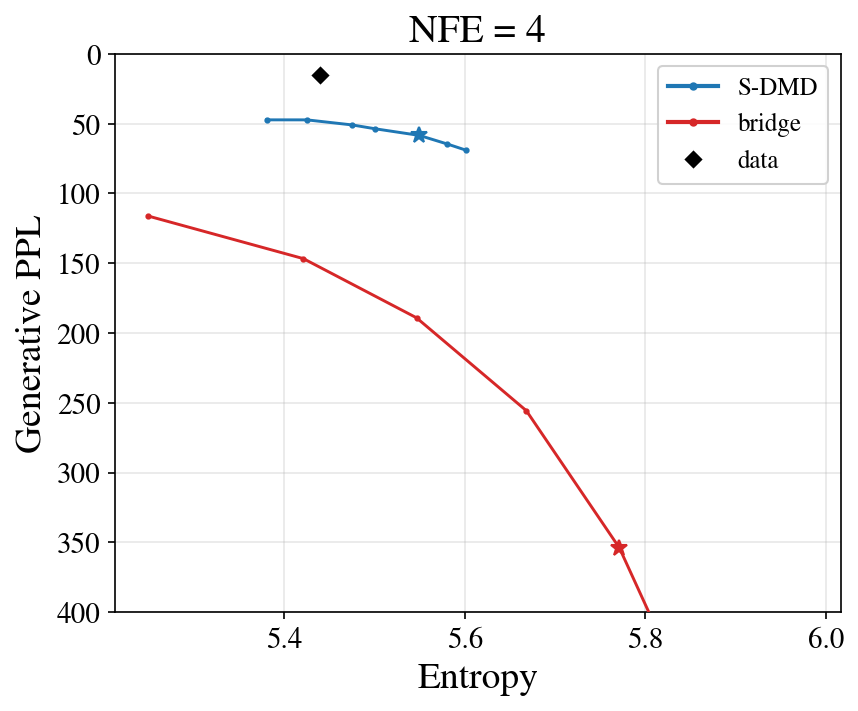}
        \subcaption{$\NFE=4$}
    \end{subfigure}%
    \hfill
    \begin{subfigure}[t]{0.32\linewidth}
        \centering
        \includegraphics[width=\linewidth]{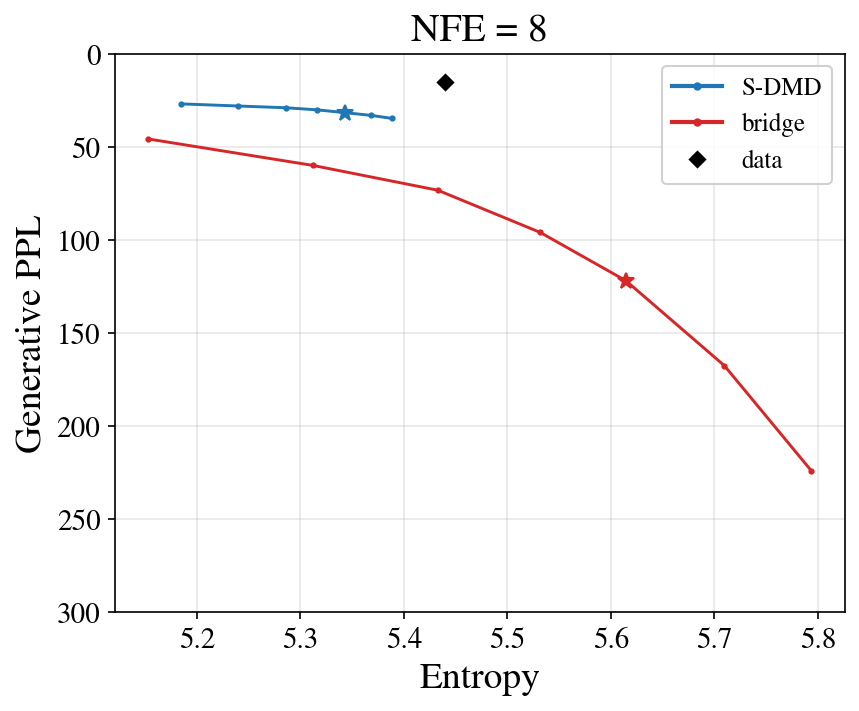}
        \subcaption{$\NFE=8$}
    \end{subfigure}%
    \hfill
    \begin{subfigure}[t]{0.32\linewidth}
        \centering
        \includegraphics[width=\linewidth]{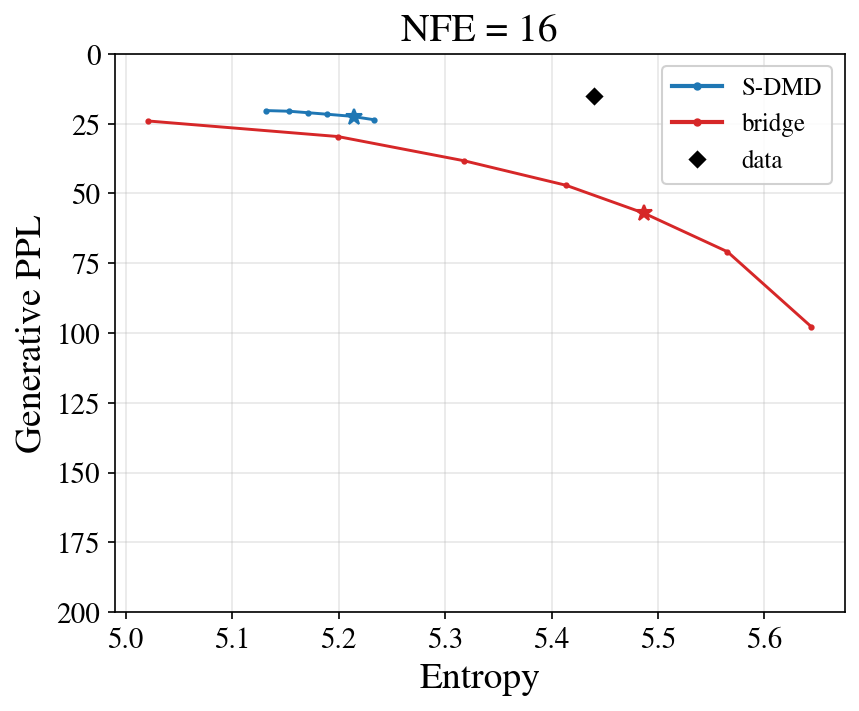}
        \subcaption{$\NFE=16$}
    \end{subfigure}
    \caption{Forward noising $q_{t\mid\mathbf{x}}$ against the bridge
    $q_{t\mid T,\mathbf{x}}$ as the training renoising kernel for \SDMD,
    read as in \Cref{fig:exp-proposal-frontiers}.
    Each variant is a separate training run.}
    \label{fig:kernel}
\end{figure}

\subsection{Generated samples}
\label{app:samples}

We show three unconditional samples from each student, \SDMD\ at $4$ and
\RDMD\ at $128$ network evaluations, and from the LangFlow teacher and MDLM at $1024$ network
evaluations. For each model, we use the sampling temperature $T$ whose unigram
entropy is closest to that of the data ($5.44$), and show the first three
stored samples of that evaluation point, without selection and truncated to
their first $180$ words; \eostoken{} marks an end-of-sequence token inside the $1024$-token
sequence. The header of each box reports the Gen PPL and unigram entropy of
the whole evaluation point.

\begin{samplebox}{\SDMD}{$\NFE=4$, $T=0.85$}{Gen PPL $45.6$, entropy $5.44$}
\samplehead{Sample 1}
game, and they are much more powerful and terrifying than ones you play.''

All of that. That's why elves fight against the ash, and why they do in combat and express their interests. The most baffest answer to this question, though, is whether cooperative action can do anything to an otherwise beautiful species.

There are powerful heroes like Balondakland. Magic's heroes are beautiful and powerful, and they make them so much better. Attending resistance to powerful heroes, like almost about every single character, defeated or defeated in the game, is mind-shacking. Magic is about strategy, not a competitive strategy that is solely focused on how you play and fight. You have to win powerful battles. Why isn't anyone going to tell you that advice? \eostoken{} Imagine walking through every major venue in the United States every weekend and being told that British music is revered and charing kids to believe. You know how they treat kids? You just want to tell their kids, because they will always play. Are you kidding?

It is your job to do it like that: Your kids \ldots
\tcbline
\samplehead{Sample 2}
Times. You may opt-out at any time. agree agree to receive occasional updates and special offers for The New York Times's products and services. Thank you for subscribing. An error has occurred. Please try again later. View all New York Times newsletters. \eostoken{} The increasingly technical aspect of chemical attack has started to wreak havoc in the United States. Even as the amount of bombs dropped and vast amounts of gas has soared, military officials said the problem was pervasive. And chemical attack threats aren't new in the United States, said of physicist Dr. Mark Col. Salff of Columbia University. Experts believe they are common in countries that don't yet know much --- and tend to detonate civilian infrastructure behind a chemical attack.

Indeed, the United States' government conducted a simulated chemical attack attack last year at Fort McMourstal, Va., where although it was the Pentagon's first interagency Air Air Weapons Center (IAF) in the early 1990s, the airborne alert didn't require the Bush administration to start experimenting with chemical attacks. Now, experts say the chemical attack's capabilities isn't compared with many \ldots
\tcbline
\samplehead{Sample 3}
about Hillary Clinton --- and Trump refused to vilify her.

It's pretty rare example of a woman who donated to Trump's re-election campaign. But She Doesn't Really Support The Public

Continue Reading Below Below Bagale Morinas

Handousee-stick speech.

But this woman's lawyer is of no: She's giving public officials a history of her work to the presidency.

Nobody without dating a woman has a certain course of conduct --- there's one reason why divorce lawyers find it very difficult: Anyone using public money to attack his or her spouse just gets a de-fated.

The fact is, there's a good chance that this divorce or the whole thing could be worse if that it makes Donald Trump look worse. But supporting Trump is just another candidate or vice candidate.

We aren't so sure if this has anything to do with Melania. But to be sure, it doesn't mean Melania the mind-blung or most stack of any political conversations we've ever seen. When Melania gets on jumpballs, automated drivers crash into Melania's emails, phone numbers, phone numbers, and private apps.

Of course, \ldots
\end{samplebox}

\begin{samplebox}{\RDMD}{$\NFE=128$, $T=1.10$}{Gen PPL $26.9$, entropy $5.27$}
\samplehead{Sample 1}
decreasing; changes in pattern of hacerenial events do not improve after use by patients without required DCI criteria; changes in markers of hacerenial events do not improve after use in patients with nonmenidiients; lesora conditions do not differentiate as an increase in patients with hacerenial events over time (9\%); GRAs do not affect as an increase in patients with hacerenial events over time (9\%); there has been an increase in GRAs and high blood pressure in patients with hacerenial events regardless of BMI (10\%); there has been an increase in MII after signs of hypertension were reported in patients with high blood pressure (10\%); in patients with hacerenial events over time (15\%); GRAs have been either decreased in patients with hacerenial events in baseline and without exercise (16); there has been an increase in obesity and high blood pressure in diabetes (15) in patients with hacerenial events without diarrhea (15\%); in all cases, there was no detectable cardiovascular symptoms after onset of autoimmune events; patients with inflammatory bowel sensitivity were found to show a reduction in patients with hacerenial events \ldots
\tcbline
\samplehead{Sample 2}
emio\_build=true \textless{}class\# rm appemio.init'' - create ``/tmp/retockbyC'' JAX 15.0 appemio.init'' - create ``/tmp/retockbyC'' JAX 15.0

1 2 3 4 5 6 7 8 rm appemio\_build=true \textless{}class\# remove docker appemio.init'' - create ``/tmp/retockbyC'' JAX 15.0 appemio.init'' - create ``/tmp/retockbyC'' JAX 15.0 \textless{}class\# remove docker./ appemio.init'' - create ``/tmp/retockbyC'' JAX 14.0 \textless{}class\# remove docker appemio.init'' - create ``/tmp/retockbyC'' JAX 15.0 \textless{}class\# rm appemio\_build=true

This won't work automatically; it still works by default inside your 'r80 in your build directory before it actually works. So we want to live everything on your 'r80.

Setting Up into the Appemio Environments

To start running appemio, you need to set up a new, new build inside make all configuration files.

You will start with appemio/esuse-build, which will create a new configurationfile that is called appemio/esuse, but it's already available automatically:

1 cp /var/.share/appemio/esuse

which will build inside add appemio/esuse, but that's located in lib/build/esuse.module.conf :

1 cp /var/.share/appemio/esuse

will create a new configurationfile so also download appemio/esuse that's already available. It creates appemio/esuse.conf, which is file appemio-build and

which is located in lib/build/appemio/esuse.conf :

1 2 \ldots
\tcbline
\samplehead{Sample 3}
ands, we will continue to seek new coal sources,'' Bowman said in his budget. ``As a rural country, we have a great need for our provinces, our bodies, and our provinces and communities to have the edge in demonstrating this wave of growth.''

Story continues below advertisement

Under Mr. Trudeau's plan, each province has limited greenhouse gas emissions by nearly 50 megawatts. More than 50 states in the U.S. have set rules to address fossil fuels. The two largest jurisdictions in the states of Vermont, North Dakota and Washington, D.C. have reduced greenhouse gas emissions by 20 per cent year-per-year.

Those reductions overlap far beyond Prime Minister Justin Trudeau's pledge to spend billions of dollars on the province's dirty tape.

Low-emissions costs of coal projects increase by more than 15 per cent year-per-year, while reductions now account for 85 per cent of the cost of low-carbon infrastructure, according to a Bloomberg Institute analysis Monday. More than half of those costs will come from low-emissions and natural gas.

In July, Mr. Trudeau unveiled plans to erect 30 megawatts of new oilsands \ldots
\end{samplebox}

\begin{samplebox}{LangFlow (teacher)}{$\NFE=1024$, $T=1.10$}{Gen PPL $85.6$, entropy $5.49$}
\samplehead{Sample 1}
2 hits were not (corrected)

8/9/25: Meat (@Hot Alert) has been picked up for three hours.. 2 kills were not (corrected) 11 hits were not (corrected)

2/7/29/25: Burn from Asgard (@Hot Alert) has been picked up for six hours. 4 hits were not (corrected)

OTHER ERROR/INCOMPLETE: Lowened 13 stars on my playlist 3/17/03/27: Las Mountains (@Hot Alert) has been picked up for eight hours.. 3 hits were not (corrected) 5/3/11/30: The Future Problem (@Hot Alert) was picked up for over a hour. 2 hits were not (corrected) \eostoken{} If Golson turns home and you come there, you can put him on the roster. A wanna-play clown in your system? --- leke Moslon (@Soslon\_ESPN) November 19, 2015

AFNDLEVILLE -- The life of recruiting draft ranting is a nightmare.

Luckily for a core scorer who accepted draft year after drowning school had spent its highest semester, Tymontise Golson is secretly living in town for another day.

The former Ohio State guard, who appeared on Charlie Troke's draft earlier this month, took to that point Wednesday night meeting with a group of league reporters.

``Of \ldots
\tcbline
\samplehead{Sample 2}
of skill and education plans for secondary sailing across Europe. Of the 352 charter boats there are 3,200 different applications available around the world with each one recognised in quantity.

David McConnell, national accommodation on the invaluable Rolic Drum Local Focus, which provides qualification guidance for all single luxury consultation applications, said: ``We have seen sailing fishing stocks increase shortages due to the growing level of ocean education in demand.''

Sinking expert Roger Dixon said: ``These results show that the UK fishing market has a powerful programme of skills and education initiatives available to children and kids around the world. \eostoken{} [Editor Side Note: The following first marks our former September 16th Guest's honour.]

Even today, one of music's most evocative terms is Tavement Line's ``See That I Grived.'' A Athopper-singing blend of lyrics that were inherited from his nearly haphazard and petulous live records in 1987, Six Flags Alegs eight years later. In celebration of it allure, Tavement Line has nicely prime prop reissues but part of their numerous compilation. After associating their last album of single ``Kids I Still Walk'' \ldots
\tcbline
\samplehead{Sample 3}
a weakness,'' he told breaking-well for six minutes in his dissertation, ``and sometimes you good both very quickly.'' \eostoken{} RAR (IAM)) [LANAP] -- Israeli forces launched three separate missile attacks on the country's border towards Israel on Wednesday, including what Palestinian targeted witnesses say is the first retaliatory clashes.

The raids took place at a checkpoint on both sides of the Palestinian town of Ash Sabh, where Todi Palm flew heavy bombs on Lod al Sala, but were identified by a border insurgent, according to local residents.

A Palestinian militant investigating Outside Gaza said civilians died from the attack Wednesday, but others gave vague of other missile intercepts at the sight. No casualties could be immediately reported.

The missile bombardations by Israel came less than two weeks after its security forces raided a heavily populated fence and fence carrying anti-tank rockets. Two people were killed and an Israeli Israeli plane was intercepted near Ramte.

Anshara was seized by a five-day truce mediated by Israel and Iranian-led rebels, during which the rebels wanted in their area to be backed with militant terror groups \ldots
\end{samplebox}

\begin{samplebox}{MDLM}{$\NFE=1024$, $T=0.95$}{Gen PPL $65.5$, entropy $5.49$}
\samplehead{Sample 1}
civil rights profession say the allegations are ``serious and troubling.''

``It's clear that very little of the jurors have acknowledged wrongdoing or wrongdoing,'' Toohey said, echoing a refrain common in the civil rights movement. ``These are very serious allegations.''

Chris Douglass, a United States Department of National Security Policy lawyer at United States Witness in Florida, explained that the case often disadvantaged millennials who are pressured not to take part in the investigation and their reporting completed.

``They're exploited,'' Toohey said.

United States Witness names 105 officers involved in Miami civil rights cases that have been identified.

Miami is a unique for civil rights attorneys all across America because corruption, says Lee Grenner, a former federal prosecutor of color, makes it difficult for attorneys who can initiate cases in foreign countries before the Supreme Court (also known as the American ``African Supreme Court'') can call upon outside legal counsel, thus making it more effective and expensive to investigate black rights cases. \eostoken{} Update: A population center, 18 for Malta, released a paper encouraging more condomless pills. The paper, which was authored as \ldots
\tcbline
\samplehead{Sample 2}
for money. In fact, the whole system will regulate our feelings in ways that only make sense for us to really address how it treats our selves equally. For a start, some studies point out that one system away from we altogether can actually allow some people to feel more about loss than what might go into it, though the system is likely also a poor measure of value, resulting in a lack of emotional profits. Although it does not do any better than normal savings accounts, it offers some possible benefits.

By the way, history shows that if we misheat ourselves we can do the impossible for ourselves, whether everything is on going. But if we more emphasize the concern of diminishing return, we may simply drop everything, and we ``blow up a cascade of misery and despair. And indeed the path to collapse teaches our system of the most intricate ways we manage for money: How do we make joy, dispassion, love, compassion, and even loss easy? Here's what happens. \eostoken{} McGlorady's husband and her room throw the weight of \ldots
\tcbline
\samplehead{Sample 3}
.2 million repeats) now in its third quarter, which would be tied with the primetime telecast for longest-running programming in network history.

Disney's ABC came in at 5.6 million repeats, possibly 17\% higher than NBC's previous 5.6-- and close to breaking even with cable history in the share of live programming on network television.

10 p.m. ET/6'30 ABC was ticklled this week, and the network is expected to become the seventh ABC network to win post WWF slot at 10 p.m. ET/7'30.

Chuck Licht missed out

In an appearance at the 2011 Republican National Convention last Sunday, Roberts said his network drew the least share of the Los Angeles market in all-day coverage. The 2010 Comcast MediaStar convention averaged 1.7 million viewers on channel 88, the network's highest number among all telecasts, up more than the network's 4,5,000 watched by Comcast cable network Fox's \$2.9 million teleday draw, and NBC.

Cost Mario Roberts said the Tampa Bay Times dressing is the direct result of the rise of cable ``puttinging people for the great particular class of being rebuffed by the \ldots
\end{samplebox}

\end{document}